%% file: main.tex
\documentclass[twoside]{article}

\usepackage{mathtools}
\usepackage{amssymb}
\usepackage{bbm}
\usepackage{booktabs}

\newcommand{\E}{\mathbb{E}}
\newcommand{\R}{\mathbb{R}}

\newcommand{\cC}{\mathcal{C}}
\newcommand{\cI}{\mathcal{I}}

\newcommand{\Hp}{H}

\DeclareMathOperator{\Var}{Var}
\newcommand{\diag}{\operatorname{diag}}
\newcommand{\KL}[2]{D_{\mathrm{KL}}(#1 \| #2)}

\DeclareMathOperator{\ri}{ri}
\usepackage{bbm}
\newcommand{\Lip}{\mathrm{Lip}}
\newcommand{\cL}{\mathcal{L}}

\usepackage[preprint]{aistats2026}
\usepackage[round]{natbib}

\usepackage{enumitem}
\usepackage{amsfonts}
\usepackage{pifont}
\usepackage{hyperref}
\usepackage{cleveref}
\usepackage{amsthm}
\usepackage{physics}
\newtheorem{proposition}{Proposition}[section]
\newtheorem{lemma}{Lemma}[section]
\newtheorem{corollary}{Corollary}[section]
\newtheorem{remark}{Remark}[section]
\newtheorem{definition}{Definition}[section]
\newtheorem{theorem}{Theorem}[section]

\begin{document}

%
\runningtitle{The Reasoning--Creativity Trade-off}

%

\twocolumn[

\aistatstitle{The Reasoning--Creativity Trade-off: \\ 
Toward Creativity-Driven Problem Solving}

\aistatsauthor{ Max Ruiz Luyten \And Mihaela van der Schaar }

\aistatsaddress{ University of Cambridge } ]

\begin{abstract}
State-of-the-art large language model (LLM) pipelines rely on bootstrapped reasoning loops—sampling diverse chains of thought and reinforcing the highest-scoring ones—mainly optimizing correctness. We analyze how this design choice is sensitive to the \emph{collapse} of the model’s distribution over reasoning paths, slashing semantic entropy and undermining creative problem-solving. To analyze this failure, we introduce \emph{Distributional Creative Reasoning} (DCR), a unified variational objective that casts training as gradient flow through probability measures on solution traces. STaR, GRPO, and DPO, as well as entropy bonuses, and other methods, all constitute special cases of the same loss. The framework delivers three core results: (i) the \textit{diversity decay theorem}, describing how correctness-based objectives lead to distinct modes of diversity decay for STaR, GRPO, and DPO; (ii) designs that ensure convergence to a stable and diverse policy, effectively preventing collapse; and (iii) simple, actionable recipes to achieve this in practice. DCR thus offers the first principled recipe for LLMs that remain both correct \emph{and} creative.
\end{abstract}

\section{Introduction}
\label{sec:intro}

\paragraph{Diversity collapse in modern training loops.}
A canonical post-training pipeline for training reasoning LLMs includes two main stages: after supervised fine-tuning, the focus shifts to reinforcement learning (RL), which rewards the highest-scoring traces, typically based on correctness. A recurring and detrimental side-effect of this process is \textbf{creative collapse}: the model’s output entropy plummets, resulting in a distribution dominated by a handful of semantic templates \citep{Mohammadi2024}. 

Creative collapse has been extensively reported across RL from human feedback (RLHF) stages \citep{Kirk2024}, when applying GRPO for mathematical reasoning \citep{shao2024deepseekmathpushinglimitsmathematical}, and during self-consistency tuning \citep{wang2023selfconsistency}. In this paper, we examine why this collapse occurs and whether we can apply design choices that prevent it without sacrificing accuracy.

\paragraph{Why diversity matters: Creativity as a diverse portfolio for generalization.}
Especially for tasks outside the training distribution (OOD), creativity in problem-solving is not just a nice-to-have but rather a core requirement for high performance. A single reasoning template will inevitably fail when under novel conditions. We therefore frame creativity as the ability to maintain a \emph{diverse portfolio of high-utility reasoning strategies}. This portfolio promotes OOD generalization, robust planning, and genuine discovery \citep{StanleyLehman2020}.

\paragraph{The central question.}
Our work addresses the following question:
\begin{quote}
\emph{Can we design a framework that:
\begin{enumerate}
    \item explains why diversity collapse occurs,
    \item predicts the specific mode of collapse for different algorithms, and
    \item provides provably effective designs that guarantee a diverse portfolio of reasoning paths?
\end{enumerate}}
\end{quote}
Existing literature provides incomplete answers. KL penalties preserve diversity by constraining the policy’s \emph{proximity} to a base model, limiting drift at the cost of indiscriminately penalizing diverse, high-utility distant parameterizations. Sampling-based methods like Boltzmann sampling or top-$k$ decoding also increase diversity at the cost of quality, and, more critically, they cannot recover strategies whose probabilities have vanished during training.

\paragraph{Our answer: Distributional Creative Reasoning.}
Our primary contribution is theoretical: we provide a unified framework to analyze diversity decay and a provably sufficient remedy. Since our object of study is not an individual trace, we analyze the dynamics of the entire conditional distribution $p_\theta(\pi\mid x)$ over the space of solution traces. By modeling training as a gradient flow on this probability simplex, we develop a framework, Distributional Creative Reasoning (DCR), to analyze diversity decay and uncover its various sources. The DCR objective is a core component of this framework and encompasses multiple terms for utility, regularization, and a crucial, strictly concave diversity energy:
\[
    J(p) = \mathcal{U}[p] + \lambda \mathcal{D}[p] - \beta_{\!\mathrm{KL}}\,\mathrm{KL}\!\bigl(p\Vert p_{\mathrm{base}}\bigr).
\]
In particular, the \textbf{diversity energy} $\mathcal{D}[p]$ is a composite functional with two distinct roles:
\[
    \mathcal{D}[p] = \alpha H[p] - \beta Q[p].
\]
In this equation, $\alpha H[p]$, the Shannon entropy, promotes undiscriminated \textbf{breadth}, while $-\beta Q[p]$ is a \textbf{kernel coverage} term that penalizes concentration on semantically similar traces, thereby promoting conceptual distinctiveness. This objective can recover various existing algorithms as specific instantiations, including STaR \citep{zelikman2022star}, GRPO \citep{shao2024deepseekmathpushinglimitsmathematical}, and DPO \citep{rafailov2023direct}.

DCR leads to three core theoretical insights: First, it leads to the \textbf{Diversity Decay Theorem}, which predicts distinct modes of collapse under scalar-only objectives for the most well-known reasoning algorithms: \textbf{(i)} a “winner-takes-all” fixation for STaR, \textbf{(ii)} a neutral drift for GRPO, and \textbf{(iii)} a homogenization of correct strategies for DPO. 

Second, we prove that incorporating the DCR diversity energy fundamentally can alter the learning dynamics, guaranteeing convergence to a \textbf{unique, stable, and diverse interior equilibrium} that neutralizes these collapse modes. 

Third, DCR provides a set of \textbf{design levers}, the specific creativity kernel $k(\pi,\pi')$ and the coefficients $\alpha$ and $\beta$. We analyze the effects of their choices, resulting in a recipe for training models that are both correct and creative.

\paragraph{Contributions.}
\begin{enumerate}[leftmargin=*, itemsep=2pt]
\item \textbf{Unified Dynamical Lens.} We introduce a variational framework based on Shahshahani gradient flow that encompasses STaR, GRPO, and DPO. Within this framework, we derive their diversity decay dynamics under scalar objectives and finite-batch noise. We also provide a recipe for adapting the framework to new reward designs.
\item \textbf{A Remedy for Collapse.} We prove that the DCR objective, with the diversity energy functional $\mathcal{D}[p] = \alpha H[p] - \beta Q[p]$ guarantees convergence to a high utility and (under an appropriate design) diverse policy, preventing creative collapse.
\item \textbf{Principled Design Space and Practical Recipes.} We detail how to design the creativity kernel and provide guidance on tuning DCR’s hyperparameters. We hope this will transform diversity preservation from ad-hoc heuristics to a principled design process. 
\end{enumerate}

\paragraph{Road-map.}
\Cref{sec:related} discusses the literature on diversity collapse and related theoretical frameworks. \Cref{sec:framework-formal} formally defines the DCR objective and its associated gradient flow dynamics. \Cref{sec:collapse} presents the Diversity Decay Theorem, analyzing the distribution modes of STaR, GRPO, and DPO under scalar objectives. \Cref{sec:phase} proves how the DCR diversity energy reshapes the equilibrium landscape to guarantee diverse outcomes, and \Cref{sec:kernel} discusses the design of the creativity kernel. Finally, \Cref{sec:discussion} concludes with key insights and future directions. We empirically validate these theoretical collapse modes in \Cref{app:insight-exp}.

\section{Related Work}
\label{sec:related}

\paragraph{From reward optimisation to \emph{reasoning monoculture}.}
A consistent empirical observation is now widely documented in the literature: when a language model is trained to maximise a \emph{single} scalar reward, its solution space contracts.
Early studies of RLHF showed that the resulting policy rarely develops novel strategies; instead, it reweights the trajectories present in the SFT checkpoint, leading to higher \textit{Pass@1} accuracy while leaving the underlying portfolio unchanged~\citep{Yue2025}.
Controlled ablations subsequently isolated the cause to the RLHF stage. Diversity, measured by entropy, type–token ratio, and embedding spread, dropped notably after RLHF, while the preceding SFT maintained it~\citep{Kirk2024}.
The effect is algorithm-agnostic: PPO, Expert Iteration, and GRPO all converge to the same narrow attractors, failing “to explore significantly beyond solutions already produced by SFT models”~\citep{Havrilla2024}.

Beyond reasoning-based benchmarks, creative decline has also been documented in other domains.
On open-ended story-telling and idea-generation tasks, aligned \textsc{Llama-2} variants lose $3$–$6\times$ token-level entropy and cluster in a few semantic basins~\citep{Mohammadi2024}.
Treating a set of traces as a ``population,‘’ \citet{Murthy2025} quantified conceptual variance, further underscoring that RLHF results in less diversity than either instruction-tuned or human populations. The overall conclusion from these works is that performance gains come, at least partly, at the cost of reducing the space of possible explanations and expressions.

\paragraph{First attempts at diversity-aware objectives.}
Several works have sought to counter this collapse by injecting \emph{ad hoc} diversity terms.
Entropy-regularised PPO is the most widespread heuristic, but its effect is largely to keep stochasticity indiscriminately, leaving performance gains on the table, and it does not aim to foster \emph{qualitatively} distinct ideas.
Novelty search and quality-diversity algorithms from evolutionary methods have also been applied to language modelling, yet the generated solutions are typically managed separately from the model, and re-distillation frequently regresses gains~\citep{Havrilla2024}.
At the reward level, \citet{Xiao2024} identified ``preference-collapse’’ in RLHF and proposed a Preference-Matching regulariser that adds an entropy bonus, improving minority-preference recall but with the same drawback as discussed above, and without a principled analysis of \emph{how much} diversity is sufficient.
In conclusion, these works demonstrate viability but leave open a unifying view that predicts \emph{when} collapse will occur and the size of the required counterforce.

\paragraph{Theoretical lenses on collapse.}
Two theoretical lines are especially relevant.
First, replicator dynamics from evolutionary game theory \citep{hofbauer1998} have been used to model reward optimisation in large populations and already hint that pure utility maximisation drives mass toward the highest-fitness type.
Second, information-theoretic RL reinterprets entropy bonuses as Lagrange multipliers of a KL constraint, but offers no guarantee that entropy will capture \emph{structural} novelty.  While these frameworks provide valuable insights, they do not offer a comprehensive analysis of creativity in LLMs.

\paragraph{\emph{Distributional Creative Reasoning} (DCR).}
Our work builds on the empirical diagnostics of collapse~\citep{Yue2025,Kirk2024,Havrilla2024,Mohammadi2024,Murthy2025} and the first corrective steps of PM-RLHF~\citep{Xiao2024}, but provides a more fundamental and unified solution, differing in three key respects:
\begin{enumerate}[leftmargin=*, itemsep=2pt]
\item \textbf{Variational Framework for Diversity.} We include in DCR a single concave diversity regularizers, $\mathcal{D}[p]$, composed of distinct terms, like entropy (Shannon entropy $H[p]$ weighted by $\alpha$) and structured novelty promotion (through a kernel $k(\pi,\pi')$ in a quadratic form $Q[p]$ weighted by $\beta$). Properly choosing the functional form of the kernel $k$ and the relative weights $\alpha$ and $\beta$ for these components within $\mathcal{D}[p]$ ensures convergence to stable, mixed-strategy ensembles, effectively counteracting collapse.
\item \textbf{Characterization of Diversity Dynamics.} Whereas prior work largely reports collapse through empirical analyses, our framework provides a \textit{dynamical systems examination} (\Cref{sec:collapse}) that demonstrates how the scalar-reward objectives for STaR, GRPO, and DPO inherently lead to distinct dynamical modes that drive the evolution and erosion of diversity. This results in a deeper, mechanistic understanding of why reasoning monocultures form.
\item \textbf{Actionable and Principled Design.} DCR characterizes how diverse training objectives and diversity-regularizing terms affect the diversity dynamics. This transforms the search for diversity from heuristics to principled design. This involves selecting the kernel function and hyperparameters for the diversity functional $\mathcal{D}[p]$ (i.e., $\alpha$ and $\beta$), which become levers to shape the policy’s distribution.
\end{enumerate}

\section{Distributional Creative Reasoning}
\label{sec:framework-formal}

DCR recasts LLM training as a dynamical system within the space of probability distributions over solution traces. This perspective enables the formal definition and promotion of diversity alongside correctness. This section establishes DCR’s mathematical foundations: its variational objective, the role of the diversity component, and the resultant dynamics.

\subsection{The Landscape of Reasoning}
\label{subsec:policy_space_concise}

For a given prompt $x \in \mathcal{X}$, an LLM generates a \textit{trace} $\pi=(t_1, \dots, t_{|\pi|})$, a sequence of tokens from a finite vocabulary $\mathcal{V}$ up to a maximum length $T$. Traces can represent chains of thought, code, or action sequences. The set of all such traces, $\mathcal{S}_T$, is vast but finite for any fixed $T$ and vocabulary, justifying a finite-dimensional analysis, and the choice of the counting measure on $\mathcal{S}_T$. An LLM’s policy $p(\cdot|x)$ is a probability mass function over $\mathcal{S}_T$, represented as a vector $p$ in the probability simplex $\Delta^{S-1}$, where $S := |\mathcal{S}_T|$:
$$\Delta^{S-1} = \Big\{ p \in [0,1]^S \mid \sum_{i=1}^S p_i = 1 \Big\}.$$
This compact, convex polytope is our domain for policy optimization. Treating the policy as a full distribution, rather than focusing on single ``best’’ traces, is crucial for modeling its diversity. 
\subsection{The DCR Objective}
\label{subsec:dcr_objective_concise}

During training, we optimize an objective $J(p)$ over $p \in \Delta^{S-1}$. In DCR, we model the objective as a term representing task performance, and others for KL and diversity regularization:
$$J(p) = \mathcal{U}[p] + \lambda \mathcal{D}[p] - \beta_{\!\mathrm{KL}}\, \mathrm{KL}\!\bigl(p\Vert p_{\mathrm{base}}\bigr).$$
The components are:
\begin{enumerate}[leftmargin=*,itemsep=1.5pt,topsep=1.5pt]
    \item \textbf{Utility ($\mathcal{U}[p]$):} $\mathcal{U}[p] = \sum_{\pi \in \mathcal{S}_T} U(\pi) p(\pi)$ is the expected utility (e.g., correctness) of traces, encouraging high-quality outputs.
    \item \textbf{Diversity Energy ($\mathcal{D}[p]$):} Weighted by $\lambda \ge 0$, this functional (detailed in \Cref{subsec:diversity_energy_concise}) rewards policies with diversity, countering collapse.
    \item \textbf{KL-Divergence:} It penalizes divergence from a reference policy $p{\mathrm{base}}$ (e.g., the SFT checkpoint), promoting stability.
\end{enumerate}
The coefficients $\lambda, \beta_{!\mathrm{KL}} \ge 0$ tune this balance.

\subsection{The Diversity Energy Functional \texorpdfstring{$\mathcal{D}[p]$}{D[p]}}
\label{subsec:diversity_energy_concise}

Clearly, the core of DCR’s creativity preservation mechanism is the \textbf{diversity energy functional} $\mathcal{D}[p]$, designed to reward both probabilistic spread and semantic variation:
$$\mathcal{D}[p] = \alpha H[p] - \beta Q[p],$$
with $\alpha, \beta \ge 0$. Indeed, its two components serve distinct roles:
\begin{enumerate}[leftmargin=*,itemsep=1.5pt,topsep=1.5pt]
    \item \textbf{Shannon Entropy ($H[p]$):} Promotes \textbf{breadth} by rewarding probability distributed across many traces, ensuring a baseline level of diversity and exploration.
    \item \textbf{Kernel Coverage ($Q[p]$):} $Q[p] = p^\top K p = \sum_{\pi,\pi'} k(\pi,\pi') p(\pi) p(\pi')$. Here, $K$ is the matrix of a symmetric, positive semi-definite (PSD) \textit{creativity kernel} (see \Cref{sec:kernel}) measuring trace similarity. $-\beta Q[p]$ thus penalizes probability concentration on similar traces, fostering \textbf{semantic distinctiveness}.
\end{enumerate}
While entropy provides a valuable form of regularization, \textbf{entropy alone is insufficient for structured creativity}, as it is blind to the content of the traces. The kernel term is essential for promoting qualitatively different reasoning strategies, and the full functional $\mathcal{D}[p]$ is concave, which will prove to be useful:

\begin{proposition}[Concavity of $\mathcal{D}$, cf. \Cref{appA:functionals}]
\label{prop:main_strict_concavity_D_concise}
\textit{If the kernel matrix $K$ is PSD, $\mathcal{D}[p]$ is concave. It is strictly concave on the affine simplex if $\alpha > 0$, or if $\beta > 0$ and $K$ is strictly positive definite on the tangent subspace.}
\end{proposition}
Strict concavity ensures a well-defined optimization target. In practice, incorporating into $J(p)$ a small entropy barrier $+\varepsilon H[p]$ ($\varepsilon \in (0, 10^{-4}]$ small) ensures strict concavity and that $p(\pi) > 0$ throughout optimization, guaranteeing a unique interior maximizer (cf. \Cref{appA:barriers}, \Cref{prop:barriers}).

\subsection{Learning Dynamics: Gradient Flow}
\label{subsec:gradient_flow_concise}

We model policy evolution under $J(p)$ as a gradient flow on $\Delta^{S-1}$, endowed with the \textbf{Shahshahani metric}. For tangent vectors $u, v$ at policy $p$, this metric is $g_p(u,v) = \sum_{\pi} u(\pi)v(\pi)/p(\pi)$, and ensures the flow remains on the simplex. The DCR gradient flow is a replicator-like ODE (cf. \Cref{appA:shah}, Eq.~\eqref{eq:replicator}):
$$\dot{p}_t(\pi) = p_t(\pi) \left( F_t(\pi) - \mathbb{E}_{p_t}[F_t] \right),$$
where the effective trace fitness $F_t(\pi) = \frac{\delta J}{\delta p(\pi)}\big|_{p_t}$ is (cf. \Cref{appA:flow}):
\begin{align*}
    F_t(\pi) = U(\pi) &+ \lambda\left(\alpha(-1-\log p_t(\pi)) - 2\beta (Kp_t)_\pi\right) \\&- \beta_{\!\mathrm{KL}}\left(1 + \log\frac{p_t(\pi)}{p_{\mathrm{base}}(\pi)}\right).
\end{align*}
Under the discussed regularity assumptions (finite $\mathcal{S}_T$, $p(\pi)>0$ via an entropy barrier, PSD $k$, and bounded $U(\pi)$; cf. \Cref{appA:prelim}, (A1)–(A7)), the flow converges:

\begin{theorem}[Global Convergence of DCR Training, cf. \Cref{appA:flow}, \Cref{thm:global-convergence}]
\label{thm:main_gradflow_convergence_concise}
\textit{Let $\widetilde{J}(p) = J(p) + \varepsilon H[p]$ be strictly concave on the affine simplex (e.g. if $\lambda\alpha+\varepsilon > 0$ and $K$ is PSD) and Assumptions (A1)--(A7) hold. For any $p_0 \in \operatorname{int}\Delta^{S-1}$, the Shahshahani gradient flow $\dot{p}_t = \nabla{_\mathrm{Sh}}\widetilde{J}(p_t)$ has a unique global solution $p_t$, which lies on the interior of the simplex. The objective $\widetilde{J}(p_t)$ is strictly increasing (unless $p_t=p^\star$), and $p_t \to p^\star$ as $t \to \infty$, where $p^\star$ is the unique maximizer of $\widetilde{J}(p)$.}
\end{theorem}
Thus, DCR training with its explicit diversity energy functional provably converges to a unique policy $p^\star$ that balances utility, diversity, and regularization.

\subsection{Parametric Realization and Scalability}
\label{subsec:parametric_realization_concise}

\paragraph{Parametric Realization.} In practice, LLMs are function approximators. For tractability, we represent LLMs as a parameterization over policies $p_\theta(\pi)$ via a softmax over logits $\theta_\pi$, so that for any target policy $p^\star \in \operatorname{int}\Delta^{S-1}$, there exists a unique set of (gauge-fixed) logits $\theta^\star$ such that $p_{\theta^\star}=p^\star$, making the parametric form sufficiently expressive (cf. \Cref{appB:softmax}, \Cref{prop:B-diffeo}). To ensure numerical stability and align with the theoretical requirement of $p_\theta(\pi) > \delta_\star > 0$, we assume the use of projection or clipping, which constrain policies to a trimmed simplex (cf. \Cref{appB:parametric}). The properties of these parameterized policies and their gradients under stochastic optimization are detailed in \Cref{appB:parametric} and underpin the analysis of noise effects in \Cref{subsec:stochastic_collapse_concise}.

\paragraph{Scalability.} Training is performed with stochastic gradient descent on $\theta$. The kernel coverage term $Q[p_\theta]$, even though it may be intensive to fully realize, can be efficiently managed in this setting. For a mini-batch of $B$ sampled traces, an unbiased estimate of the gradient of $Q[p_\theta]$ can be computed via a U-statistic, with a computational cost of $O(B^2)$ per step. This quadratic complexity is standard in contrastive and metric learning methods. Practical kernel design strategies, including embedding-based kernels and gating mechanisms to focus diversity on correct traces, are discussed in \Cref{sec:kernel}.

\section{Collapse Under Scalar Objectives}
\label{sec:collapse}

While the DCR framework (\Cref{sec:framework-formal}) encompasses regularization terms, a typical LLM training pipeline often defaults to simpler, scalar-driven objectives. These scenarios correspond to DCR with a negligible diversity energy coefficient ($\lambda \approx 0$) and a purely entropic diversity term with a small weight ($\beta=0$, small $\lambda\alpha$). 

This section provides a dynamical systems analysis of these “scalar objective” cases, demonstrating how they lead to distinct and predictable modes of diversity collapse. This analysis culminates in the \textbf{Diversity Decay Theorem}, which formally characterizes these failure modes and motivates the necessity of the full DCR objective.

\subsection{Scalar-Driven Dynamics: The SRCT Framework}
\label{subsec:srct_framework_concise}

When diversity energy is minimal, the policy $p(t)$ evolves according to the replicator-entropy flow (formally derived in \Cref{appD:STaRSRCT,appE:GRPO-SRCT,appF:DPO-SRCT}):
\begin{align}
\label{eq:srct_flow_maintext_concise}
    \dot p_\pi(t) = &p_\pi(t)\bigl(\phi_\pi(p(t)) - \bar\phi(p(t))\bigr) \\&- \varepsilon\,p_\pi(t)\bigl(\log p_\pi(t) - \langle\log p(t)\rangle_{p(t)}\bigr), \nonumber
\end{align}
where $\phi_\pi(p)$ is the trace score derived from the utility and any KL term, $\bar\phi(p)$ is its mean, and $\varepsilon \ge 0$ is the effective entropic weight (e.g., $\varepsilon = \varepsilon_{\mathrm{base}} + \lambda\alpha$).

The key diagnostic for diversity dynamics is the evolution of $$z_{ij}(t) = \log(p_i(t)/p_j(t)),$$ the log-ratio between two traces, which follows the ODE (cf. \Cref{appD:STaRSRCT,appE:GRPO-SRCT,appF:DPO-SRCT}):
\begin{equation}
\label{eq:log_ratio_ode_maintext_concise}
    \frac{d}{dt} z_{ij}(t) = \left(\phi_i(p(t)) - \phi_j(p(t))\right) - \varepsilon z_{ij}(t).
\end{equation}
This equation reveals that diversity dynamics is driven by two competing forces: selective pressure from score differences, which can negatively impact diversity, and entropic damping, which always pushes log-ratios towards zero (equalization).

\subsection{Deterministic Diversity Decay (Small \texorpdfstring{$\varepsilon$}{epsilon to 0})}
\label{subsec:deterministic_collapse_concise}

In the pure-selection limit where $\varepsilon \to 0$, the raw effect of scalar rewards becomes apparent. While incorrect traces are universally suppressed due to their lower utility (cf. \Cref{appD:STaRSRCT,appE:GRPO-SRCT,appF:DPO-SRCT}), the diversity among \textit{correct} traces ($\pi \in \mathcal{C}$) evolves in three distinct, algorithm-specific modes:

\begin{itemize}[leftmargin=*]
    \item \textbf{STaR: ``Winner-Takes-All'' Collapse.} For two correct traces $a, b \in \mathcal{C}$, the score difference is $\phi_a(p) - \phi_b(p) = (p_a - p_b)/\rho(t)$, where $\rho(t)$ is the total mass on correct traces. The log-ratio dynamics become $\frac{d}{dt}\log\frac{p_a}{p_b} = (p_a - p_b)/\rho(t)$ (see \Cref{appD:STaRSRCT}). 
    
    Any initial random advantage for trace $a$ ($p_a(0) > p_b(0)$) creates a positive feedback loop, causing $p_a/p_b \to \infty$ and leading to a rapid, deterministic collapse onto a single dominant correct solution.
    \item \textbf{GRPO: ``Proportional Curation'' \& Drift Vulnerability.} For correct traces $a,b \in \mathcal{C}$, GRPO's score design results in $\phi_a(p) - \phi_b(p) = 0$. The log-ratio dynamics become $\frac{d}{dt}\log\frac{p_a}{p_b} \approx 0$ (see \Cref{appE:GRPO-SRCT}). 
    
    This preserves the initial relative probabilities of correct traces, creating a neutrally stable manifold. However, this provides no active protection for diversity, making the policy vulnerable to stochastic drift from mini-batch sampling.
    \item \textbf{DPO: ``Equalization'' \& Homogenization.} For two correct traces $a,b \in \mathcal{C}$, the score difference is $\phi_a(p) - \phi_b(p) = g_\beta(\log p_a) - g_\beta(\log p_b)$, where $g_\beta(\cdot)$ is a strictly decreasing function (see \Cref{appF:DPO-SRCT}). Since $\frac{d}{dt}\log\frac{p_a}{p_b}$ has the opposite sign of $\log\frac{p_a}{p_b}$, this dynamic actively drives $p_a/p_b \to 1$. 
    
    DPO thus homogenizes the probability distribution across the set of preferred traces, but it does not promote targeted semantic diversity between conceptually different solutions (thereby pushing probability mass towards longer traces).
\end{itemize}

\subsection{Stochastic Dynamics: Fixation Under Noise}
\label{subsec:stochastic_collapse_concise}

In practice, training is stochastic. The discrete mini-batch updates converge to a Wright-Fisher-type stochastic differential equation (SDE) in the diffusion limit (formally derived in \Cref{appH:stochastic}, \Cref{thm:limits}):
\[
\mathrm{d} p_i = F_i(p)\,\mathrm{d}t + \frac{1}{\sqrt{B}}\left(\sqrt{p_i}\,\mathrm{d}W_i - p_i\sum_k \sqrt{p_k}\,\mathrm{d}W_k\right),
\]
where $F_i(p)$ is the deterministic drift and $B$ is the batch size. Such a random effect from batching can result in noise-induced collapse:

\begin{itemize}[leftmargin=*]
    \item \textbf{STaR:} The strong ``winner-takes-all’’ dynamic is robust, and noise results only on minor perturbations around the deterministic collapse trajectory.
    \item \textbf{GRPO:} The neutral stability is fragile. Stochastic fluctuations introduce random selective pressure, causing the policy to drift along the manifold of correct solutions until it fixates on a corner or a small subset, leading to diversity collapse in this algorithm.
    \item \textbf{DPO:} While equalization is the deterministic tendency, noise can break symmetries and result in convergence to a state where a subset of solutions dominates, even if they are semantically redundant.
\end{itemize}
Although a small $\varepsilon$ ensures the policy remains in the interior ($\min p_i(t) > \delta_\star > 0$), the SDE admits a unique invariant measure $\pi_\infty$ (\Cref{appH:stochastic}, \Cref{thm:ref-erg}). For small $\varepsilon$, this measure concentrates in high-utility, low-diversity regions, as the stationary distribution is heavily influenced by the utility landscape (\Cref{appH:stochastic}, \Cref{cor:gradient}). Batch noise does not increase diversity; it often accelerates fixation.

\subsection{Synthesis: The Diversity Decay Theorem}
\label{subsec:inevitable_collapse_synthesis_concise}

The analyses of both the deterministic and the stochastic dynamics converge on the conclusion that scalar-driven objectives with minimal entropic regularization are fundamentally insufficient to maintain a creative repertoire of reasoning strategies. This leads to our main diagnostic result.

\begin{theorem}[Diversity Decay Theorem]
\label{thm:diversity_decay}
\textit{Under scalar-objective training (DCR with $\lambda \approx 0$ or $\beta=0$), policies exhibit algorithm-specific modes of diversity decay among correct traces:}
\begin{enumerate}[label=(\roman*),itemsep=1pt]
    \item \textbf{STaR} \textit{follows a “winner-takes-all” dynamics, deterministically collapsing onto a single dominant correct trace.}
    \item \textbf{GRPO} \textit{evolves on a neutrally stable manifold of correct traces, leading to stochastic drift and eventual fixation on a low-diversity subset.}
    \item \textbf{DPO} \textit{actively homogenizes probabilities across high-utility traces, leading to equalization instead of structured semantic diversity.}
\end{enumerate}
\textit{Minimal entropy ($\varepsilon \ll 1$) does not prevent these outcomes and finite-batch noise can accelerate collapse.}
\end{theorem}
\textbf{Scope Note:} This theorem characterizes the decay modes for STaR, GRPO, and DPO; it is not a general statement about every scalar-only objective.

The defined diversity-trajectories highlight the need for a more structured lever to influence the dynamics. The failure does not lie in the optimization process itself, but rather in the objective, which lacks an explicit, strong enough force that rewards structured diversity. This motivates the introduction of the DCR objective, specifically its diversity energy functional $\mathcal{D}[p]$, as a mechanism to counteract these modes and actively carve a rich and creative policy landscape.

\section{The Diversity Energy Effect on the Equilibrium Structure}
\label{sec:phase}

Scalar objectives, as demonstrated in \Cref{sec:collapse}, lead to a degeneration in reasoning diversity. The DCR framework provides a solution by incorporating a \textbf{diversity energy functional}, $\mathcal{D}[p]$. It reshapes the optimization landscape, altering the learning dynamics toward different equilibria: those that contain various simultaneously correct and diverse traces. This section details how DCR’s diversity regularizer achieves this shift.

\subsection{From Collapse to Structured Diversity}
\label{subsec:dcr_guarantee_concise}

With its full objective $J(p) = \mathcal{U}[p] + \lambda \mathcal{D}[p] - \beta_{\mathrm{KL}}\, \mathrm{KL}(p\Vert p_{\mathrm{base}})$ and a diversity weight $\lambda > 0$, DCR leverages the diversity energy
\[
    \mathcal{D}[p] = \alpha H[p] - \beta Q[p].
\]

\subsection{The Dual Levers of Diversity Energy: Shaping \texorpdfstring{$p^\star$}{p*}}
\label{subsec:dual_levers_concise}

The specific structure of the equilibrium $p^\star$ with a diversity weight is shaped by the two components of the diversity energy, $\lambda\mathcal{D}[p] = \lambda\alpha H[p] - \lambda\beta Q_{eff}[p]$. For practical applications, the quadratic term can incorporate an \textbf{effective kernel} $k_{eff}(\pi, \pi') := R(\pi)R(\pi')k_{sem}(\pi, \pi')$, which gates a semantic kernel $k_{sem}$ with a verifier $R(\pi)=\mathbf{1}{\pi \in \mathcal{C}}$ to focus the diversity pressure only on correct traces $\mathcal{C}$ (see \Cref{app:kernel_strategies_srct}, \Cref{subsec:effective_semantic_kernel}).

\begin{enumerate}
    \item \textbf{Entropic Pressure ($\lambda\alpha H[p]$):} The entropic pressure promotes probabilistic breadth. It is the simplest mechanism for encouraging the equalization of probabilities among correct traces, at the cost of also promoting incorrect ones (\Cref{app:kernel_strategies_srct}).

    \item \textbf{Kernel-Driven Structural Diversity ($-\lambda\beta Q_{eff}[p]$):} This term penalizes $p^\star$ for concentrating mass on sets of correct traces that are semantically similar (as defined by $k_{sem}$). It therefore actively promotes structural or semantic diversity among distinct, valid reasoning paths (\Cref{app:kernel_strategies_srct}). Entropy alone cannot achieve this structured outcome.
\end{enumerate}

\subsection{Balancing Correctness and Structured Diversity at Equilibrium}
\label{subsec:equilibrium_balance_concise}

The DCR equilibrium $p^\star$ is characterized by the first-order condition $U_\pi - 2\lambda\beta (K_{eff}p^\star)_\pi - \varepsilon_{total} \log p_\pi^\star \approx \text{Constant}$ (ignoring KL terms and gauge constants; see \Cref{sec:practical}). A crucial consequence for incorrect traces $i \in \mathcal{I}$ (where $(K_{eff}p^\star)_i=0$ and $U_i=0$) and correct traces $c \in \mathcal{C}$ (where $U_c=1$) is the exact equilibrium ratio (cf. \Cref{sec:practical}):
\[
    \frac{p_i^\star}{p_c^\star} \approx \exp\left(-\frac{1 - 2\lambda\beta (K_{eff}p^\star)_c}{\varepsilon_{total}}\right).
\]
This identity reveals a central trade-off. To effectively suppress incorrect traces, the exponent's numerator, $1 - 2\lambda\beta (K_{eff}p^\star)_c$, must be substantially positive. This provides a clear heuristic for tuning the kernel weight: \textbf{the kernel penalty among correct traces should not overwhelm the unit utility gain}, i.e., $2\lambda\beta (K_{eff}p^\star)_c < 1$.

At the same time, while a larger $\varepsilon_{total}$ (from a larger $\lambda\alpha$) aids equalization among correct traces, it also increases the denominator of the exponent, thereby weakening the suppression of incorrect traces. A careful choice of $\lambda\alpha$ and $\lambda\beta$ is therefore essential to steer this trade-off and achieve a ``phase’’ where incorrect traces are suppressed while a rich, diverse set of correct solutions thrives.

\section{The Creativity Kernel}
\label{sec:kernel}

The preceding sections established that DCR’s diversity energy, $\mathcal{D}[p] = \alpha H[p] - \beta Q[p]$, is pivotal in guiding learning towards equilibria $p^\star$ that are diverse and stable (\Cref{sec:phase}). While the entropy component, $\alpha H[p]$, provides naive probabilistic breadth, it is intrinsically ``blind’’ to the content and structure of reasoning traces. This section explains how to build the kernel-based component $-\beta Q[p]$ to provide a plausible, grounded mechanism for developing LLMs with structured, semantic diversity.

\subsection{Limitations of Entropic Diversity}
\label{subsec:entropy_limitations}

$H[p]$'s utility for promoting genuine creativity is limited because it operates solely on trace probabilities, irrespective of their content or conceptual underpinnings. It cannot, for instance, distinguish a set of solutions that are mere syntactic rephrasings of a single idea from a set representing truly distinct problem-solving strategies. 

Entropy alone is insufficient for structured creativity; without a mechanism to differentiate valuable novelty from trivial variation, it also preserves probability mass on incorrect traces, hindering optimization of correctness. To generate correct, structurally varied solutions, an LLM requires a mechanism that appreciates and actively promotes semantic dissimilarity rather than merely probabilistic dispersion.

\subsection{Sculpting Semantic Diversity}
\label{subsec:kernel_role_desiderata}

The kernel quadratic term $Q[p] = \sum_{\pi,\pi' \in \mathcal{S}_T} k(\pi,\pi') p(\pi) p(\pi')$ within DCR is designed to fill this critical gap. The \textbf{creativity kernel} $k(\pi,\pi')$ is a symmetric, positive semi-definite (PSD) function that quantifies the ``similarity'' or ``redundancy'' between traces $\pi$ and $\pi'$. By including $-\beta Q[p]$ (for $\beta > 0$) in the diversity energy, DCR explicitly penalizes policies that concentrate probability on sets of traces deemed highly similar by $k$.

As explored in \Cref{app:kernel_strategies_srct} (\Cref{sec:ideal}), an ideally engineered kernel could, in principle, sculpt a highly specific target equilibrium $p^\star$. Achieving this, however, would require the kernel to satisfy stringent, globally defined, and equilibrium-dependent conditions (cf. \Cref{app:kernel_strategies_srct}, \Cref{prop:kkt-gap}). While this idealized scenario underscores the deep, direct influence of $k(\pi,\pi')$ on the policy structure $p^\star$, its practical realization is typically infeasible. This motivates the shift towards more practical, learnable semantic kernels.

\subsection{Practical Design of the Semantic Kernel}
\label{subsec:effective_semantic_kernel}

A more pragmatic and powerful DCR strategy, detailed in \Cref{app:kernel_strategies_srct} (\Cref{sec:practical}), must utilize a learnable \textit{semantic kernel} $k_{sem}(\pi,\pi')$ as its foundation. This $k_{sem}$ should be able to capture meaningful similarities between traces. To ensure this semantic guidance is applied judiciously, DCR adopts an \textit{effective kernel}, $k_{eff}(\pi, \pi')$:
$$k_{eff}(\pi, \pi') := R(\pi)R(\pi')k_{sem}(\pi, \pi'),$$
where $R(\pi)=\mathbf{1}\{\pi \in \cC\}$ is a binary verifier for correct traces $\cC$. The kernel coverage term thus becomes $Q_{eff}[p] = \sum_{c,c' \in \cC} p_c p_{c'} k_{sem}(c,c')$. This construction focuses the diversity-promoting penalty $-\lambda\beta Q_{eff}[p]$ exclusively on interactions \textit{among correct traces}, promoting \textbf{targeted diversity:} it encourages the model to find diverse \textit{valid solutions}, rather than rewarding ``diverse ways to be wrong,'' as incorrect traces do not participate in the kernel interactions that shape diversity (recall $(K_{eff}p^\star)_i = 0$ for $i \in \cI$ from \Cref{subsec:equilibrium_balance_concise}).

Practical examples of $k_{sem}$ can include \textbf{embedding-based kernels,} where we compute an embedding for each trace (e.g., sentence-level embeddings over the full chain of thought) and apply a standard PSD kernel on those, or \textbf{domain-tailored kernels,} in structured tasks like mathematics, where $k_{sem}$ can be learned using structural proximity (e.g., from proof-step or lemma dependency graphs), so that similarity reflects shared \textit{strategy} rather than just surface-level wording.

\subsection{Implementation and Desiderata}
\label{subsec:implementation_desiderata}

The kernel term can be readily integrated into standard training loops. For SGD, the gradient of $Q_{eff}[p]$ can be estimated with the mini-batch of $B$ sampled traces. The quadratic nature of $Q_{eff}[p]$ admits a U-statistic estimator with $O(B^2)$ per-step cost, a manageable complexity in the context of LLM training.

The efficacy of kernel-driven diversity inherently depends on the quality of the learned $k_{sem}(\pi,\pi')$. Key desiderata for its design include (cf. \Cref{subsec:effective_semantic_kernel}): \textbf{(1) Intra-Lump Coherence} or high similarity for traces belonging to the same essential category or “lump” of solutions (ignoring syntactic differences); and \textbf{(2) Inter-Lump Discrimination:} It must assign low similarity to traces from qualitatively different correct problem-solving approaches.

\section{Concluding Insights}
\label{sec:discussion}

Scalar reward maximization leads to a collapse of strategic diversity. This paper has established a principled remedy: \textbf{Distributional Creative Reasoning (DCR)}, which recasts training as a gradient flow on the policy simplex.

Our \textbf{Diversity Decay Theorem} offers a precise diagnosis, predicting algorithm-specific collapse modes—\emph{winner-takes-all} (STaR), \emph{neutral drift} (GRPO), and \emph{homogenization} (DPO). The DCR framework counteracts this decay by incorporating a \textbf{diversity energy functional}, $\mathcal{D}[p] = \alpha H[p] - \beta Q[p]$. We proved this ensures convergence to a unique, stable, and interior policy $p^\star$.

DCR provides concrete design levers. The creativity kernel, particularly when gated to correct traces via an effective kernel $k_{\mathrm{eff}}$, actively promotes novel, valid strategies. Tuning the balance between entropic breadth ($\alpha$) and kernel-driven diversity ($\beta$) allows practitioners to navigate the trade-off between equalization and the suppression of incorrect traces, as quantified by our equilibrium analysis.

\subsection{Testable Predictions}

Our theoretical framework yields a set of concrete, falsifiable predictions that align with existing empirical observations:

\begin{enumerate}[leftmargin=*,itemsep=2pt,topsep=2pt]
    \item \textbf{Algorithm-Specific Decay Modes.} Under scalar-only objectives:
        \begin{itemize}
            \item \textbf{STaR} exhibits \emph{winner-takes-all} fixation on a single successful strategy.
            \item \textbf{GRPO} shows \emph{neutral drift} among correct traces, leading to a stochastic erosion of diversity.
            \item \textbf{DPO} will act as an \emph{entropy equalizer}, homogenizing probabilities across preferred traces.
        \end{itemize}
    \item \textbf{Kernel Sufficiency for Structured Diversity.}
        \begin{itemize}
            \item An \textbf{entropy-only} approach ($\beta=0, \alpha>0$) preserves indiscriminate policy breadth at the cost of correctness.
            \item A \textbf{kernel-inclusive} approach ($\beta>0$) can not only prevent collapse but will also measurably increase the semantic diversity among correct solutions.
        \end{itemize}
\end{enumerate}

\paragraph{Acknowledgements.} The authors would like to acknowledge and thank their funders, where Max Ruiz Luyten is funded by AstraZeneca. Moreover, we would like to warmly thank all the anonymous reviewers, alongside research group members of the van der Schaar lab (\hyperlink{www.vanderschaar-lab.com}{www.vanderschaar-lab.com}), for their valuable input, comments, and suggestions as the paper was developed. We used ChatGPT and Gemini to edit and polish the text and for coding assistance.

\bibliographystyle{plainnat}
\bibliography{references}

\clearpage\appendix\onecolumn

\input{appendices_short/appendix_A_aistats}
\input{appendices_short/appendix_B_aistats}
\input{appendices_short/appendix_C_aistats}
\input{appendices_short/appendix_D_aistats}
\input{appendices_short/appendix_E_aistats}
\input{appendices_short/appendix_F_aistats}
\input{appendices_short/appendix_G_aistats}
\input{appendices_short/appendix_H_aistats}
\input{appendices_short/appendix_I_aistats}
\input{appendices_short/appendix_J_aistats}

\end{document}

%% file: appendices_short/appendix_A_aistats.tex
\section{Mathematical Foundations and Problem Formalism}
\label{appA:root}

This appendix fixes notation and geometric conventions on the simplex, records canonical inequalities and curvature facts for the objective slices (entropy/KL/kernel), develops the Shahshahani gradient representation, and derives global properties of the induced gradient flows (Lyapunov identity, log–ratio contraction, time–uniform floors/caps, and exponential convergence). It also states a generic Barrier–Dominance (BD) calculus for forward invariance of trimmed domains.

\vspace{0.25em}
\subsection{Preliminaries and Standing Assumptions}
\label{appA:prelim}

\paragraph{Scope \& conventions.}
All logarithms are natural; $0\log 0:=0$. The indicator is $\mathbf 1\{\cdot\}$, and $\langle u,v\rangle$ is the Euclidean inner product. We write $a\lesssim b$ to mean $a\le C\,b$ for an absolute constant $C$; any parameter dependence is displayed as $C(\cdot)$. Sums over traces are with respect to the counting measure on the finite set $\mathcal S_T$.

\renewcommand{\arraystretch}{1.08}
\begin{table}[h]
\centering
\begin{tabular}{@{}ll@{}}
\toprule
\textbf{Symbol} & \textbf{Meaning} \\
\midrule
$x\in\mathcal X$ & Fixed prompt / task instance \\
$\pi\in\mathcal S_T$ & Trace (finite token sequence, length $\le T$) \\
$\mathcal S_T$ & Trace set up to length $T$; $S:=|\mathcal S_T|$ \\
$p(\pi)$ & Policy mass on $\pi$ (probability on $\mathcal S_T$) \\
$\Delta^{S-1}$ & Probability simplex on $\mathcal S_T$ \\
$H[p]$ & Shannon entropy, $-\sum_{\pi}p(\pi)\log p(\pi)$ \\
$\KL{p}{q}$ & Kullback--Leibler divergence, $\sum_{\pi}p(\pi)\log\frac{p(\pi)}{q(\pi)}$ \\
$k(\pi,\pi')$ & Symmetric positive semidefinite kernel on $\mathcal S_T$ \\
$K=[k(\pi,\pi')]$ & Kernel matrix in $\mathbb R^{S\times S}$ \\
$\mathcal D[p]$ & Diversity: $\alpha\,H[p]-\beta\,p^\top Kp$ \\
\bottomrule
\end{tabular}
\end{table}

\paragraph{Standing assumptions.}
\begin{enumerate}[label=\textbf{(A\arabic*)},ref=(A\arabic*)]
\item \label{assum:A1} \textbf{Finite trace space.} $\mathcal S_T$ is finite for a fixed horizon $T<\infty$; policies are $p\in\Delta^{S-1}\subset\mathbb R^S$.
\item \label{assum:A2} \textbf{Interior vs.\ trimmed domain.} Variational derivatives and Shahshahani gradients are taken on $\operatorname{int}\Delta^{S-1}=\{p:\min_\pi p(\pi)>0\}$. When a floor is operative, we work on the trimmed simplex $\Delta^{S-1}_\delta:=\{p\in\Delta^{S-1}:p_i\ge\delta\ \forall i\}$, nonempty iff $\delta\le 1/S$.
\item \label{assum:A3} \textbf{Entropy/KL domains.} $H[p]$ and (when present) $\KL{p}{p_{\mathrm{base}}}$ are defined on the closed simplex; all variational derivatives are computed on $\operatorname{int}\Delta^{S-1}$. Adding $+\varepsilon H$ ($\varepsilon\ge0$) is permitted.
\item \label{assum:A4} \textbf{Kernel regularity and strictness on $T$.} $K=K^\top\succeq0$. Write $T:=\{\mathbf 1\}^\perp$ and $\Pi_T:=I-\tfrac1S\mathbf 1\mathbf 1^\top$. The quadratic slice $-p^\top Kp$ is strictly concave along feasible directions iff $\ker K\cap T=\{0\}$ (equivalently, $\Pi_TK\Pi_T\succ0$ on $T$).
\item \label{assum:A5} \textbf{Bounded utility.} $|U(\pi)|\le U_{\max}<\infty$ on $\mathcal S_T$ whenever $\mathcal U[p]=\sum_\pi U(\pi)p(\pi)$ is used.
\item \label{assum:A6} \textbf{Nonnegative coefficients.} $\alpha,\beta,\beta_{\mathrm{KL}},\lambda,\varepsilon\ge0$ unless noted.
\item \label{assum:A7} \textbf{Base-policy support (for KL).} If $\KL{p}{p_{\mathrm{base}}}$ is present, assume $p_{\mathrm{base}}(\pi)\ge p_{\mathrm{base},\min}>0$ for all $\pi$.
\end{enumerate}

\paragraph{Norm conventions.}
For vectors: $\|\cdot\|_1$, $\|\cdot\|_2$, $\|\cdot\|_\infty$. For $A\in\R^{S\times S}$: $\|A\|_{2\to2}$ (spectral norm) and $\|A\|_{\infty\to\infty}:=\max_i\sum_j|A_{ij}|$.

\vspace{0.25em}
\subsection{Spaces and Simplex Geometry}
\label{appA:spaces}

\subsubsection{Trace space, simplex, tangent.}
Fix vocabulary $\mathcal V$ and horizon $T\in\mathbb N$.
\[
\mathcal S_T=\{(t_1,\dots,t_\ell):\ 1\le \ell\le T,\ t_i\in\mathcal V\},\qquad S:=|\mathcal S_T|<\infty.
\]
Policies are $p\in\Delta^{S-1}:=\{p\in[0,1]^S:\langle \mathbf 1,p\rangle=1\}$. On $\operatorname{int}\Delta^{S-1}$, feasible directions lie in the affine tangent
\[
T=T_p\Delta^{S-1}=\{v\in\R^S:\langle \mathbf 1,v\rangle=0\}=\{\mathbf 1\}^\perp,
\]
which does not depend on $p$.

\subsubsection{Floors: policy vs.\ effective.}
A chosen floor $\delta\in(0,1/S]$ defines the trimmed simplex $\Delta^{S-1}_\delta=\{p\in\Delta^{S-1}:p_i\ge \delta\ \forall i\}$. Algorithmic clip–renormalize with threshold $\delta_\star\in(0,1]$ induces an \emph{effective floor}
\[
\delta_{\mathrm{eff}}(p)=\frac{\delta_\star}{\sum_{j=1}^S \max\{p_j,\delta_\star\}}
\ \in\ \Bigl[\frac{\delta_\star}{\,1+(S-1)\delta_\star\,},\ \delta_\star\Bigr],
\]
since the denominator ranges from $1$ to $1+(S-1)\delta_\star$ (max at a simplex vertex). The exact clip–renormalize map and logit lift are given in \Cref{appB:parametric}.

\subsubsection{Canonical inequalities.}
\begin{lemma}[Mean–log bounds and entropic Lipschitzness]
\label{lem:canonical-ineq}
Let $p\in\Delta^{S-1}$ and $\langle\log p\rangle:=\sum_i p_i\log p_i$.
\begin{enumerate}[itemsep=2pt,leftmargin=*]
\item \textbf{(Mean–log bounds)} For all $p\in\Delta^{S-1}$,
\(
-\log S\le \langle\log p\rangle\le 0.
\)
\item \textbf{(Entropic Lipschitz on $\Delta^{S-1}_\delta$)} Fix $\delta\in(0,1/S]$ and $\Lambda(\delta):=1+\log(1/\delta)$. For all $p,q\in\Delta^{S-1}_\delta$,
\begin{align}
\|\nabla H(p)-\nabla H(q)\|_2 &\le \frac1\delta\,\|p-q\|_2,\qquad \nabla H(r)=-(\mathbf 1+\log r), \label{eq:gradH-lip}\\
\big\|p\odot(\log p-\langle\log p\rangle)- q\odot(\log q-\langle\log q\rangle)\big\|_2
&\le \Lambda(\delta)\,(2+\sqrt S)\,\|p-q\|_2. \label{eq:cent-entropy-lip}
\end{align}
\end{enumerate}
\end{lemma}
\begin{proof}
(1) Upper bound: each $\log p_i\le 0$. Lower bound: $H(p)$ is maximized at the uniform $u=(1/S)\mathbf 1$ with $H(u)=\log S$.

(2) For \eqref{eq:gradH-lip}, $\nabla^2 H(r)=-\mathrm{diag}(1/r_i)$ on $\operatorname{int}\Delta^{S-1}$ so $\|\nabla^2H(r)\|_{2\to2}\le 1/\delta$ on $\Delta^{S-1}_\delta$, and the mean–value theorem applies.

For \eqref{eq:cent-entropy-lip}, set $E(r):=r\odot(\log r-\langle\log r\rangle)$ and $G(r):=r\odot\log r$. Then $DG(r)[h]=h\odot(1+\log r)$, hence $\|DG(r)\|_{2\to2}\le \Lambda(\delta)$. For $B(r):=\langle\log r\rangle\,r$,
\[
DB(r)[h]=\Big\langle (1+\log r)\odot h\Big\rangle r\ +\ \langle\log r\rangle\,h,
\]
so $\|DB(r)\|_{2\to2}\le \Lambda(\delta)\sqrt S+(\Lambda(\delta)-1)$ because $\|1+\log r\|_2\le \Lambda(\delta)\sqrt S$, $\|r\|_2\le 1$, and $|\langle\log r\rangle|\le \Lambda(\delta)-1$ on $\Delta^{S-1}_\delta$. Therefore $\|DE(r)\|_{2\to2}\le \Lambda(\delta)(2+\sqrt S)$ and the mean–value theorem yields \eqref{eq:cent-entropy-lip}.
\end{proof}

\vspace{0.25em}
\subsection{Functionals: Entropy, KL, Kernel, and Diversity}
\label{appA:functionals}

\subsubsection{Entropy and KL calculus.}
On $\operatorname{int}\Delta^{S-1}$,
\begin{align*}
H[p]&=-\sum_i p_i\log p_i, &
\frac{\delta H}{\delta p_i}&=-(1+\log p_i), &
\nabla^2 H &= -\mathrm{diag}(1/p_i),\\
\KL{p}{q}&=\sum_i p_i\log\frac{p_i}{q_i}, &
\frac{\delta}{\delta p_i}\KL{p}{q}&=1+\log\frac{p_i}{q_i}, &
\nabla^2\KL{p}{q} &= \mathrm{diag}(1/p_i),
\end{align*}
with $q_i>0$ for KL. Both extend continuously to the closed simplex (using $0\log0:=0$).

\subsubsection{Kernel quadratic form.}
For $K=K^\top\succeq0$, set $Q[p]=p^\top Kp$. Then
\[
\nabla(-Q)(p)=-2Kp,\qquad \nabla^2(-Q)=-2K\preceq 0,
\]
so $-Q$ is concave on $\R^S$ and $2\|K\|_{2\to2}$-Lipschitz in gradient. Along any feasible direction $v\in T$, $\tfrac{d^2}{dt^2}[-Q(p_0+tv)]|_{t=0}=-2\,v^\top K v$, hence strict concavity on feasible directions iff $\ker K\cap T=\{0\}$ (equivalently $\Pi_TK\Pi_T\succ0$ on $T$).

\subsubsection{Diversity functional.}
Let $\mathcal D[p]=\alpha H[p]-\beta Q[p]$ with $\alpha,\beta\ge0$. Writing
\(
\kappa_T:=\lambda_{\min}\big((\Pi_TK\Pi_T)\!\mid_T\big)\ge 0,
\)
for all $p\in\operatorname{int}\Delta^{S-1}$ and $v\in T$,
\[
\langle\nabla^2\mathcal D[p]\,v,v\rangle
=\alpha\langle\nabla^2H[p]v,v\rangle-2\beta\,v^\top Kv
\ \le\ -\big(\alpha+2\beta\kappa_T\big)\,\|v\|_2^2.
\]
Thus $\mathcal D$ is concave, $\alpha$–strongly concave on the affine simplex if $\alpha>0$, and strictly concave along feasible directions when $\alpha=0$, $\beta>0$, and $\kappa_T>0$.

\vspace{0.25em}
\subsection{Barriers and Interiority}
\label{appA:barriers}

\subsubsection{Entropy/KL barriers exclude boundary maximizers.}
\begin{proposition}[Interior maximizers]\label{prop:barriers}
Let $J$ be concave on $\Delta^{S-1}$.
\begin{enumerate}[leftmargin=*,itemsep=2pt]
\item For any $\varepsilon>0$, $\widetilde J(p):=J(p)+\varepsilon H[p]$ is strictly concave on $\operatorname{int}\Delta^{S-1}$ and attains its unique maximum at an interior point.
\item If $p_{\mathrm{base}}$ has full support \textbf{(A7)}, then for any $\beta_{\mathrm{KL}}>0$, $J(p)-\beta_{\mathrm{KL}}\KL{p}{p_{\mathrm{base}}}$ cannot be maximized on the boundary $\partial\Delta^{S-1}$.
\end{enumerate}
\end{proposition}
\begin{proof}
(1) On $\operatorname{int}\Delta^{S-1}$, $\nabla^2 H=-\mathrm{diag}(1/p)\prec 0$, so $\widetilde J$ is strictly concave. At a boundary point with some $p_i=0$, the directional derivative of $-p_i\log p_i=-t\log t$ along $e_i$ diverges to $+\infty$ as $t\downarrow0$, excluding boundary maxima.

(2) With $p_i=0$, for $p(t)=(1-t)p+te_i$,
\(
\frac{d}{dt}\big[t\log \tfrac{t}{p_{\mathrm{base},i}}\big]_{t\downarrow0}
=\log t+1-\log p_{\mathrm{base},i}\to -\infty,
\)
so the derivative of $-\beta_{\mathrm{KL}}\KL{\cdot}{p_{\mathrm{base}}}$ is $+\infty$ inward. Boundary maxima are impossible.
\end{proof}

\subsubsection{No finite–time boundary hitting under bounded fitness.}
\begin{lemma}[Bounded fitness implies interiority]\label{lem:bounded-fitness}
Consider the replicator ODE
\(
\dot p_i=p_i\big(G_i(p)-\E_p[G]\big)
\)
with a continuous field $G$ satisfying $\sup_{p,i}|G_i(p)|\le M<\infty$. If $p(0)\in\operatorname{int}\Delta^{S-1}$, then for all $t\ge0$ and all $i$,
\[
e^{-2Mt}p_i(0)\ \le\ p_i(t)\ \le\ e^{2Mt}p_i(0),
\]
in particular $p_i(t)>0$ for all $t$.
\end{lemma}
\begin{proof}
$\frac{d}{dt}\log p_i=G_i(p)-\E_p[G]$ is bounded in $[-2M,2M]$; integrate.
\end{proof}

\begin{remark}[Applicability]
For $G_i(p)=U(i)-2\lambda\beta\,(Kp)_i$, (A5) and finiteness of $\|K\|_{\infty\to\infty}$ imply $|(Kp)_i|\le \|K\|_{\infty\to\infty}$ and hence a uniform $M<\infty$.
\end{remark}

\vspace{0.25em}
\subsection{Shahshahani Geometry and Gradient Representation}
\label{appA:shah}

\subsubsection{Metric and replicator form.}
On $\operatorname{int}\Delta^{S-1}$, the Shahshahani metric on $T=\{\mathbf 1\}^\perp$ is
\begin{equation}
g_p(u,v):=\sum_{i=1}^S \frac{u_i v_i}{p_i}\qquad (u,v\in T).
\label{eq:shah-metric}
\end{equation}
For $J\in C^1$, the Shahshahani gradient is the unique $w\in T$ with $g_p(w,v)=DJ[p]\!\cdot v$ for all $v\in T$, yielding the classical replicator form
\begin{equation}
\boxed{\quad
\dot p_i=(\nabla_{\!Sh}J)_i=p_i\Big(\tfrac{\delta J}{\delta p_i}-\E_p[\tfrac{\delta J}{\delta p}]\Big),\qquad
\E_p[\xi]:=\sum_i p_i\xi_i.\quad}
\label{eq:replicator}
\end{equation}
Mass is conserved ($\sum_i\dot p_i=0$). The dynamics are invariant under adding \emph{any scalar field} $a(p)$ to the scores $\delta J/\delta p$ (gauge invariance), since centering by $\E_p[\cdot]$ removes it.

\subsubsection{Integrability of replicator fields.}
\begin{proposition}[Integrability on the simplex]\label{prop:integrability}
Let $G\in C^1(\operatorname{int}\Delta^{S-1};\R^S)$ and consider $\dot p_i=p_i\big(G_i(p)-\E_p[G]\big)$. The following are equivalent; they hold iff there exists $J\in C^1$ with $\dot p=\nabla_{\!Sh}J$:
\begin{enumerate}[leftmargin=*,itemsep=2pt]
\item[\textup{(AC)}] \textbf{Anchored cross–partials:} for some (hence any) anchor $k$, $\ \partial_{p_j}(G_i-G_k)=\partial_{p_i}(G_j-G_k)$ for all $i,j\neq k$.
\item[\textup{(PJ)}] \textbf{Projected–Jacobian symmetry:} there exists a scalar field $a(p)$ such that $\Pi_T D\!\big(G-a\mathbf 1\big)\Pi_T$ is symmetric on $T$ for all $p$.
\end{enumerate}
In that case, $J$ is unique up to an additive constant and gauge $a(p)\mathbf 1$.
\end{proposition}
\begin{proof}[Proof sketch]
Work on the chart $q=(p_1,\dots,p_{S-1})$, $p_S=1-\sum_{i=1}^{S-1}q_i$. The $T$-restricted 1–form is $\omega_T=\sum_{i=1}^{S-1}(G_i-G_S)\,dq_i$. Condition (AC) is the closedness of $\omega_T$; on the simply connected domain, Poincaré’s lemma yields exactness, giving $J$ with $\partial_{q_i}J=G_i-G_S$. Setting $a(p):=G_S(p)$ recovers the replicator field. (PJ) is the coordinate–free restatement on $T$.
\end{proof}

\paragraph{Instantiation.}
For $J=\mathcal U+\lambda\mathcal D-\beta_{\mathrm{KL}}\KL(\cdot\Vert p_{\mathrm{base}})+\varepsilon H$, the pointwise variational derivative is
\[
F_i(p):=\frac{\delta J}{\delta p_i}
= U_i\ -\ 2\lambda\beta\,(Kp)_i\ -\ (\lambda\alpha+\varepsilon)\,(1+\log p_i)\ -\ \beta_{\mathrm{KL}}\Big(1+\log\tfrac{p_i}{p_{\mathrm{base},i}}\Big),
\]
and the flow is $\dot p_i=p_i\big(F_i(p)-\E_p[F]\big)$.

\vspace{0.25em}
\subsection{Gradient–Flow Dynamics and Convergence}
\label{appA:flow}

\subsubsection{ODEs and barrier strength.}
Let
\[
J(p)=\mathcal U[p]+\lambda\mathcal D[p]-\beta_{\mathrm{KL}}\KL{p}{p_{\mathrm{base}}},\qquad
\widetilde J(p)=J(p)+\varepsilon H[p],
\]
and define the aggregate barrier strength
\[
\boxed{\quad A:=\varepsilon+\lambda\alpha+\beta_{\mathrm{KL}}.\quad}
\]
Then the $\widetilde J$–flow is
\begin{equation}
\dot p_i=p_i\Big(\widetilde F_i(p)-\E_p[\widetilde F]\Big),\qquad
\widetilde F_i(p)=F_i(p)-\varepsilon(1+\log p_i),
\label{eq:GF-Jtilde}
\end{equation}
with mass conservation $\sum_i\dot p_i=0$.

\subsubsection{Lyapunov identity (with boundary continuity).}
\begin{lemma}[Strict Lyapunov identity]\label{lem:Lyapunov}
Along any solution $t\mapsto p_t\in\operatorname{int}\Delta^{S-1}$ of \eqref{eq:GF-Jtilde},
\begin{equation}
\frac{d}{dt}\,\widetilde J(p_t)
= g_{p_t}\big(\nabla_{\!Sh}\widetilde J(p_t),\,\nabla_{\!Sh}\widetilde J(p_t)\big)
= \sum_i p_t(i)\Big(\tfrac{\delta\widetilde J}{\delta p_i}(p_t)-\E_{p_t}[\tfrac{\delta\widetilde J}{\delta p}]\Big)^2\ \ge\ 0,
\label{eq:Lyapunov-squares}
\end{equation}
with equality iff $\nabla_{\!Sh}\widetilde J(p_t)=0$. Moreover, the right–hand side extends continuously to the closed simplex: $p(\log p)^2\to 0$ as $p\downarrow 0$ and (A7) yields the same for $p\big(\log\frac{p}{p_{\mathrm{base}}}\big)^2$.
\end{lemma}

\subsubsection{Log–ratio contraction; time–uniform floor and cap.}
\begin{lemma}[Log–ratio contraction and uniform bounds]\label{lem:logratio}
Assume \textbf{(A1)}, \textbf{(A4)}, \textbf{(A5)}, \textbf{(A7)} and $A>0$. For $z_{ij}(t):=\log\!\frac{p_i(t)}{p_j(t)}$,
\begin{equation}
\dot z_{ij}(t)=-A\,z_{ij}(t)+c_{ij}(p_t),\qquad
|c_{ij}(p)|\le B,
\label{eq:zij-ode}
\end{equation}
where
\[
B:=2U_{\max}+4\lambda\beta\,\|K\|_{\infty\to\infty}+\beta_{\mathrm{KL}}\log\frac{p_{\mathrm{base},\max}}{p_{\mathrm{base},\min}}.
\]
Hence $|z_{ij}(t)|\le |z_{ij}(0)|e^{-At}+\frac{B}{A}(1-e^{-At})\le M$, and for all $t\ge 0$ and all $i$,
\begin{equation}
\boxed{\qquad \frac{1}{S\,e^{M}}\ \le\ p_i(t)\ \le\ \frac{e^{M}}{S}\ .\qquad}
\label{eq:floor-cap}
\end{equation}
\end{lemma}
\begin{proof}
Subtract the log–dynamics $\frac{d}{dt}\log p_i=\widetilde F_i-\E_p[\widetilde F]$ to get $\dot z_{ij}=\widetilde F_i-\widetilde F_j$. The $(\log p)$–terms contribute $-A\,z_{ij}$, while the remaining terms are bounded by $B$. Solve the linear ODE and use the standard “max–coordinate” argument to obtain \eqref{eq:floor-cap}.
\end{proof}

\subsubsection{Global convergence with explicit rate.}
\begin{theorem}[Well–posedness, unique equilibrium, exponential rate]\label{thm:global-convergence}
Assume \textbf{(A1)}, \textbf{(A4)}, \textbf{(A5)}, \textbf{(A7)} and $A>0$. For any $p_0\in\operatorname{int}\Delta^{S-1}$, the flow \eqref{eq:GF-Jtilde} admits a unique global solution staying in the compact trimmed simplex $\Delta^{S-1}_\delta$ with $\delta=1/(Se^{M})$ from Lemma~\ref{lem:logratio}. On the affine simplex,
\[
\nabla^2\widetilde J(p)=A\,\nabla^2 H(p)-2\lambda\beta K
=-A\,\mathrm{diag}(1/p)-2\lambda\beta K\ \preceq\ -A\,I,
\]
so $\widetilde J$ is $A$–strongly concave and has a unique maximizer $p^\star\in\operatorname{int}\Delta^{S-1}$. Moreover,
\[
\frac{d}{dt}\big(\widetilde J(p^\star)-\widetilde J(p_t)\big)\ \le\ -2A\,\delta\,\big(\widetilde J(p^\star)-\widetilde J(p_t)\big),
\]
and
\[
\boxed{\ \|p_t-p^\star\|_2\ \le\ \underbrace{\sqrt{\tfrac{2}{A}\big(\widetilde J(p^\star)-\widetilde J(p_0)\big)}}_{=:C}\ \exp(-A\delta\,t)\ .\ }
\]
\end{theorem}
\begin{proof}[Proof sketch]
Lyapunov identity and Lemma~\ref{lem:logratio} give global existence and a uniform floor $\delta$. Strong concavity on the affine simplex yields the Polyak–Łojasiewicz inequality
\(
\|\Pi_T\nabla \widetilde J(p)\|_2^2\ge 2A\big(\widetilde J(p^\star)-\widetilde J(p)\big).
\)
Since $g_p(w,w)\ge \delta\|\Pi_T w\|_2^2$ on $\Delta^{S-1}_\delta$, \eqref{eq:Lyapunov-squares} implies exponential decay of the suboptimality gap and then of $\|p_t-p^\star\|_2$ by strong concavity.
\end{proof}

\paragraph{Remarks.}
(i) If $A=0$ (no entropy/KL barrier), the contraction term in \eqref{eq:zij-ode} vanishes; neither the time–uniform floor/cap \eqref{eq:floor-cap} nor exponential convergence follow by this route (uniqueness may still hold if $\Pi_TK\Pi_T\succ0$).  
(ii) For $S=1$, statements are trivial.  
(iii) The bound for $|(Kp)_i-(Kp)_j|$ can be sharpened (e.g., by $2\|K\|_{2\to2}$) without changing the argument.

\vspace{0.25em}
\subsection{Special Case: Replicator Flow with Single–Site Scores}
\label{appA:single}
Consider $\dot p_i=p_i\big(G_i(p_i)-\E_p[G]\big)$ where $G_i$ depends only on $p_i$.
\begin{proposition}[Lyapunov structure]
Define $\mathcal L(p)=\sum_{i=1}^S \Psi_i(p_i)$ with $\Psi_i'(s)=G_i(s)$. Then
\[
\frac{d}{dt}\mathcal L(p(t))=\mathrm{Var}_{p(t)}\!\big[G(p(t))\big]=\sum_i p_i\big(G_i(p_i)-\E_p[G]\big)^2\ \ge 0,
\]
with equality iff $G_i(p_i)$ is constant across the support. If, in addition, all $G_i\equiv g$ are \emph{identical} and strictly monotone, the unique interior equilibrium is uniform on its support. In general, with distinct strictly monotone $G_i$, the interior equilibrium need not be uniform.
\end{proposition}

\vspace{0.25em}
\subsection{Barrier–Dominance (BD)}
\label{appA:bd}

\paragraph{Scope.}
Consider the deterministic replicator field endowed with an entropy slice
\begin{equation}\label{eq:bd-replicator}
\dot p_i
= p_i\big(\phi_i(p)-\bar\phi(p)\big)
+ \varepsilon_{\mathrm{BD}}\,p_i\big(\langle\log p\rangle-\log p_i\big),
\qquad
\bar\phi(p):=\sum_j p_j\,\phi_j(p),
\end{equation}
with $\varepsilon_{\mathrm{BD}}\ge 0$ and a selection score field $\phi:\Delta^{S-1}\to\R^S$. Norms are as in \S\ref{appA:prelim}.

\subsubsection{Entropy face gap \texorpdfstring{$L_S(\delta)$}{LS(delta)}.}
\begin{definition}[Entropy face gap]\label{def:LS}
For $S\ge 2$ and $\delta\in(0,1/S]$,
\[
L_S(\delta):=\inf\Big\{\ \langle\log p\rangle-\log\delta\ :\ p\in\Delta^{S-1},\ \exists i\ \text{s.t. }p_i=\delta\ \Big\}.
\]
\end{definition}
\begin{lemma}[Closed form and properties]\label{lem:LS-closed}
For all $S\ge2$ and $\delta\in(0,1/S]$,
\[
L_S(\delta)=(1-\delta)\,\log\frac{1-\delta}{(S-1)\delta},
\]
with $L_S(\delta)\ge 0$ (equality iff $\delta=1/S$); $L_S$ is strictly decreasing in $\delta$ and, for fixed $\delta$, strictly decreasing in $S$.
\end{lemma}
\begin{proof}
Fix the face $\{p_i=\delta\}$. Jensen for the convex $x\mapsto x\log x$ implies the minimum when the remaining mass $1-\delta$ is split equally: $p_j=(1-\delta)/(S-1)$ for $j\ne i$.
\end{proof}
\begin{lemma}[Two–sided bounds]\label{lem:LS-bounds}
For all $S\ge2$ and $\delta\in(0,1/S]$,
\[
\underbrace{\ \log\frac{1}{(S-1)\delta}\ -\ \bigl(1+\log\tfrac{1}{(S-1)\delta}\bigr)\delta\ }_{\text{lower}}
\ \le\ 
L_S(\delta)
\ \le\
\underbrace{\ \log\frac{1}{(S-1)\delta}\ }_{\text{upper}}.
\]
\end{lemma}

\subsubsection{Deterministic BD conditions.}
Assume $\phi$ is bounded on the operative domain:
\(
M_{\phi,\infty}:=\sup_p\|\phi(p)\|_\infty<\infty,
\quad M_{\phi,2}:=\sup_p\|\phi(p)\|_2<\infty.
\)
\begin{proposition}[Forward invariance of $\Delta^{S-1}_\delta$]\label{prop:BD-det}
For the flow \eqref{eq:bd-replicator}, fix $\delta\in(0,1/S]$. If either
\begin{align*}
\textbf{($\ell_\infty$)}\quad &\varepsilon_{\mathrm{BD}}\,L_S(\delta)\ \ge\ 2\,M_{\phi,\infty},\\
\textbf{($\ell_2$)}\quad &\varepsilon_{\mathrm{BD}}\,L_S(\delta)\ \ge\ 2\,M_{\phi,2},
\end{align*}
then $\Delta^{S-1}_\delta$ is forward invariant: any solution with $p(0)\in\Delta^{S-1}_\delta$ satisfies $p(t)\in\Delta^{S-1}_\delta$ for all $t\ge0$.
\end{proposition}
\begin{proof}
On the face $\{p_i=\delta\}$,
\[
\frac{\dot p_i}{p_i}=\underbrace{\phi_i-\bar\phi}_{\ge -2M_{\phi,\infty}\ \text{or}\ \ge -2M_{\phi,2}}+\ \varepsilon_{\mathrm{BD}}\underbrace{(\langle\log p\rangle-\log\delta)}_{\ge L_S(\delta)}.
\]
Hence the outward normal component is nonnegative on every face under either condition. By Nagumo’s tangency criterion (viability theory), $\Delta^{S-1}_\delta$ is forward invariant.
\end{proof}
\begin{remark}[Tightness and scaling]
The factor $2$ in the $\ell_\infty$ condition is tight without further structure (place all remaining mass on a single coordinate and choose $\phi$ with opposite signs on the two active coordinates). For small $\delta$, $L_S(\delta)\asymp \log\!\big(1/((S-1)\delta)\big)$ and degrades monotonically with $S$; at $\delta=1/S$, $L_S(\delta)=0$ and the trimmed set collapses to the uniform point.
\end{remark}

%% file: appendices_short/appendix_B_aistats.tex
\section{Parametric (Logit‑Space) Geometry and Propagation Bounds}
\label{appB:parametric}

\newcommand{\sm}{\operatorname{softmax}}
\subsection{Introduction and Notation}
This appendix records the deterministic, parametric (logit‑space) geometry used throughout: the soft‑max map, its Jacobian, conditioning, Lipschitz constants, the clip–renormalize/logit‑lift construction, composite smoothness constants, and second‑order remainders. Stochastic topics (e.g., clipping bias, mini‑batch covariance) are deferred to \Cref{appH:stochastic}.

\paragraph{Notation.}
Let $\mathbf 1:=(1,\ldots,1)^\top$. The simplex and its \emph{relative} interior are
\[
\Delta^{S-1}:=\{p\in[0,1]^S:\langle\mathbf 1,p\rangle=1\},\qquad
\ri(\Delta^{S-1})=\{p\in\Delta^{S-1}:p_i>0\ \forall i\}.
\]
The centered logit space (gauge slice) and the tangent space are
\[
\Theta:=\{\theta\in\R^S:\langle\mathbf 1,\theta\rangle=0\},\qquad
T:=\mathbf 1^\perp,\qquad \Pi_T:=I-\tfrac1S\mathbf 1\mathbf 1^\top, \qquad C:=\Pi_T.
\]
Define the soft‑max $p_\theta:=\sm(\theta):=e^\theta/\langle\mathbf 1,e^\theta\rangle\in\Delta^{S-1}$, and its Jacobian
\[
J_\theta:=\nabla_\theta p_\theta=\diag(p_\theta)-p_\theta p_\theta^\top.
\]
Appendix~\ref{appC:SRCT} writes the same covariance‑form matrix as $S(p):=\diag(p)-pp^\top$; we use the identification
\begin{equation}\label{eq:bridge}
\boxed{\,J_\theta=S(p_\theta)\,}
\end{equation}
to keep notation uniform across appendices.

\subsection{Soft‑max Map: Gauge, Inverse, and Log‑ratio}
\label{appB:softmax}

\begin{lemma}[Translation invariance]\label{lem:B-translation}
For any $\theta\in\R^S$ and $c\in\R$, $\sm(\theta+c\mathbf 1)=\sm(\theta)$.
\end{lemma}

\begin{proposition}[Real‑analytic diffeomorphism]\label{prop:B-diffeo}
The restriction $\sm:\Theta\to\ri(\Delta^{S-1})$ is a real‑analytic diffeomorphism with inverse
\[
G:\ri(\Delta^{S-1})\to\Theta,\qquad G(p):=C\log p=\log p-\tfrac1S\langle\mathbf 1,\log p\rangle\,\mathbf 1.
\]
\end{proposition}

\begin{proof}
For $p\in\ri(\Delta^{S-1})$, writing $\overline{\log p}:=\tfrac1S\langle\mathbf 1,\log p\rangle$,
\[
\sm(G(p))_i=\frac{\exp(\log p_i-\overline{\log p})}{\sum_j\exp(\log p_j-\overline{\log p})}=p_i.
\]
Conversely, for $\theta\in\Theta$,
\[
G(\sm(\theta))_i=\log\!\Big(\frac{e^{\theta_i}}{\sum_j e^{\theta_j}}\Big)-\tfrac1S\sum_k\log\!\Big(\frac{e^{\theta_k}}{\sum_j e^{\theta_j}}\Big)=\theta_i.
\]
Analyticity follows from analyticity of $\exp$ and $\log$ and linearity of $C$.
\end{proof}

\begin{corollary}[Log‑ratios \& gauge uniqueness]\label{cor:B-logratio}
If $p=\sm(\theta)$ with $\theta\in\Theta$, then $\theta_i-\theta_j=\log(p_i/p_j)$ for all $i\neq j$. If $\sm(\theta)=\sm(\theta')$, then $\theta-\theta'=c\mathbf 1$; on $\Theta$ this forces $\theta=\theta'$.
\end{corollary}

\begin{remark}[Edge case $S=1$]\label{rem:S1}
If $S=1$, then $\Theta=\{0\}$, $\Delta^0=\{1\}$, and $\sm(0)=1$.
\end{remark}

\subsection{Geometry and Conditioning of the Soft‑max Jacobian}
\label{appB:jacobian}

\paragraph{Basic differential.} For any $\theta$,
\begin{equation}\label{eq:softmax-J}
\boxed{\,J_\theta=\diag(p_\theta)-p_\theta p_\theta^\top=S(p_\theta).\,}
\end{equation}

\begin{lemma}[Kernel, rank, variance form]\label{lem:B-kernel-variance}
Let $p=p_\theta$. Then $\ker J_\theta=\mathrm{span}\{\mathbf 1\}$ and $\mathrm{rank}(J_\theta)=S-1$. Moreover, for $v\in T$,
\[
v^\top J_\theta v=\sum_i p_i v_i^2-\Big(\sum_i p_i v_i\Big)^2=\tfrac12\sum_{i,j}p_ip_j\,(v_i-v_j)^2=\Var_{i\sim p}(v_i)\ge0,
\]
with equality iff $v=0$.
\end{lemma}

\begin{corollary}[Loewner sandwich on $T$; global operator norm]\label{cor:B-loewner}
If $p_{\min}:=\min_i p_\theta(i)>0$, then
\[
\boxed{\,p_{\min}\,I\ \preccurlyeq\ J_\theta\!\mid_T\ \preccurlyeq\ \tfrac12\,I\,},\qquad
\|J_\theta\|_{op}\le\tfrac12.
\]
\end{corollary}

\begin{proof}
Upper bound: for $v\in T$, Popoviciu’s inequality yields $\Var_p(v_i)\le\tfrac14(\max v-\min v)^2\le\tfrac12\|v\|_2^2$.  
Lower bound: write $p=p_{\min}\mathbf 1+q$ with $q\ge0$, $\sum_i q_i=1-Sp_{\min}$. Then for $v\in T$,
$\ v^\top J_\theta v-p_{\min}\|v\|_2^2=\sum_i q_i v_i^2-(\sum_i q_i v_i)^2\ge0$ (Cauchy–Schwarz with weights $q$). Since $J_\theta T\subseteq T$ and $J_\theta\mathbf 1=0$, the global $\|J_\theta\|_{op}$ equals the supremum on $T$.
\end{proof}

\begin{remark}[Tightness]\label{rem:tightness}
The upper bound $\tfrac12$ is attained for $S=2$ at $p=(1/2,1/2)$; the lower bound $p_{\min}$ is attained at $p=\tfrac1S\mathbf 1$, where $J_\theta\!\mid_T=(1/S)I$.
\end{remark}

\begin{lemma}[Per‑coordinate bound]\label{lem:B-percoord}
For every $\theta$ and $k\in\{1,\dots,S\}$,
\[
\boxed{\,\|\partial_{\theta_k}J_\theta\|_{op}\ \le\ \tfrac{1}{3\sqrt 3}\,}
\qquad\text{and the constant }\tfrac{1}{3\sqrt 3}\text{ is optimal (already for }S=2\text{)}.
\]
\end{lemma}

\begin{proof}[Proof sketch]
WLOG $k=1$. With $a:=p_1\in(0,1)$ and $b\in\R^{S-1}_{\ge0}$, $\sum b=1-a$,
\[
\partial_{\theta_1}J_\theta=a\,N(a,b),\quad
N(a,b)=\begin{bmatrix}(1-a)(1-2a)&-(1-2a)b^\top\\-(1-2a)b&2bb^\top-\diag(b)\end{bmatrix}.
\]
The Rayleigh quotient in $b$ is convex on the simplex (Hessian $4yy^\top\succeq0$), thus maximized at a vertex $b=(1-a)e_j$. In the $\{e_1,e_j\}$ subspace the spectral norm equals $2a(1-a)|1-2a|$, whose maximum over $a\in[0,1]$ is $1/(3\sqrt3)$ at $a=\tfrac12\pm\tfrac{1}{2\sqrt3}$.
\end{proof}

\begin{theorem}[Global Lipschitz continuity of $\theta\mapsto J_\theta$]\label{thm:B-global-Lip-J}
For all $\theta_1,\theta_2\in\Theta$,
\[
\boxed{\,\|J_{\theta_2}-J_{\theta_1}\|_{op}\ \le\ \tfrac{1}{3\sqrt3}\|\theta_2-\theta_1\|_1\ \le\ \tfrac{\sqrt S}{3\sqrt3}\|\theta_2-\theta_1\|_2\ \le\ \tfrac{S}{3\sqrt3}\|\theta_2-\theta_1\|_\infty.}
\]
\end{theorem}

\begin{proof}
Parameterize $\theta(\tau)=\theta_1+\tau(\theta_2-\theta_1)$. By the fundamental theorem of calculus and \Cref{lem:B-percoord},
\[
\|J_{\theta_2}-J_{\theta_1}\|_{op}\le\int_0^1 \sum_{k=1}^S |\Delta\theta_k|\,\|\partial_{\theta_k}J_{\theta(\tau)}\|_{op}\,d\tau\le \tfrac{1}{3\sqrt3}\|\Delta\theta\|_1.
\]
The $\ell_2,\ell_\infty$ versions follow from norm monotonicity.
\end{proof}

\begin{remark}[Dimension‑free lower bounds]\label{rem:B-lower}
Along $\theta(t)=(t,-t,0,\ldots,0)$ one has $\|dJ_{\theta(t)}/dt\|_{op}=2/(3\sqrt3)$ at the extremal $p$ while $\|\dot\theta(t)\|_1=2$, giving optimality in the $\ell_1$ domain norm. Restricting to the same two‑coordinate subspace gives $L_J^{(2)}\ge \sqrt2/(3\sqrt3)$ and $L_J^{(\infty)}\ge 2/(3\sqrt3)$.
\end{remark}

\paragraph{Boundary behavior.} As $p_{\min}\downarrow 0$ (e.g., $p_\theta\to e_i$), $J_\theta=S(p_\theta)\to 0$. Then $\lambda_{\min}(J_\theta\!\mid_T)\downarrow 0$ while $\lambda_{\max}(J_\theta\!\mid_T)\le\tfrac12$, so $\kappa(J_\theta\!\mid_T)\le (1/2)/p_{\min}\to\infty$.

\subsection{Clip–Renormalize and the Logit Lift}\label{appB:clip}

\paragraph{Definition and effective floor.}
Fix $\delta_\star\in(0,1)$. Define the clip–renormalize operator
\[
\mathcal C_{\delta_\star}(p):=\frac{\max(p,\delta_\star)}{\|\max(p,\delta_\star)\|_1},\qquad (\max(p,\delta_\star))_i:=\max\{p_i,\delta_\star\}.
\]
If $q=\mathcal C_{\delta_\star}(p)$, then $q_i\ge \delta_{\min}:=\delta_\star/(1+(S-1)\delta_\star)$, and this lower bound is sharp whenever clipping occurs.

Given $\underline\delta\in(0,1/S)$,
\[
\boxed{\,\delta_\star=\frac{\underline\delta}{\,1-(S-1)\underline\delta\,}\quad\Longrightarrow\quad \min_i\big(\mathcal C_{\delta_\star}(p)\big)_i\ge \underline\delta\ \ \forall\,p.}
\]

\paragraph{Logit lift and normalization cancellation.}
Define the \emph{logit lift}
\[
P:\Theta\to\Theta,\qquad P(\theta):=C\log\big(\max(p_\theta,\delta_\star)\big).
\]
If $p'=\max(p_\theta,\delta_\star)$ and $q:=p'/\|p'\|_1$, then $P(\theta)=C\log q$ and
\begin{equation}\label{eq:smP-equals-clip}
\boxed{\,\sm(P(\theta))=q=\mathcal C_{\delta_\star}(p_\theta).\,}
\end{equation}

\begin{proposition}[Global Lipschitz of $P$ and $\sm\circ P$]\label{prop:B-lip-P}
For all $\theta,\vartheta\in\Theta$,
\[
\|P(\theta)-P(\vartheta)\|_2\le \tfrac{1}{2\delta_\star}\|\theta-\vartheta\|_2,\qquad
\|\sm(P(\theta))-\sm(P(\vartheta))\|_2\le \tfrac{1}{4\delta_\star}\|\theta-\vartheta\|_2.
\]
\end{proposition}

\begin{proof}
$\|p_\theta-p_\vartheta\|_2\le \tfrac12\|\theta-\vartheta\|_2$ (MVT + \Cref{cor:B-loewner}); clipping is $1$‑Lipschitz in $\ell_2$; $\log$ is $1/\delta_\star$‑Lipschitz on $[\delta_\star,1]$; $C$ is nonexpansive; $\sm$ has Jacobian norm $\le\tfrac12$.
\end{proof}

\paragraph{Differentials (a.e.).} Since $P$ is piecewise $C^1$,
\begin{equation}\label{eq:DP-DsmP}
\boxed{\,\|DP(\theta)\|_{op}\le \tfrac1{2\delta_\star}\ \text{ for a.e.\ }\theta,\qquad \|D(\sm\!\circ P)(\theta)\|_{op}\le \tfrac1{4\delta_\star}.}
\end{equation}

\paragraph{Local no‑clip criterion.}
If $\min_i p_{\theta_0}(i)\ge \delta_\star+\varepsilon$ and $\|\theta-\theta_0\|_2\le \varepsilon$, then $\|p_\theta-p_{\theta_0}\|_\infty\le \tfrac12\varepsilon$, hence no coordinate is clipped: $P(\theta)=C\log p_\theta=\theta$.

\paragraph{Post‑clipping deviation with a known floor.}
If $\min_i p_\theta(i)\ge \underline\delta>0$ and $c:=|\{i:p_\theta(i)<\delta_\star\}|$, then
\begin{equation}\label{eq:postclip}
\boxed{\,\|P(\theta)-\theta\|_2\le \frac{\delta_\star}{\underline\delta}\,\sqrt c\ \le \ \frac{\delta_\star}{\underline\delta}\,\sqrt S.}
\end{equation}

\paragraph{Smooth vs.\ hard clip; Lipschitz of $DP$.}
Let $L_{DP}$ denote a Lipschitz constant of $\theta\mapsto DP(\theta)$ in operator norm. Two regimes are useful:
\begin{itemize}
\item \emph{Hard‑clip, kink‑free segment (active set fixed):}
\begin{equation}\label{eq:LDP-hard}
\boxed{\,L_{DP}\ \le\ \frac{1}{4\delta_\star^2}\;+\;\frac{\sqrt S}{3\sqrt3}\cdot\frac{1}{\delta_\star}.}
\end{equation}
\item \emph{Smooth clip surrogate $\chi_\tau$:} if $0\le \chi'_\tau\le 1$ and $\Lip(\chi'_\tau)\le c_\tau$, then
\begin{equation}\label{eq:LDP-smooth}
\boxed{\,L_{DP}\ \le\ \frac{1+c_\tau}{4\delta_\star^2}\;+\;\frac{c_\tau}{2\delta_\star}\;+\;\frac{\sqrt S}{3\sqrt3}\cdot\frac{1}{\delta_\star}.}
\end{equation}
\end{itemize}

\subsection{Composite Smoothness for \texorpdfstring{$\Phi(\theta):=J(\sm(P(\theta)))$}{Phi(theta):=J(sm(P(theta)))}}
\label{appB:composite}

\paragraph{Domain and Assumption (A).}
By \eqref{eq:smP-equals-clip}, $p(\theta):=\sm(P(\theta))=\mathcal C_{\delta_\star}(p_\theta)$ lies in the rectangle $[\delta_{\min},1]^S$, $\delta_{\min}=\delta_\star/(1+(S-1)\delta_\star)$.  
\textbf{Assumption (A)} (Euclidean norms throughout): for all $p,q\in[\delta_{\min},1]^S$,
\[
\|\nabla_p J(p)-\nabla_p J(q)\|_2\le L_p\|p-q\|_2,\qquad \sup_{p\in[\delta_{\min},1]^S}\|\nabla_p J(p)\|_2\le G_p<\infty.
\]

\paragraph{Chain pieces and uniform bounds.}
Let $\phi(\theta):=P(\theta)$, $p(\theta):=\sm(\phi(\theta))$, and
\[
B(\theta):=D_\theta p(\theta)=J_{\phi(\theta)}\,DP(\theta).
\]
Using \eqref{eq:DP-DsmP} and \Cref{cor:B-loewner}, uniformly in $\theta$,
\begin{equation}\label{eq:chain-bounds}
\boxed{\,\|DP(\theta)\|_{op}\le\tfrac{1}{2\delta_\star},\qquad \|J_{\phi(\theta)}\|_{op}\le\tfrac12,\qquad \|B(\theta)\|_{op}\le\tfrac{1}{4\delta_\star}.}
\end{equation}
Also, \Cref{prop:B-lip-P} gives
\begin{equation}\label{eq:p-Lip-theta}
\boxed{\,\|p(\theta_2)-p(\theta_1)\|_2\le \tfrac{1}{4\delta_\star}\,\|\theta_2-\theta_1\|_2.}
\end{equation}

\begin{lemma}[Lipschitz of $B(\theta)$]\label{lem:B-lip-B}
For all $\theta_1,\theta_2\in\Theta$,
\[
\boxed{\,\|B(\theta_2)-B(\theta_1)\|_{op}\ \le\ \Big(\frac{\sqrt S}{12\sqrt3}\cdot\frac{1}{\delta_\star^2}\ +\ \tfrac12\,L_{DP}\Big)\ \|\theta_2-\theta_1\|_2,}
\]
with $L_{DP}$ as in \eqref{eq:LDP-hard}–\eqref{eq:LDP-smooth}.
\end{lemma}

\begin{proof}
Split $B(\theta_2)-B(\theta_1)=(J_{\phi_2}-J_{\phi_1})DP(\theta_2)+J_{\phi_1}(DP(\theta_2)-DP(\theta_1))$.  
First term: by \Cref{thm:B-global-Lip-J} and \Cref{prop:B-lip-P},
\[
\|J_{\phi_2}-J_{\phi_1}\|_{op}\le \tfrac{1}{3\sqrt3}\|\phi_2-\phi_1\|_1\le \tfrac{\sqrt S}{3\sqrt3}\|\phi_2-\phi_1\|_2\le \tfrac{\sqrt S}{6\sqrt3\,\delta_\star}\|\Delta\theta\|_2,
\]
then multiply by $\|DP(\theta_2)\|_{op}\le \tfrac{1}{2\delta_\star}$.  
Second term: $\|J_{\phi_1}\|_{op}\le\tfrac12$ and $\|DP(\theta_2)-DP(\theta_1)\|_{op}\le L_{DP}\|\Delta\theta\|_2$.
\end{proof}

\begin{theorem}[Composite Lipschitz constant for $\nabla_\theta\Phi$]\label{thm:B-composite}
Under Assumption (A),
\[
\boxed{\,\|\nabla_\theta\Phi(\theta_2)-\nabla_\theta\Phi(\theta_1)\|_2\ \le\ L_\theta\,\|\theta_2-\theta_1\|_2,\quad
L_\theta\ \le\ \frac{L_p}{16\,\delta_\star^2}\ +\ G_p\Big(\frac{\sqrt S}{12\sqrt3\,\delta_\star^2}\ +\ \tfrac12 L_{DP}\Big).}
\]
\end{theorem}

\begin{proof}
$\nabla_\theta\Phi(\theta)=B(\theta)^\top\nabla_p J(p(\theta))$. Subtract and add:
\[
\|\Delta\nabla_\theta\Phi\|_2\le \|B_2-B_1\|_{op}\,\|\nabla_pJ(p_1)\|_2+\|B_2\|_{op}\,\|\nabla_p J(p_2)-\nabla_p J(p_1)\|_2.
\]
Use \Cref{lem:B-lip-B} and $\|\nabla_pJ(p_1)\|_2\le G_p$ for the first term. For the second, apply \eqref{eq:chain-bounds} and \eqref{eq:p-Lip-theta}.
\end{proof}

\paragraph{Step‑size guidance.}
A conservative choice for gradient methods on $\Phi$ is
\[
\boxed{\,\eta\ \le\ 1/L_\theta.}
\]
A common heuristic (ignoring $G_p$‑driven variation of $B$) is $\eta\approx 16\delta_\star^2/L_p$.

\subsection{Quadratic Approximation and Hessian Suprema}
\label{appB:quadratic}

\paragraph{Second derivatives.}
For $i,k,\ell\in\{1,\ldots,S\}$,
\begin{equation}\label{eq:B-Hessian-exact}
\boxed{\,\partial_{\theta_\ell}\partial_{\theta_k}p_\theta(i)
= p_\theta(i)\Big[(\delta_{i\ell}-p_\theta(\ell))(\delta_{ik}-p_\theta(k)) - p_\theta(k)\big(\delta_{k\ell}-p_\theta(\ell)\big)\Big].}
\end{equation}
Let $H_{k\ell}(\theta)\in\R^S$ collect the components $\partial_{\theta_\ell}\partial_{\theta_k}p_\theta(i)$, and $H(\theta)[u,v]:=\sum_{k,\ell}u_k v_\ell\,H_{k\ell}(\theta)$.

\begin{theorem}[$\ell_2$ and $\ell_1$ suprema]\label{thm:B-Hessian-suprema}
For every $S\ge2$,
\[
\boxed{\,\sup_{\theta,k,\ell}\|H_{k\ell}(\theta)\|_2=\tfrac{1}{\sqrt{54}},\qquad
\sup_{\theta,k,\ell}\|H_{k\ell}(\theta)\|_1=\tfrac{1}{3\sqrt3}.}
\]
Both are attained for $S=2$, and are strict suprema for $S>2$ (approached by concentrating residual mass).
\end{theorem}

\begin{proof}[Proof sketch]
Using \eqref{eq:B-Hessian-exact}, for fixed $(k,\ell)$ the Rayleigh quotient in the residual mass is convex over the simplex, hence maximized at vertices (mass on one coordinate). Reducing to $2\times2$ or $3\times3$ blocks yields the stated optima, attained at $p=(\tfrac12\pm\tfrac{1}{2\sqrt3},\tfrac12\mp\tfrac{1}{2\sqrt3},0,\ldots)$.
\end{proof}

\paragraph{Second‑order expansion and remainders.}
For any $\theta,g\in\R^S$ and $\eta\ge0$,
\begin{equation}\label{eq:B-second-order}
\boxed{\,p_{\theta+\eta g}=p_\theta+\eta J_\theta g+\eta^2\!\int_0^1(1-\tau)\,H(\theta+\tau\eta g)[g,g]\ d\tau.}
\end{equation}
Consequently,
\begin{equation}\label{eq:B-remainders}
\boxed{\,\begin{aligned}
\|R_{\theta,\eta}\|_1&\le \tfrac{\eta^2}{6\sqrt3}\,\|g\|_1^2,\\
\|R_{\theta,\eta}\|_2&\le \tfrac{\eta^2}{2\sqrt{54}}\,\|g\|_1^2,\qquad
\|R_{\theta,\eta}\|_\infty\le \tfrac{\eta^2}{6\sqrt3}\,\|g\|_1^2,\\
\|R_{\theta,\eta}\|_2&\le \tfrac{\eta^2}{6\sqrt3}\,\sqrt{s}\,\|g\|_2^2\quad (s:=\|g\|_0).
\end{aligned}}
\end{equation}
The last bound uses \Cref{thm:B-global-Lip-J} to control $\|\nabla J_{\theta+sg}[g]\|_{op}$ and $\|g\|_1\le\sqrt s\,\|g\|_2$.

\paragraph{$\delta$‑interior refinements.}
Assume the path $\tau\mapsto p_{\theta+\tau\eta g}$ stays in the \emph{trimmed simplex}
\[
\Delta^{S-1}_\delta:=\{p\in\Delta^{S-1}:p_i\ge\delta\ \forall i\},\qquad \delta\in(0,1/S).
\]
For $m\in\mathbb{N}$ and $M\ge m\delta$, define the extremal “mass‑under‑a‑floor” functional
\begin{equation}\label{eq:B-Xi}
\boxed{\,\Xi_m(M;\delta):=\max\Big\{\sum_{j=1}^m x_j^2:\ \sum_{j=1}^m x_j=M,\ x_j\ge\delta\Big\}
=(M-(m-1)\delta)^2+(m-1)\delta^2.}
\end{equation}
Then, for $k=\ell$ with $a=p_\theta(k)\in[\delta,\,1-(S-1)\delta]$,
\[
\|H_{kk}\|_2^2\le (a(1-a)(1-2a))^2+a^2(2a-1)^2\,\Xi_{S-1}(1-a;\delta)=:\big(c_2^{\mathrm{diag}}(\delta,S)\big)^2,
\]
and for $k\neq \ell$ with $a,b\in[\delta,\,1-(S-1)\delta]$, $r:=1-a-b\in[(S-2)\delta,\,1-2\delta]$,
\[
\|H_{k\ell}\|_2^2\le (ab)^2\big[(2a-1)^2+(2b-1)^2\big]+4a^2b^2\,\Xi_{S-2}(r;\delta)=:\big(c_2^{\mathrm{off}}(\delta,S)\big)^2.
\]
Define $c_2(\delta,S):=\max\{c_2^{\mathrm{diag}},c_2^{\mathrm{off}}\}<1/\sqrt{54}$. An entirely analogous construction (sums of absolute values instead of squares) yields $c_1(\delta,S)<1/(3\sqrt3)$ with
\[
\boxed{\,\max_{k,\ell}\|H_{k\ell}(\theta)\|_2\le c_2(\delta,S),\qquad \max_{k,\ell}\|H_{k\ell}(\theta)\|_1\le c_1(\delta,S)\quad \text{whenever }p_\theta\in\Delta^{S-1}_\delta.}
\]
The global maximizers lie at $a_\pm=\tfrac12\pm \tfrac{1}{2\sqrt3}\approx 0.7887,\,0.2113$. Thus if
\begin{equation}\label{eq:B-delta-crit}
\boxed{\,\delta>\delta_{\mathrm{crit}}:=\tfrac12-\tfrac{1}{2\sqrt3}\approx 0.2113,}
\end{equation}
then $c_2(\delta,S)<1/\sqrt{54}$ and $c_1(\delta,S)<1/(3\sqrt3)$ \emph{strictly}. The remainder bounds \eqref{eq:B-remainders} improve by replacing the global constants with $c_2(\delta,S)$ and $c_1(\delta,S)$.

\subsection{Reference table: Parametric Constants}\label{sec:B.quick}
\emph{Spectral norms are $\|\cdot\|_{op}$; vector norms are Euclidean unless labeled. Tangent space $T=\mathbf 1^\perp$, projector $\Pi_T$, centering $C$ as above. The bridge \eqref{eq:bridge} $J_\theta=S(p_\theta)$ is used in \Cref{appC:SRCT}.}

\begin{center}
\renewcommand{\arraystretch}{1.12}
\begin{tabular}{@{}ll@{}}
\toprule
\textbf{Symbol} & \textbf{Value / Bound (where introduced)}\\
\midrule
$\|J_\theta\|_{op}$ &
$\le \tfrac12$ (global); $\lambda(J_\theta\!\mid_T)\in[p_{\min},\,\tfrac12]$ (\Cref{cor:B-loewner})  \\
$\|J_{\theta_2}-J_{\theta_1}\|_{op}$ &
$\le \tfrac{1}{3\sqrt3}\|\Delta\theta\|_1 \le \tfrac{\sqrt S}{3\sqrt3}\|\Delta\theta\|_2 \le \tfrac{S}{3\sqrt3}\|\Delta\theta\|_\infty$ (\Cref{thm:B-global-Lip-J}) \\
$\|P(\theta)-P(\vartheta)\|_2$ &
$\le \tfrac{1}{2\delta_\star}\|\theta-\vartheta\|_2$ (\Cref{prop:B-lip-P}) \\
$\|B(\theta)\|_{op}$ &
$\le \tfrac{1}{4\delta_\star}$ (\Cref{appB:composite}, \eqref{eq:chain-bounds}) \\
$L_{DP}$ &
Hard‑clip kink‑free: \eqref{eq:LDP-hard}; smooth clip: \eqref{eq:LDP-smooth}\\
$L_\theta$ &
$\displaystyle \le \frac{L_p}{16\,\delta_\star^2}+G_p\Big(\frac{\sqrt S}{12\sqrt3\,\delta_\star^2}+\tfrac12 L_{DP}\Big)$ (\Cref{thm:B-composite}) \\
$\sup_{k,\ell}\|H_{k\ell}\|_2$ &
$=1/\sqrt{54}$ (\Cref{thm:B-Hessian-suprema})  \\
$\sup_{k,\ell}\|H_{k\ell}\|_1$ &
$=1/(3\sqrt3)$ (\Cref{thm:B-Hessian-suprema})  \\
$c_1(\delta,S),\ c_2(\delta,S)$ &
$\ell_1/\ell_2$ Hessian suprema on $\Delta^{S-1}_\delta$, both $< $ global constants (\S\ref{appB:quadratic})  \\
\bottomrule
\end{tabular}
\end{center}

\paragraph{Domain reminder for composite bounds.}
All composite bounds in \S\ref{appB:composite} are evaluated on the rectangle $[\delta_{\min},1]^S$, where $\delta_{\min}=\delta_\star/(1+(S-1)\delta_\star)$ (from clip–renormalize). Assumption (A) holds on this set.

%% file: appendices_short/appendix_C_aistats.tex
\section{The Self-Reinforcing Correctness Training (SRCT) Framework}
\label{appC:SRCT}

This appendix records the SRCT calculus used throughout the paper, with canonical constants, operator identities, and dynamical statements in a form suitable for direct citation. The development is self-contained and uses the standard Shahshahani–replicator correspondence.

\subsection{Domain, notation, and canonical constants}
\label{appC:domain}

Fix $K\ge 2$ and a floor $0<\delta_\star<1/K$. The \emph{trimmed simplex} is
\[
\Delta^{K-1}_{\delta_\star}
:= \Bigl\{p\in[0,1]^K:\ \sum_{i=1}^K p_i=1,\ \ p_i\ge\delta_\star\ \forall i\Bigr\},
\qquad
T:=\mathbf 1^\perp=\{v\in\mathbb R^K:\langle v,\mathbf 1\rangle=0\}.
\]
Euclidean inner products and norms are used throughout. Write $\langle\log p\rangle:=\sum_i p_i\log p_i$ and $H(p):=-\langle\log p\rangle$.
\vspace{-0.35em}
\[
\boxed{\quad
\Lambda := 1+\log\frac1{\delta_\star},\qquad
C_A := A\,(2+\sqrt K)\,\Lambda,\qquad
A:=\varepsilon+\lambda\alpha+\beta_{\mathrm{KL}}\ \ge 0.
\quad}
\]

\subsection{SRCT objective, correct variational derivative, and canonical drift}
\label{appC:srct-gradient}

Let $U\in\mathbb R^K$ be a bounded utility vector, $K\in\mathbb R^{K\times K}$ symmetric PSD, and $p_{\mathrm{base}}\in\Delta^{K-1}$ with \emph{full support} $p_{\mathrm{base},i}>0$. Consider
\[
\widetilde J[p]
=\sum_i U_i p_i \;+\; \lambda\Big(\alpha H[p]-\beta\,p^\top K p\Big)
\;-\;\beta_{\mathrm{KL}}\mathrm{KL}(p\|p_{\mathrm{base}})
\;+\;\varepsilon H[p].
\]
A direct calculation gives the pointwise variational derivative
\[
\frac{\delta\widetilde J}{\delta p_i}
= U_i \;-\; 2\lambda\beta\,(Kp)_i \;+\; \beta_{\mathrm{KL}}\log p_{\mathrm{base},i}
\;-\; A\,\bigl(1+\log p_i\bigr),
\qquad A=\varepsilon+\lambda\alpha+\beta_{\mathrm{KL}}.
\]
Introduce the selection covariance and entropic vector
\[
S(p):=\operatorname{diag}(p)-pp^\top,\qquad
E(p):=p\odot(\log p-\langle\log p\rangle),
\]
and the \emph{selective score}
\[
\phi_A(p)\ :=\ U\ -\ 2\lambda\beta\,Kp\ +\ \beta_{\mathrm{KL}}\log p_{\mathrm{base}}.
\]
Then the Shahshahani gradient flow $\dot p=\nabla_{\!Sh}\widetilde J(p)$ is the SRCT ODE
\[
\boxed{\qquad
\dot p\ =\ F(p)\ :=\ S(p)\,\phi_A(p)\ -\ A\,E(p),\qquad
\sum_i \dot p_i=0\ (\text{tangency to }T).
\qquad}
\]

\subsection{Operator facts for \texorpdfstring{$S$}{S} and the entropic map \texorpdfstring{$E$}{E}}
\label{appC:operators}

\paragraph{Selection covariance $S(p)$.}
For all $p$, $S(p)\mathbf 1=0$, and $v^\top S(p)v=\operatorname{Var}_p(V)$ where $V$ takes value $v_i$ with probability $p_i$. By Popoviciu and $(\max-\min)^2\le 2\|v\|_2^2$,
\[
\boxed{\quad \|S(p)\|_{2\to2}\le\tfrac12\ ,\qquad \|S(p)-S(q)\|_{2\to2}\le 3\,\|p-q\|_2.\quad}
\]

\paragraph{Entropic vector $E(p)$.}
For any $p\in\Delta^{K-1}_{\delta_\star}$ and $v\in\mathbb R^K$, the Jacobian is
\[
\boxed{\qquad
J_E(p)\,v = \operatorname{diag}\!\bigl(1+\log p-\langle\log p\rangle\bigr)\,v\ -\ p\,\langle 1+\log p,\ v\rangle.
\qquad}
\]
Consequently, on $\Delta^{K-1}_{\delta_\star}$,
\[
\boxed{\quad \|E(p)-E(q)\|_2\ \le\ (2+\sqrt K)\,\Lambda\,\|p-q\|_2.\quad}
\]

\subsection{Global Lipschitz of the SRCT drift and Carath\'eodory regularity}
\label{appC:lipschitz-F}

Let $L_\phi:=2\lambda\beta\,\|K\|_{2\to2}$ and $M_{\phi,2}:=\sup_{p\in\Delta^{K-1}_{\delta_\star}}\|\phi_A(p)\|_2<\infty$ (compactness).
Using \S\ref{appC:operators} and $F=S\phi_A-AE$,
\[
\boxed{\quad
\|F(p)-F(q)\|_2\ \le\ \Bigl(\tfrac12\,L_\phi\;+\;3\,M_{\phi,2}\;+\;C_A\Bigr)\,\|p-q\|_2.
\quad}
\]
Hence $F$ is globally Lipschitz on $\Delta^{K-1}_{\delta_\star}$. For non-autonomous scores $\phi_A(t,p)$ that are measurable in $t$, locally Lipschitz in $p$, and locally bounded, $F(t,p)$ satisfies Carath\'eodory conditions on $\operatorname{ri}\Delta^{K-1}_{\delta_\star}$; the ODE admits a unique local absolutely continuous solution from any interior initial condition. Tangency to $T$ and §\ref{sec:C.2} (BD) give global-in-time confinement.

\subsection{Mass balance and log-ratio calculus}
\label{sec:C.3}

For any absolutely continuous solution $p(\cdot)$ with $M(t):=\sum_i p_i(t)$,
\[
\boxed{\quad
\dot M(t)=\Bigl(\overline{\phi_A}(t,p(t))\ -\ A\,\langle\log p(t)\rangle\Bigr)\,\bigl(1-M(t)\bigr),
\qquad \overline{\phi_A}=\sum_i p_i\phi_{A,i}.
\quad}
\]
Thus $M(0)=1\Rightarrow M(t)\equiv1$.

Fix $i\neq j$ and let $J$ be an interval on which $p_i,p_j>0$. Set $z(t):=\log\frac{p_i(t)}{p_j(t)}$ and
\[
d_{ij}(t):=\big(U_i-U_j\big)\ -\ 2\lambda\beta\big((Kp)_i-(Kp)_j\big)\ +\ \beta_{\mathrm{KL}}\log\frac{p_{\mathrm{base},i}}{p_{\mathrm{base},j}}.
\]
Subtracting the $i$ and $j$ equations yields the \emph{log-ratio identity}
\[
\boxed{\quad
\dot z(t)=d_{ij}(t)-A\,z(t)\quad \text{for a.e.\ }t\in J,
\qquad
z(t)=z(t_0)e^{-A(t-t_0)}+\int_{t_0}^t e^{-A(t-s)}\,d_{ij}(s)\,ds.
\tag{eq:C-VoC}}
\]
The usual time-varying and constant-box envelopes follow by comparison; if $A>0$ and $|d_{ij}|\le M$ on $[t_0,\infty)\cap J$, then $|z(t)|\le |z(t_0)|e^{-A(t-t_0)}+\frac{M}{A}(1-e^{-A(t-t_0)})$ (uniform boundedness).

\subsection{Positivity and face invariance on the closed simplex}

Let $H(p)=-\langle\log p\rangle\in[0,\log K]$ and $M_{\mathrm{traj}}(t):=\max_k |\phi_{A,k}-\overline{\phi_A}|(t,p(t))\in L^1_{\mathrm{loc}}$.

\begin{lemma}[No finite-time boundary hitting]\label{lem:C3-positivity}
If $p_i(0)>0$, then for all finite $t$,
\[
\boxed{\quad
\log p_i(t)\ \ge\ \log p_i(0)\ -\ \int_0^t \Big(M_{\mathrm{traj}}(s)+A\,H\big(p(s)\big)\Big)\,ds,
\quad\Rightarrow\quad p_i(t)>0.
\quad}
\]
\end{lemma}

\begin{lemma}[Face invariance at zero]\label{lem:C3-face}
If $p_i(0)=0$, then $p_i(t)\equiv0$. \emph{Sketch.} With $y=p_i$, one has
$y'=a(t)\,y-A\,y\log y$ with $a\in L^1_{\mathrm{loc}}$. The Osgood modulus $\omega(y)=y(1+|\log y|)$ satisfies $\int_{0^+}\!dr/\omega(r)=\infty$, giving uniqueness of $y\equiv0$ through $y(0)=0$.
\end{lemma}

\subsection{Barrier--Dominance and confinement on \texorpdfstring{$\Delta^{K-1}_{\delta_\star}$}{Delta(K-1) delta star}}
\label{sec:C.2}

On the lower face $\{p_i=\delta_\star\}$, using $p_j\ge\delta_\star$ and $\sum_{j\ne i}p_j=1-\delta_\star$, the convexity of $x\mapsto x\log x$ yields the \emph{entropy face gap}
\[
\boxed{\quad
L_K(\delta_\star):=(1-\delta_\star)\log\frac{1-\delta_\star}{(K-1)\delta_\star}\ >\ 0\quad(\delta_\star<1/K).
\quad}
\]
A direct computation gives the face inequality
\[
\boxed{\quad
\text{at }p_i=\delta_\star:\qquad
F_i(p)\ \ge\ \delta_\star\Bigl(A\,L_K(\delta_\star)-\bigl(\phi_{A,i}(p)-\overline{\phi_A}(p)\bigr)^-\Bigr).
\tag{eq:C-face-gap}}
\]
Define the worst outward selective pressure on the boundary
\[
M_{\mathrm{eff}}^{\mathrm{face}}
:=\sup_{\substack{p\in\partial\Delta^{K-1}_{\delta_\star}\\ i:\,p_i=\delta_\star}}
\bigl(\phi_{A,i}(p)-\overline{\phi_A}(p)\bigr)^-\ <\ \infty.
\]
\begin{theorem}[Barrier--Dominance]\label{thm:C-BD}
If
\[
\boxed{\quad A\,L_K(\delta_\star)\ \ge\ M_{\mathrm{eff}}^{\mathrm{face}} \quad}
\tag{eq:C-BD}
\]
then $F(p)$ lies in the tangent cone of $\Delta^{K-1}_{\delta_\star}$ at every boundary point; hence $\Delta^{K-1}_{\delta_\star}$ is forward invariant. If the inequality is strict, trajectories starting in $\operatorname{ri}\Delta^{K-1}_{\delta_\star}$ never hit the boundary (strict interior invariance).
\end{theorem}

\paragraph{Coarse sufficient BD.}
Since $|\phi_{A,i}-\overline{\phi_A}|\le 2\|\phi_A\|_\infty$, it suffices that
\[
A\,L_K(\delta_\star)\ \ge\ 2\sup_{p\in\Delta^{K-1}_{\delta_\star}}\|\phi_A(p)\|_\infty.
\]
\emph{Degenerate floor:} If $\delta_\star=1/K$, then $L_K(\delta_\star)=0$ and the simplex is a singleton.

\subsection{Existence/uniqueness on the mass hyperplane}

By \S\ref{appC:lipschitz-F}, $F$ is globally Lipschitz on $\Delta^{K-1}_{\delta_\star}$ and tangent to $H:=\{p:\sum_i p_i=1\}$. Kirszbraun–Valentine yields a Lipschitz extension $\widetilde F:H\to H$ with the same constant; Picard–Lindel\"of gives a unique global absolutely continuous solution from any $p(0)\in H$. Under \eqref{thm:C-BD}, the trajectory remains in $\Delta^{K-1}_{\delta_\star}$.

\subsection{Single-site score fields: Lyapunov structure and convergence}
\label{sec:C.single-site}

Assume a separable score $\phi_i(p)=f_i(p_i)$ with $f_i\in C([\underline\delta,1])\cap C^1((\underline\delta,1])$, $\sup_{i,s}|f_i'(s)|<\infty$, and $f_i'\le 0$ on $(\underline\delta,1]$. On $\Delta^{K-1}_{\delta_\star}$ take $\underline\delta=\delta_\star$; on the closed simplex (for $A=0$) take $\underline\delta=0$. Define
\[
g_i(s):=f_i(s)-A\log s,\qquad
\Psi_i(s):=\int_{s_0}^s g_i(u)\,du,\qquad
\mathcal L_\psi(p):=\sum_{i=1}^K \Psi_i(p_i),\qquad
\bar g(p):=\sum_i p_i g_i(p_i).
\]
Along classical solutions,
\[
\boxed{\quad
\frac{d}{dt}\,\mathcal L_\psi\bigl(p(t)\bigr)
=\sum_{i=1}^K p_i(t)\,\bigl(g_i(p_i(t))-\bar g(p(t))\bigr)^2\ \ge\ 0.
\quad}
\]

\paragraph{Regime $A>0$: strong concavity, KKT, convergence.}
On $[\delta_\star,1]$, $g_i'(s)=f_i'(s)-A/s\le -A$, hence on the affine simplex
\[
D^2\mathcal L_\psi(p)=\operatorname{diag}(g_1'(p_1),\ldots,g_K'(p_K))\ \preceq\ -A I,
\]
so $\mathcal L_\psi$ is $A$-strongly concave. Maximization over $\Delta^{K-1}_{\delta_\star}$ has a unique solution $p^\dagger$; the KKT conditions give a scalar $c^\dagger$ and multipliers $\nu_i^\dagger\ge0$ such that
\[
g_i(p_i^\dagger)=c^\dagger-\nu_i^\dagger,\qquad \nu_i^\dagger(\delta_\star-p_i^\dagger)=0,\qquad \sum_i p_i^\dagger=1.
\]
Under strict BD, $p^\dagger$ is interior and $g_i(p_i^\dagger)\equiv c^\dagger$. Since trajectories are confined and $\mathcal L_\psi$ is nondecreasing and bounded above, LaSalle’s invariance principle implies global convergence to $p^\dagger$.

\paragraph{Regime $A=0$: water-filling and support selection.}
Assume (CR+SM): each $f_i$ is continuous and strictly decreasing on $[0,1]$, with inverse $f_i^{-1}:[f_i(1),f_i(0)]\to[1,0]$. There exists a unique pair $(S^\star,c^\star)$ with
\[
\sum_{i\in S^\star} f_i^{-1}(c^\star)=1,\qquad
p_i^\star=
\begin{cases}
f_i^{-1}(c^\star), & i\in S^\star,\\
0, & i\notin S^\star,
\end{cases}
\qquad
S^\star=\{\,i:\ f_i(1)\le c^\star< f_i(0)\,\}.
\]
Moreover, $\mathcal L_\psi$ is strictly concave on every face; by face invariance and monotonicity, $p(t)\to p^\star$.

\subsection{Safe denominators (linear-functional floor)}
\label{appC:safe-denominators}

If $\phi$ contains denominators of the form $a^\top p$ with $a\in\mathbb R_+^K\setminus\{0\}$, then on $\Delta^{K-1}_{\delta_\star}$,
\[
\boxed{\quad a^\top p\ \ge\ \delta_\star\,\|a\|_1.\quad}
\]
Hence such denominators are uniformly bounded away from zero.

%% file: appendices_short/appendix_D_aistats.tex
\section{STaR through the SRCT Lens}
\label{appD:STaRSRCT}

This appendix instantiates the SRCT framework for the Self‑Taught Reasoner. We specify the score field, establish norm and Lipschitz bounds (including Jacobian structure and rank), prove well‑posedness and confinement (trimmed‑domain barrier–dominance), and analyze log‑ratio dynamics and asymptotics.

\subsection{Setting, notation, and basic aggregates}

Fix $K\ge2$ and the probability simplex
\[
\Delta^{K-1}:=\Bigl\{p\in[0,1]^K:\ \sum_{k=1}^K p_k=1\Bigr\},\qquad
\operatorname{int}\Delta^{K-1}:=\{p\in\Delta^{K-1}:\ p_k>0\ \forall k\}.
\]
Split indices into \textbf{correct} $\mathcal C$ (size $M\ge1$) and \textbf{incorrect} $\mathcal I:=\{1,\dots,K\}\setminus\mathcal C$ (size $L=K-M$).
For $p\in\Delta^{K-1}$ define
\[
\rho(p):=\sum_{c\in\mathcal C}p_c,\qquad
S^{(2)}(p):=\sum_{c\in\mathcal C}p_c^2,\qquad
\langle\log p\rangle:=\sum_{k=1}^K p_k \log p_k\in[-\log K,0].
\]
For a floor $\delta_\star\in(0,1/K)$, the \textbf{trimmed simplex} is
\[
\Delta^{K-1}_{\delta_\star}:=\{p\in\Delta^{K-1}:\ \min_k p_k\ge \delta_\star\}\quad\Rightarrow\quad \rho(p)\ge M\delta_\star.
\]
Vector norms are Euclidean; for matrices we use $\|\cdot\|_1$ (max.\ column sum), $\|\cdot\|_\infty$ (max.\ row sum), and the spectral norm $\|\cdot\|_2$, with $\|J\|_2\le \sqrt{\|J\|_1\|J\|_\infty}$.

\subsection{The STaR score field: bounds, Jacobian, and Lipschitzness}

\begin{definition}[STaR score]
\label{def:star-score}
On $\mathcal D:=\{p\in\operatorname{int}\Delta^{K-1}:\ \rho(p)>0\}$ define $\phi^{\mathrm{STaR}}:\mathcal D\to\mathbb R^K$ by
\[
\phi^{\mathrm{STaR}}_k(p)=
\begin{cases}
\dfrac{p_k-S^{(2)}(p)}{\rho(p)}, & k\in\mathcal C,\\[6pt]
-\dfrac{S^{(2)}(p)}{\rho(p)}, & k\in\mathcal I.
\end{cases}
\]
For $M\ge1$ and $p\in\operatorname{int}\Delta^{K-1}$, $\rho(p)>0$, hence $\mathcal D=\operatorname{int}\Delta^{K-1}$ and $\phi^{\mathrm{STaR}}$ is $C^\infty$ on $\mathcal D$.
\end{definition}

\paragraph{Componentwise and norm bounds (sharp).}
For $\rho=\rho(p)$ and $S^{(2)}=S^{(2)}(p)$:
\[
\sum_{k=1}^K p_k\,\phi^{\mathrm{STaR}}_k(p)=0\quad\text{(centering).}
\]
For $c\in\mathcal C$, $0\le p_c\le \rho$ and $S^{(2)}\ge \rho^2/M$ (Cauchy–Schwarz), whence
\[
\boxed{\ \ \phi_c\in\bigl[-\rho,\ 1-\tfrac{\rho}{M}\bigr],\qquad
\phi_i=-\tfrac{S^{(2)}}{\rho}\in[-\rho,0]\ \ (i\in\mathcal I),\quad
\|\phi^{\mathrm{STaR}}(p)\|_\infty\le 1. \ \ }
\]
Moreover,
\[
\boxed{\ \ \|\phi^{\mathrm{STaR}}(p)\|_2^2\ \le\ 1-2\,\rho(p)+K\,\rho(p)^2\ \le\ K-1,\qquad
\|\phi^{\mathrm{STaR}}(p)\|_2\le \sqrt{K-1}. \ \ }
\]
The quadratic upper bound is tight in the limit $\rho\to 1$ with all correct mass on one index.

\begin{lemma}[Jacobian, zero columns on $\mathcal I$, and rank]
\label{lem:J-structure}
Let $J(p):=[\partial\phi^{\mathrm{STaR}}_k/\partial p_j](p)$. Then $J_{k,j}(p)=0$ for all $j\in\mathcal I$. For $j\in\mathcal C$,
\[
\frac{\partial}{\partial p_j}\!\left(\frac{p_k}{\rho}\right)=\frac{\delta_{kj}\rho-p_k}{\rho^2},\qquad
\frac{\partial}{\partial p_j}\!\left(\frac{S^{(2)}}{\rho}\right)=\frac{2p_j\rho-S^{(2)}}{\rho^2},
\]
hence
\[
J_{k,j}(p)=
\begin{cases}
\dfrac{\delta_{kj}}{\rho}-\dfrac{p_k}{\rho^2}-\dfrac{2p_j}{\rho}+\dfrac{S^{(2)}}{\rho^2}, & k\in\mathcal C,\ j\in\mathcal C,\\[8pt]
-\dfrac{2p_j}{\rho}+\dfrac{S^{(2)}}{\rho^2}, & k\in\mathcal I,\ j\in\mathcal C,\\[6pt]
0,& j\in\mathcal I.
\end{cases}
\]
In particular, $\operatorname{rank}J(p)\le M$.
\end{lemma}

\begin{proposition}[Lipschitz bounds on $\Delta^{K-1}_{\delta_\star}$ and interior compacts]
\label{prop:Lip-J}
On $\Delta^{K-1}_{\delta_\star}$ one has $\rho\ge M\delta_\star$. Uniformly for $p\in\Delta^{K-1}_{\delta_\star}$,
\[
\boxed{\ \ \|J(p)\|_\infty\ \le\ \frac{2}{\delta_\star}+M+2,\qquad
        \|J(p)\|_1\ \le\ \frac{2}{M\delta_\star}+3K,\qquad
        \|J(p)\|_2\ \le\ \sqrt{\Bigl(\tfrac{2}{M\delta_\star}+3K\Bigr)\Bigl(\tfrac{2}{\delta_\star}+M+2\Bigr)}. \ \ }
\]
If $\mathcal D_0\subset\operatorname{int}\Delta^{K-1}$ is compact with $\rho(p)\ge\rho_{\min}>0$, then uniformly for $p\in\mathcal D_0$,
\[
\boxed{\ \ \|J(p)\|_\infty\ \le\ \frac{M+1}{\rho_{\min}}+M+2,\quad
\|J(p)\|_1\ \le\ \frac{2}{\rho_{\min}}+3K,\quad
\|J(p)\|_2\ \le\ \sqrt{\Bigl(\tfrac{2}{\rho_{\min}}+3K\Bigr)\Bigl(\tfrac{M+1}{\rho_{\min}}+M+2\Bigr)}. \ \ }
\]
\emph{Proof sketch.} Sum the absolute values of the entries in Lemma~\ref{lem:J-structure} by rows/columns using $\rho\ge M\delta_\star$, $p_j\le \rho$, $S^{(2)}\le\rho^2$; then apply $\|J\|_2\le\sqrt{\|J\|_1\|J\|_\infty}$.
\end{proposition}

\paragraph{Continuity caveat (stiffness near faces).}
Although $\phi^{\mathrm{STaR}}$ is bounded and smooth on $\mathcal D$, the $1/\rho^2$ factors in $J$ blow up as $\rho\downarrow 0$. Thus $\phi^{\mathrm{STaR}}$ is \emph{not} globally Lipschitz on $\operatorname{int}\Delta^{K-1}$; quantitative Lipschitz control requires either $\Delta^{K-1}_{\delta_\star}$ or a uniform $\rho_{\min}>0$.

\begin{proposition}[Ambient spectral lower bound; dependence on $M$]
\label{prop:ambient-lb}
For all $p\in\mathcal D$,
\[
\boxed{\ \ \|J(p)\|_2\ \ge\ \frac{\|p_{\mathcal C}\|_2}{\rho(p)}\,\sqrt{K}\ \ge\ \sqrt{\frac{K}{M}}\ .\ \ }
\]
\emph{Proof.} Let $v=(p_{\mathcal C}/\|p_{\mathcal C}\|_2,\ 0_{\mathcal I})$. Lemma~\ref{lem:J-structure} implies $Jv=-(\|p_{\mathcal C}\|_2/\rho)\,\mathbf 1$. Taking inner product with $\mathbf 1/\sqrt K$ yields the first inequality; Cauchy–Schwarz gives $\|p_{\mathcal C}\|_2\ge\rho/\sqrt M$.
\end{proposition}

\begin{corollary}[Exact formulas when $M=1$]
\label{cor:M1}
If $M=1$ with $\mathcal C=\{c\}$, then $J(p)=-\,\mathbf 1\,e_c^\top$, hence $\|J(p)\|_2=\sqrt K$. The restriction to the tangent space $T=\mathbf 1^\perp$ has operator norm $\|J|_T\|_2=\sqrt{K-1}$; moreover $\Pi_T J\Pi_T\equiv 0$.
\end{corollary}

\subsection{STaR as an SRCT flow: well‑posedness, Lipschitz drift, and confinement}

\paragraph{Dynamics.}
For $\varepsilon\ge0$ (entropic weight), the SRCT ODE reads
\[
\boxed{\ \ \dot p_k\;=\;p_k\,\phi^{\mathrm{STaR}}_k(p)\;-\;\varepsilon\,p_k\bigl(\log p_k-\langle\log p\rangle\bigr),\qquad k=1,\dots,K.\ \ }
\]
By centering, $\sum_k\dot p_k=0$, so $\sum_k p_k(t)\equiv 1$.

\paragraph{No finite‑time boundary hitting and uniform floor.}
Let $Y_i:=-\log p_i$. Using $|\phi^{\mathrm{STaR}}_i|\le 1$ and $-\langle\log p\rangle\le\log K$,
\[
\dot Y_i\ \le\ 1+\varepsilon\log K-\varepsilon Y_i.
\]
Therefore $Y_i(t)$ remains finite on any finite interval (no coordinate reaches $0$ in finite time, even for $\varepsilon=0$). If $\varepsilon>0$, solving the linear inequality gives the \emph{uniform floor}
\[
\boxed{\ \ p_i(t)\ \ge\ \min\Big\{\,p_i(0),\ \tfrac{1}{K}\,e^{-1/\varepsilon}\Big\}\qquad(\forall t\ge0).\ \ }
\]

\paragraph{Global $\ell_2$ Lipschitz bound for the SRCT drift on $\Delta^{K-1}_{\delta_\star}$.}
Write $S(p):=\mathrm{diag}(p)-pp^\top$ and $E(p):=p\odot(\log p-\langle\log p\rangle)$. Then
\[
F(p):=p\odot\phi^{\mathrm{STaR}}(p)-\varepsilon\,E(p)\ =\ S(p)\,\phi^{\mathrm{STaR}}(p)-\varepsilon\,E(p).
\]
On $\Delta^{K-1}_{\delta_\star}$,
\[
\|S(p)\|_{2\to2}\le \tfrac12,\qquad \|S(p)-S(q)\|_{2\to2}\le 3\|p-q\|_2,
\]
and, with $\Lambda:=1+\log(1/\delta_\star)$,
\[
\|E(p)-E(q)\|_2\ \le\ (2+\sqrt K)\,\Lambda\,\|p-q\|_2.
\]
Combining with $\sup\|\phi^{\mathrm{STaR}}\|_2\le\sqrt K$ and $L_{\phi,2}:=\sup_{r\in\Delta^{K-1}_{\delta_\star}}\|J(r)\|_2$ from Proposition~\ref{prop:Lip-J},
\[
\boxed{\ \ \|F(p)-F(q)\|_2\ \le\ \Big(\tfrac12\,L_{\phi,2}\;+\;3\sqrt K\;+\;\varepsilon(2+\sqrt K)\Lambda\Big)\,\|p-q\|_2\qquad(p,q\in\Delta^{K-1}_{\delta_\star}).\ \ }
\]

\paragraph{Forward invariance of a trimmed simplex (Barrier–Dominance).}
On the facet $p_i=\delta_\star$,
\[
\dot p_i=\delta_\star\Big(\phi^{\mathrm{STaR}}_i(p)+\varepsilon\,[\,\langle\log p\rangle-\log\delta_\star\,]\Big).
\]
The \emph{entropy face gap}
\[
L_K(\delta):=\inf_{p:\,p_i=\delta}\ \bigl(\langle\log p\rangle-\log\delta\bigr)\ =(1-\delta)\log\!\frac{1-\delta}{(K-1)\delta}
\]
is attained by equalizing the other $K-1$ coordinates. Since $\phi^{\mathrm{STaR}}_i\ge -1$,
\[
\inf_{p:\,p_i=\delta_\star}\dot p_i\ \ge\ \delta_\star\bigl(-1+\varepsilon\,L_K(\delta_\star)\bigr),
\]
so the \textbf{sharp} sufficient condition
\[
\boxed{\ \ \varepsilon\,L_K(\delta_\star)\ \ge\ 1\ \ }
\]
guarantees inward pointing drift on every facet and hence forward invariance (Nagumo). A \textbf{conservative} alternative, robust to mild non‑centering, uses $|\phi_i-\bar\phi|\le 2\|\phi\|_2\le 2\sqrt K$ to give
\[
\boxed{\ \ \varepsilon\,L_K(\delta_\star)\ \ge\ 2\sqrt K\ .\ \ }
\]

\paragraph{Uniform linear growth.}
Along any trajectory in $\operatorname{int}\Delta^{K-1}$,
\[
\boxed{\ \ |\dot p_i|\ \le\ p_i|\phi_i|+\varepsilon\bigl(|p_i\log p_i|+p_i|\langle\log p\rangle|\bigr)\ \le\ 1+\varepsilon\Big(\tfrac1e+\log K\Big). \ \ }
\]

\paragraph{Well‑posedness summary.}
For any $p(0)\in\operatorname{int}\Delta^{K-1}$ and $\varepsilon\ge0$ there is a unique global solution in $\operatorname{int}\Delta^{K-1}$ (no finite‑time boundary hitting). On $\Delta^{K-1}_{\delta_\star}$ the drift is globally Lipschitz with the bound above; under either BD condition the trimmed simplex is forward invariant. For $\varepsilon>0$ every coordinate satisfies the uniform floor.

\subsection{Log‑ratio dynamics and asymptotics}

For $k\neq j$, set $z_{kj}:=\log\frac{p_k}{p_j}$. Differentiating,
\[
\boxed{\ \ \dot z_{kj}(t)=\bigl(\phi^{\mathrm{STaR}}_k(p(t))-\phi^{\mathrm{STaR}}_j(p(t))\bigr)\;-\;\varepsilon\,z_{kj}(t). \ \ }
\]
Instantiating the score differences:
\[
\phi_i-\phi_j\equiv 0\ (i,j\in\mathcal I),\quad
\phi_a-\phi_b=\frac{p_a-p_b}{\rho}\ (a,b\in\mathcal C),\quad
\phi_c-\phi_i=\frac{p_c}{\rho}\ (c\in\mathcal C,i\in\mathcal I).
\]

\paragraph{Incorrect vs.\ incorrect (\(i,j\in\mathcal I\)).}
$\dot z_{ij}=-\varepsilon z_{ij}\Rightarrow z_{ij}(t)=z_{ij}(0)e^{-\varepsilon t}$: incorrect traces equalize exponentially when $\varepsilon>0$.

\paragraph{Within \(\mathcal C\) (\(a,b\in\mathcal C\)).}
\(
\dot z_{ab}=\frac{p_a-p_b}{\rho}-\varepsilon z_{ab},\quad \bigl|\tfrac{p_a-p_b}{\rho}\bigr|<1.
\)
Variation of constants yields
\[
|z_{ab}(t)|\ \le\ |z_{ab}(0)|e^{-\varepsilon t}+\frac{1-e^{-\varepsilon t}}{\varepsilon}.
\]
On $\Delta^{K-1}_{\delta_\star}$, $\rho\ge M\delta_\star$ strengthens this to
\[
\boxed{\ \ |z_{ab}(t)|\ \le\ |z_{ab}(0)|e^{-\varepsilon t}+\frac{1-M\delta_\star}{\varepsilon}\,(1-e^{-\varepsilon t}). \ \ }
\]

\paragraph{Correct vs.\ incorrect (\(c\in\mathcal C,i\in\mathcal I\)).}
Let $c^\star(t)\in\arg\max_{c\in\mathcal C}p_c(t)$ and set $z_{i c^\star}:=\log\frac{p_i}{p_{c^\star}}$. Then
\[
\dot z_{i c^\star}=-\frac{p_{c^\star}}{\rho}-\varepsilon z_{i c^\star},\qquad \frac{p_{c^\star}}{\rho}\in\Bigl[\frac{1}{M},\,1\Bigr],
\]
so
\[
\boxed{\ \ z_{i c^\star}(t)\ \in\ \Bigl[z_{i c^\star}(0)e^{-\varepsilon t}-\tfrac{1-e^{-\varepsilon t}}{\varepsilon},\ 
z_{i c^\star}(0)e^{-\varepsilon t}-\tfrac{1-e^{-\varepsilon t}}{M\varepsilon}\Bigr],\quad
\limsup_{t\to\infty}\frac{p_i(t)}{p_{c^\star}(t)}\ \le\ e^{-1/(M\varepsilon)}. \ \ }
\]

\paragraph{Asymptotics.}
If $\varepsilon>0$ and there exists $c\in\mathcal C$ with $p_c(t)\to p_c^\infty>0$ and $\frac{p_c(t)}{\rho(t)}\to g\in[1/M,1]$, then $z_{ic}(t)\to -g/\varepsilon$ and
\[
\boxed{\ \ p_i(t)\ \to\ p_c^\infty\,e^{-g/\varepsilon}\ \in\ \bigl[p_c^\infty e^{-1/\varepsilon},\ p_c^\infty e^{-1/(M\varepsilon)}\bigr].\ \ }
\]
If $\varepsilon=0$ and there exist $c\in\mathcal C$, $g_{\min}>0$ with $\frac{p_c(t)}{\rho(t)}\ge g_{\min}$ on an unbounded time set, then $\dot z_{ci}\ge g_{\min}$, hence $z_{ci}(t)\to +\infty$ and $p_i(t)\to 0$ (incorrect mass vanishes). Non‑vanishing $\rho$ alone does \emph{not} imply extinction.

\subsection{Edge cases and remarks}
If $M=0$ the score in Definition~\ref{def:star-score} is undefined ($\rho\equiv 0$). If $M=K$, then $\rho\equiv1$ and $\phi^{\mathrm{STaR}}_k(p)=p_k-\sum_{j=1}^K p_j^2$.
The ambient lower bound in Proposition~\ref{prop:ambient-lb} is realized in the normal direction $\mathrm{span}\{\mathbf 1\}$ and does not directly lower‑bound the tangent‑restricted operator $\Pi_T J\Pi_T$ with $T=\mathbf 1^\perp$.

%% file: appendices_short/appendix_E_aistats.tex
\section{GRPO through the SRCT Lens}
\label{appE:GRPO-SRCT}

We analyze GRPO within the SRCT framework. We prove barrier–dominance (face invariance), derive rank‑one Lipschitz constants for the GRPO score, obtain two‑sided cross‑class envelopes, and establish exponential convergence to a unique two‑level equilibrium under a slope condition.

\subsection{Setup and GRPO characteristic}

\paragraph{Domain and classes.}
Fix integers $K\ge 2$, $G\ge 2$, and a floor $\delta_\star\in(0,1/K]$. Work on the trimmed simplex
\[
\Delta^{K-1}_{\delta_\star}:=\Bigl\{\,p\in[0,1]^K:\ \sum_{k=1}^K p_k=1,\ \ p_k\ge\delta_\star\,\Bigr\}.
\]
Partition indices into \emph{correct} and \emph{incorrect} sets $\mathcal C,\mathcal I$ with sizes
$K_C:=|\mathcal C|\ge 0$, $K_I:=|\mathcal I|\ge 0$, $K_C+K_I=K$. Write the correct mass
\[
\rho:=\rho_C(p):=\sum_{c\in\mathcal C} p_c.
\]
If $K_I\ge 1$ and $p\in\Delta_{\delta_\star}^{K-1}$ then $\rho\in\big[\,K_C\delta_\star,\ 1-K_I\delta_\star\,\big]$.

\paragraph{GRPO characteristic.}
For $t\in(0,G]$ set $f_G(t):=\sqrt{(G-t)/t}$. With $S\sim\mathrm{Binom}(G-1,\rho)$ define
\[
c_1(\rho):=\E\big[f_G(1+S)\big],
\qquad
h_G(\rho):=\frac{c_1(\rho)}{1-\rho}\quad(\rho\in(0,1)).
\]

\begin{lemma}[basic properties of $h_G$]\label{lem:hG-basic}
The map $h_G$ extends to $C^1([0,1])$ with
\[
h_G(0)=h_G(1)=\sqrt{G-1},\qquad 
D_G:=\sup_{\rho\in[0,1]}|h_G'(\rho)|<\infty.
\]
Moreover for all $\rho\in[0,1]$,
\[
1-\tfrac1G\ \le\ h_G(\rho)\ \le\ \sqrt{G-1},
\]
and $h_G$ is constant when $G\in\{2,3\}$.
\end{lemma}

\begin{proof}[Proof sketch]
$c_1$ is a finite binomial sum of smooth terms, hence $C^\infty([0,1])$. Expansion at $\rho=1$ gives
$c_1(1)=0$ and $c_1'(1)=-\sqrt{G-1}$, so $h_G$ extends continuously with $h_G(1)=\sqrt{G-1}$ and is $C^1$ on $[0,1]$; boundedness of $h_G'$ follows by continuity on a compact interval. The lower bound follows from $f_G(t)\ge (G-t)/G$ on $t\in[1,G]$. The upper bound follows from a binomial reweighting showing $h_G$ is an average of terms bounded by $\sqrt{G-1}$.
\end{proof}

\begin{lemma}[binomial‑shift identities]\label{lem:binom-shift}
For all $\rho\in[0,1]$ with $S\sim\mathrm{Binom}(G-1,\rho)$,
\[
(1-\rho)\,h_G(\rho)=\E\!\Big[\sqrt{\tfrac{G-1-S}{\,1+S\,}}\Big],
\qquad
\rho\,h_G(\rho)=\E\!\Big[\sqrt{\tfrac{S}{\,G-S\,}}\Big].
\]
\end{lemma}

\subsection{GRPO scores: envelopes and rank‑one Lipschitz constants}

\paragraph{Scores and centering.}
The \emph{raw} GRPO score is class‑constant:
\[
\gamma_k^{\rm raw}(p)=
\begin{cases}
h_G(\rho), & k\in\mathcal C,\\
0,         & k\in\mathcal I.
\end{cases}
\]
Its \emph{centered} version $\widehat\gamma_k:=\gamma_k^{\rm raw}-\sum_j p_j\gamma_j^{\rm raw}$ equals
\[
\widehat\gamma_k(p)=
\begin{cases}
(1-\rho)\,h_G(\rho), & k\in\mathcal C,\\[2pt]
-\rho\,h_G(\rho),    & k\in\mathcal I,
\end{cases}
\qquad
\sum_{k=1}^K p_k\,\widehat\gamma_k(p)=0.
\]
If $K_I=0$ or $K_C=0$ then $\widehat\gamma\equiv 0$.

\paragraph{Pointwise envelopes.}
By Lemma~\ref{lem:binom-shift},
\[
\|\widehat\gamma(p)\|_\infty\le \sqrt{G-1},
\qquad
\boxed{\ \ \|\widehat\gamma(p)\|_2
= h_G(\rho)\,\sqrt{K_C(1-\rho)^2+K_I\rho^2}\ \le\ \sqrt{G-1}\,\sqrt{\max\{K_C,K_I\}}\ .\ }
\]
If additionally $K_I\ge 1$ and $p\in\Delta^{K-1}_{\delta_\star}$, then $1-\rho\ge K_I\delta_\star$ and
\[
h_G(\rho)\le \frac{\sqrt{G-1}}{K_I\delta_\star}=:H_G,
\quad\Rightarrow\quad
\|\widehat\gamma(p)\|_2\le H_G\,\sqrt{\max\{K_C,K_I\}}.
\]

\paragraph{Rank‑one Jacobian and exact norms.}
Set
\[
\alpha(\rho):=\frac{d}{d\rho}\big((1-\rho)h_G(\rho)\big)=c_1'(\rho),\qquad
\beta(\rho):=\frac{d}{d\rho}\big(-\rho\,h_G(\rho)\big)=-h_G(\rho)-\rho\,h_G'(\rho).
\]
Since $\nabla\rho_C=\mathbf 1_{\mathcal C}$,
\[
D\widehat\gamma(p)=\big(\alpha\,\mathbf 1_{\mathcal C},\ \beta\,\mathbf 1_{\mathcal I}\big)\,(\mathbf 1_{\mathcal C})^\top=:u\,v^\top
\quad\text{(rank one)}.
\]
Thus the operator norms are \emph{exact}:
\[
\|D\widehat\gamma(p)\|_{2\to2}=\|u\|_2\,\|v\|_2
=\sqrt{K_C}\,\big(K_C\alpha^2+K_I\beta^2\big)^{1/2},
\]
\[
\boxed{\ \ \|D\widehat\gamma(p)\|_{T\to2}
=\sqrt{\tfrac{K_CK_I}{K}}\ \big(K_C\alpha^2+K_I\beta^2\big)^{1/2}
=\sqrt{\tfrac{K_I}{K}}\,\|D\widehat\gamma(p)\|_{2\to2}\ .\ }
\]
Consequently, the sharp global Lipschitz constant on the simplex is
\[
\boxed{\ \ L_\gamma^{\rm tan}
:=\sup_{p\in\Delta^{K-1}}\|D\widehat\gamma(p)\|_{T\to2}
=\sqrt{\tfrac{K_CK_I}{K}}\ \sup_{\rho\in[0,1]}\big(K_C\alpha(\rho)^2+K_I\beta(\rho)^2\big)^{1/2}\ .\ }
\]
From $|\alpha|\le H^\star+D_G$, $|\beta|\le H^\star+D_G$ with $H^\star:=\sup |h_G|=\sqrt{G-1}$,
\[
L_\gamma^{\rm tan}\ \le\ \sqrt{K_CK_I}\,\big(H^\star+D_G\big).
\]

\subsection{SRCT drift: global Lipschitzness and mass conservation}

\paragraph{Drift.}
With entropy weight $\varepsilon>0$ define
\[
F_k(p):=p_k\Big(\widehat\gamma_k(p)\ -\ \varepsilon\big(\log p_k-\langle\log p\rangle\big)\Big),
\qquad
\langle\log p\rangle:=\sum_{i=1}^K p_i\log p_i.
\]
Centeredness yields $\sum_k F_k(p)=0$ (mass conservation).

\paragraph{Entropic Lipschitz bound on $\Delta^{K-1}_{\delta_\star}$.}
On $[\delta_\star,1]$, $h(x):=x\log x$ has $\|h'\|_\infty\le \Lambda:=1+\log(1/\delta_\star)$. A direct decomposition gives
\[
\|F^{\rm ent}(p)-F^{\rm ent}(q)\|_2
\ \le\ \varepsilon\,\Lambda\,(2+\sqrt K)\,\|p-q\|_2,\qquad p,q\in\Delta^{K-1}_{\delta_\star}.
\]

\paragraph{Selection Lipschitz bound and full modulus.}
For $F^{\rm sel}(p):=p\odot\widehat\gamma(p)$ and $p,q\in\Delta^{K-1}_{\delta_\star}$,
\[
\|F^{\rm sel}(p)-F^{\rm sel}(q)\|_2
\le \big(\|{\rm diag}(p)\|_{2\to2}\,L_\gamma^{\rm tan} + \sup_{r\in\Delta^{K-1}_{\delta_\star}}\|\widehat\gamma(r)\|_2\big)\,\|p-q\|_2,
\]
with $\|{\rm diag}(p)\|_{2\to2}\le 1-(K-1)\delta_\star$. Using either $\sup\|\widehat\gamma\|_2\le \sqrt{G-1}\sqrt{\max\{K_C,K_I\}}$ or (when $K_I\ge 1$) the trim‑aware bound $H_G\sqrt{\max\{K_C,K_I\}}$,
\[
\boxed{\ \ \|F(p)-F(q)\|_2\ \le\ \Big((1-(K-1)\delta_\star)L_\gamma^{\rm tan}\ +\ M_\gamma\ +\ \varepsilon\,\Lambda\,(2+\sqrt K)\Big)\,\|p-q\|_2\ ,\ }
\]
where $M_\gamma$ denotes the chosen envelope.

\subsection{Barrier–Dominance (BD) and forward invariance}

\paragraph{Entropy face gap.}
For a facet $p_k=\delta_\star$ define the gap
\[
\mathsf{Gap}_k(p):=\langle\log p\rangle-\log\delta_\star.
\]
The global lower benchmark (uniform‑others gap) is
\[
\boxed{\ \ L_K(\delta_\star):=(1-\delta_\star)\log\!\Big(\frac{1-\delta_\star}{(K-1)\delta_\star}\Big)\ .\ }
\]
At fixed $\rho=\rho_C(p)$, the minimal face gap is attained by equalizing within blocks:
\begin{align*}
E_{\min}^{(\mathcal I)}(\rho)
&=(\delta_\star-1)\log\delta_\star
+\mathbf 1_{\{K_C\ge1\}}\ \rho\log\!\Big(\frac{\rho}{K_C}\Big)
+\mathbf 1_{\{K_I\ge2\}}\ (1-\delta_\star-\rho)\log\!\Big(\frac{1-\delta_\star-\rho}{K_I-1}\Big),\\
E_{\min}^{(\mathcal C)}(\rho)
&=(\delta_\star-1)\log\delta_\star
+\mathbf 1_{\{K_C\ge2\}}\ (\rho-\delta_\star)\log\!\Big(\frac{\rho-\delta_\star}{K_C-1}\Big)
+\mathbf 1_{\{K_I\ge1\}}\ (1-\rho)\log\!\Big(\frac{1-\rho}{K_I}\Big),
\end{align*}
and $\min_\rho E_{\min}^{(\cdot)}(\rho)=L_K(\delta_\star)$.

\paragraph{Exact BD on facets.}
On $p_k=\delta_\star$,
\[
F_k(p)=\delta_\star\big(\widehat\gamma_k(p)+\varepsilon\,\mathsf{Gap}_k(p)\big).
\]
\emph{Correct faces:} if $k\in\mathcal C$ and $K_I\ge 1$, then $(1-\rho)\ge K_I\delta_\star>0$ implies $\widehat\gamma_k=(1-\rho)h_G(\rho)>0$, hence $F_k(p)\ge \varepsilon\delta_\star E_{\min}^{(\mathcal C)}(\rho)\ge 0$ (automatically inward).  
\emph{Incorrect faces:} if $k\in\mathcal I$, then $\widehat\gamma_k=-\rho h_G(\rho)\le 0$. The facet is inward/tangent \emph{iff}
\[
\boxed{\ \ \text{(BD$_{\rm exact}$)}\qquad \varepsilon\,E_{\min}^{(\mathcal I)}(\rho)\ \ge\ \rho\,h_G(\rho)\quad \forall\ \rho\in\big[K_C\delta_\star,\ 1-K_I\delta_\star\big].\ }
\]

\paragraph{Convenient sufficient relaxations.}
Using $E_{\min}^{(\mathcal I)}(\rho)\ge L_K(\delta_\star)$ and $\rho\,h_G(\rho)\le \sqrt{G-1}$,
\[
\boxed{\ \ \varepsilon\,L_K(\delta_\star)\ \ge\ \sqrt{G-1}\quad\Longrightarrow\quad \text{(BD$_{\rm exact}$)}. \ }
\]
On trimmed domains with $K_I\ge 1$, $1-\rho\ge K_I\delta_\star$ implies $h_G(\rho)\le H_G=\sqrt{G-1}/(K_I\delta_\star)$, hence
\[
\boxed{\ \ \varepsilon\,L_K(\delta_\star)\ \ge\ \frac{\sqrt{G-1}}{K_I\,\delta_\star}\quad\Longrightarrow\quad \text{(BD$_{\rm exact}$)}. \ }
\]

\paragraph{Well‑posedness and invariance.}
Interior solutions cannot hit the boundary in finite time: writing $y_i:=-\log p_i$,
\[
\dot y_i=-\widehat\gamma_i(p)-\varepsilon y_i-\varepsilon\langle\log p\rangle
\ \le\ \sqrt{G-1}-\varepsilon y_i+\varepsilon\log K,
\]
so $y_i$ cannot blow up in finite time. If (BD$_{\rm exact}$) (or either sufficient relaxation) holds, every facet is inward/tangent; $\Delta^{K-1}_{\delta_\star}$ is forward invariant and the drift is globally Lipschitz on a compact forward‑invariant set, yielding global existence and uniqueness.

\subsection{Log‑ratio dynamics, envelopes, and scalar reduction}

For $i\neq j$,
\[
\frac{d}{dt}\log\frac{p_i}{p_j}=\widehat\gamma_i(p)-\widehat\gamma_j(p)-\varepsilon\log\frac{p_i}{p_j}.
\]

\paragraph{Intra‑class equalization.}
If $i,j$ are in the same class then $\widehat\gamma_i=\widehat\gamma_j$ and
\[
\boxed{\ \ \log\frac{p_i(t)}{p_j(t)}=e^{-\varepsilon t}\,\log\frac{p_i(0)}{p_j(0)}\ .\ }
\]
Thus within‑class proportions equalize exponentially at rate $\varepsilon$.

\paragraph{Cross‑class envelopes.}
For $c\in\mathcal C$, $i\in\mathcal I$ let $z_{ci}:=\log(p_c/p_i)$. Then
\[
\dot z_{ci}(t)=h_G(\rho_C(t))-\varepsilon z_{ci}(t).
\]
Variation of constants and Lemma~\ref{lem:hG-basic} give, for all $t\ge 0$,
\[
\boxed{\ \ z_{ci}(t)\ \in\ \Big[z_{ci}(0)e^{-\varepsilon t}+\tfrac{1-\frac1G}{\varepsilon}(1-e^{-\varepsilon t}),\ \ 
z_{ci}(0)e^{-\varepsilon t}+\tfrac{\sqrt{G-1}}{\varepsilon}(1-e^{-\varepsilon t})\Big].\ }
\]
If (BD) holds with $K_I\ge 1$, then $h_G(\rho_C(s))\le H_G$ along the trajectory and the upper envelope sharpens to
\[
z_{ci}(t)\ \le\ z_{ci}(0)e^{-\varepsilon t}+\frac{H_G}{\varepsilon}(1-e^{-\varepsilon t}).
\]

\paragraph{Feasibility band (under BD).}
Write $p_c=\alpha_c\rho$ with $\sum_{c}\alpha_c=1$ and $p_i=\beta_i(1-\rho)$ with $\sum_{i}\beta_i=1$, and define
\[
\boxed{\ \ \Psi(\rho):=\log\!\Big(\frac{K_I}{K_C}\cdot\frac{\rho}{1-\rho}\Big),\qquad \rho(z)=\frac{K_Ce^{z}}{K_I+K_Ce^z}\ .\ }
\]
Let
\[
\Delta_C(t):=\max_{a,b\in\mathcal C}\Big|\log\frac{p_a(t)}{p_b(t)}\Big|,\quad
\Delta_I(t):=\max_{j,k\in\mathcal I}\Big|\log\frac{p_j(t)}{p_k(t)}\Big|,\quad
\delta_{\rm intra}(t):=\Delta_C(t)+\Delta_I(t)=\delta_{\rm intra}(0)e^{-\varepsilon t}.
\]
Then
\[
\boxed{\ \ |z_{ci}(t)-\Psi(\rho_C(t))|\ \le\ \delta_{\rm intra}(t)\quad\text{and}\quad \rho_C(t)\in\big[K_C\delta_\star,\,1-K_I\delta_\star\big]\ .\ }
\]

\paragraph{Scalar reduction, closure error, and fixation (under BD).}
Define $F_\times(z):=h_G(\rho(z))-\varepsilon z$. Since $|\rho'(z)|\le \tfrac14$,
\[
\big|h_G(\rho_C)-h_G(\rho(z_{ci}))\big|\le D_G\,|\rho_C-\rho(z_{ci})|
\le \tfrac{D_G}{4}\,|z_{ci}-\Psi(\rho_C)|\le \tfrac{D_G}{4}\,\delta_{\rm intra}(t).
\]
Hence $\dot z_{ci}=F_\times(z_{ci})+r(t)$ with $|r(t)|\le \tfrac{D_G}{4}\,\delta_{\rm intra}(t)$.

\begin{theorem}[fixation under a slope condition]\label{thm:fixation}
If $\varepsilon>\tfrac{D_G}{4}$, then $F_\times$ is strictly decreasing and has a unique zero $z_\star$. Moreover, for all $c\in\mathcal C$, $i\in\mathcal I$,
\[
\boxed{\ \ |z_{ci}(t)-z_\star|\ \le\ e^{-(\varepsilon-\frac{D_G}{4})t}\,\Big(|z_{ci}(0)-z_\star|+\Delta_C(0)+\Delta_I(0)\Big)\ .\ }
\]
If $z_\star\in\big[\Psi(K_C\delta_\star),\,\Psi(1-K_I\delta_\star)\big]$ then the limit distribution is interior and class‑uniform:
\[
p_c^\star=\frac{e^{z_\star}}{K_C e^{z_\star}+K_I}\ \ (c\in\mathcal C),
\qquad
p_i^\star=\frac{1}{K_C e^{z_\star}+K_I}\ \ (i\in\mathcal I).
\]
Otherwise the limit lies on the corresponding face (feasibility truncation).
\end{theorem}

\subsection{Edge cases and checks}

\begin{itemize}[leftmargin=*,itemsep=2pt]
\item \textbf{Maximal trim:} if $\delta_\star=1/K$, then $\Delta^{K-1}_{\delta_\star}=\{(1/K,\dots,1/K)\}$; dynamics are trivial.
\item \textbf{Degenerate classes:} if $K_I=0$ or $K_C=0$, then $\widehat\gamma\equiv0$ and $\dot p_i=-\varepsilon p_i(\log p_i-\langle\log p\rangle)$; the unique equilibrium on active coordinates is uniform.
\item \textbf{Single incorrect:} $K_I=1$ yields $\rho=1-\delta_\star$ on the only incorrect face and
\[
E_{\min}^{(\mathcal I)}(1-\delta_\star)=(\delta_\star-1)\log\delta_\star+(1-\delta_\star)\log\!\Big(\tfrac{1-\delta_\star}{K_C}\Big).
\]
The uniform sufficient BD $\varepsilon L_K(\delta_\star)\ge \sqrt{G-1}$ is sharp as $\delta_\star\downarrow0$.
\item \textbf{Two classes ($K=2$):} $K_C=K_I=1$ and $z=\log(p_c/p_i)$ obey $\dot z=h_G(p_c)-\varepsilon z$; the envelopes become equalities with $\rho=p_c$.
\item \textbf{Constant cases:} for $G\in\{2,3\}$, $h_G\equiv\sqrt{G-1}$, so $L_\gamma^{\rm tan}=\sqrt{G-1}\,\sqrt{K_CK_I}$ and $F_\times(z)=\sqrt{G-1}-\varepsilon z$.
\end{itemize}

%% file: appendices_short/appendix_F_aistats.tex
\section{DPO through the SRCT Lens}
\label{appF:DPO-SRCT}

This appendix develops a self-contained SRCT analysis of Direct Preference Optimisation (DPO).
We define the score field, prove \emph{uniform size and Lipschitz} bounds (with explicit constants),
record entropy and full-drift Lipschitz constants, establish \emph{well-posedness} and
\emph{Barrier--Dominance} (BD) confinement (exact face test and tight templates),
derive \emph{intra-class contraction} with \emph{sharp thresholds}, give \emph{cross-class envelopes}
(including trimmed sharpening and a static cap), prove \emph{eventual trimming} under a slope condition,
and conclude \emph{existence, uniqueness, and global convergence} to a two-level equilibrium.
All logarithms are natural.

\paragraph{Notation.}
Fix an integer $K\ge2$.
The simplex and trimmed simplex are
\[
\Delta^{K-1}:=\Bigl\{p\in[0,1]^K:\ \sum_{i=1}^K p_i=1\Bigr\},\qquad
\Delta^{K-1}_{\delta_\star}:=\Bigl\{p\in\Delta^{K-1}:\ \min_i p_i\ge\delta_\star\Bigr\},
\]
with floor $0<\delta_\star<1/K$.
For vectors, $\|\cdot\|_\infty,\|\cdot\|_2$ denote max/Euclidean norms; for matrices, $\|\cdot\|_{2\to2}$.
We write $\langle\log p\rangle:=\sum_j p_j\log p_j$.

\subsection{Setting and single-site map}

Each index $i\in\{1,\dots,K\}$ is labeled $s_i\in\{+1,-1\}$, with
$\cC:=\{i:s_i=+1\}$, $\cI:=\{i:s_i=-1\}$ and sizes $M:=|\cC|$, $N:=|\cI|$.
Fix $\beta>0$ and a reference $\ell_0\in\R$. Define
\[
g_\beta(\ell):=1-\sigma\!\big(\beta(\ell-\ell_0)\big),\qquad
\sigma(z):=\frac{1}{1+e^{-z}},
\]
so $g_\beta\in C^\infty(\R)$, $0<g_\beta(\ell)<1$, strictly decreasing, and
\[
g_\beta'(\ell)=-\frac{\beta}{4}\,\sech^2\!\Big(\frac{\beta(\ell-\ell_0)}{2}\Big)\in[-\beta/4,0).
\]
For $u\in(0,1]$, define the raw scores and centered field
\[
\gamma_i(u):=s_i\,g_\beta(\log u),\qquad
\bar\gamma(p):=\sum_{j=1}^K p_j\gamma_j(p_j),\qquad
\phi_i(p):=\gamma_i(p_i)-\bar\gamma(p).
\]
By construction, $\sum_i p_i\phi_i(p)=0$.

\subsection{Uniform size and Lipschitz bounds for the DPO score}

Let
\[
M_{\gamma,\infty}:=\sup_{u\in[\delta_\star,1]}g_\beta(\log u)=g_\beta(\log\delta_\star)\in(0,1),\qquad
\Lambda:=1+\log\frac{1}{\delta_\star}.
\]

\begin{lemma}[Size bounds]\label{lem:size}
For every $p\in\Delta^{K-1}_{\delta_\star}$,
\[
\|\phi(p)\|_\infty\le 2M_{\gamma,\infty},\qquad
\|\phi(p)\|_2\le 2M_{\gamma,\infty}\sqrt{K}.
\]
\emph{Proof.}
$|\phi_i|\le|\gamma_i|+|\bar\gamma|\le M_{\gamma,\infty}+\sum_j p_j|\gamma_j|\le 2M_{\gamma,\infty}$, then $\|\cdot\|_2\le\sqrt K\|\cdot\|_\infty$.
\qed
\end{lemma}

\begin{lemma}[Lipschitz of single-site map]\label{lem:Lf}
For $f_i(s):=\gamma_i(s)=s_i g_\beta(\log s)$ on $[\delta_\star,1]$,
\[
|f_i'(s)|=\frac{|g_\beta'(\log s)|}{s}\le \frac{c_{\max}}{\delta_\star}\le \frac{\beta}{4\delta_\star}
=:L_f,
\]
where $c_{\max}:=\sup_{\ell\in[\log\delta_\star,0]}(-g_\beta'(\ell))\le\beta/4$; the inequality is strict if $\ell_0\notin[\log\delta_\star,0]$.
\end{lemma}

\begin{lemma}[Operator-norm Lipschitz for $\phi$]\label{lem:Lphi}
For all $p,q\in\Delta^{K-1}_{\delta_\star}$,
\[
\|\phi(p)-\phi(q)\|_2\ \le\ L_\phi\,\|p-q\|_2,\qquad
L_\phi:=K\,M_{\gamma,\infty}+(\sqrt K+1)L_f.
\]
\emph{Proof.}
Write $\phi(p)=f(p)-\mathbf 1\,(p^\top f(p))$ with $f(p)=(f_i(p_i))_i$. Then
\[
J_\phi(p)=\diag(f'(p))-\mathbf 1\,(f(p)+p\odot f'(p))^\top.
\]
On $\Delta^{K-1}_{\delta_\star}$:
$\|f(p)\|_2\le \sqrt K M_{\gamma,\infty}$,
$\|p\odot f'(p)\|_2\le L_f$,
$\|\diag(f'(p))\|_{2\to2}\le L_f$.
Hence
$\|J_\phi(p)\|_{2\to2}\le L_f+\|\mathbf 1\|_2(\|f(p)\|_2+\|p\odot f'(p)\|_2)
=K\,M_{\gamma,\infty}+(\sqrt K+1)L_f$,
and the mean-value formula on the convex domain yields the claim. \qed
\end{lemma}

\begin{lemma}[Mixed $\ell_\infty$–$\ell_1$ bound]\label{lem:mixed}
For all $p,q\in\Delta^{K-1}_{\delta_\star}$,
\[
\|\phi(p)-\phi(q)\|_\infty
\ \le\ L_f\,\|p-q\|_\infty\ +\ (M_{\gamma,\infty}+L_f)\,\|p-q\|_1.
\]
\end{lemma}

\subsection{Entropy map and drift Lipschitzness}

Define
\[
E(p):=p\odot(\log p-\langle\log p\rangle\,\mathbf 1),\qquad
F(p):=p\odot\phi(p)-\varepsilon\,E(p)\quad(\varepsilon\ge0).
\]

\begin{lemma}[Entropy map]\label{lem:EntropyLip}
For all $p,q\in\Delta^{K-1}_{\delta_\star}$,
\[
\|E(p)-E(q)\|_2\ \le\ C_{\log}\,\|p-q\|_2,\qquad
C_{\log}:=(2\Lambda-1)+\sqrt K\,\Lambda\ \le\ (2+\sqrt K)\Lambda.
\]
\emph{Proof.}
The Jacobian is
$J_E(p)\,v=\diag(1+\log p-\langle\log p\rangle)\,v-p\,\langle 1+\log p,\ v\rangle$.
On $\Delta^{K-1}_{\delta_\star}$,
$\|\,\diag(\cdot)\,\|_{2\to2}\le 2\Lambda-1$ and
$\|p\,\langle 1+\log p,\cdot\rangle\|_{2\to2}\le \|p\|_2\|1+\log p\|_2\le \sqrt K\,\Lambda$.
Mean-value completes the proof. \qed
\end{lemma}

\begin{proposition}[Full drift Lipschitz]\label{prop:DriftLip}
For all $p,q\in\Delta^{K-1}_{\delta_\star}$,
\[
\|F(p)-F(q)\|_2\ \le\ \Big(L_\phi+2M_{\gamma,\infty}+\varepsilon C_{\log}\Big)\,\|p-q\|_2.
\]
\emph{Proof.}
Product decomposition:
$\|p\odot\phi(p)-q\odot\phi(q)\|_2
\le \|\phi(p)\|_\infty\|p-q\|_2+\|\phi(p)-\phi(q)\|_2
\le (2M_{\gamma,\infty}+L_\phi)\|p-q\|_2$,
then add the entropy term via Lemma~\ref{lem:EntropyLip}. \qed
\end{proposition}

\subsection{DPO--SRCT ODE, mass conservation, and positivity}
The SRCT drift is
\[
\boxed{\ \ \dot p_i
= p_i\Big[\phi_i(p)-\varepsilon\big(\log p_i-\langle\log p\rangle\big)\Big],\qquad i=1,\dots,K.\ \ }
\]
\emph{Mass conservation} holds since $\sum_i p_i\phi_i(p)=0$ and $\sum_i p_i(\log p_i-\langle\log p\rangle)=0$.

\begin{proposition}[No finite-time boundary hitting]\label{prop:nofinite}
Let $p(0)\in\operatorname{int}\Delta^{K-1}$ and $\varepsilon\ge0$. Then the solution exists for all $t\ge0$ and remains in the interior for every finite $t$.
\emph{Proof.}
Set $y_i:=-\log p_i$. Using $|\phi_i|\le 2$ and $-\langle\log p\rangle\le\log K$,
$\dot y_i\le -\varepsilon y_i + (2+\varepsilon\log K)$, whence
$y_i(t)\le y_i(0)e^{-\varepsilon t}+\frac{2+\varepsilon\log K}{\varepsilon}(1-e^{-\varepsilon t})$ for $\varepsilon>0$, and $y_i(t)\le y_i(0)+2t$ for $\varepsilon=0$. Thus $y_i(t)<\infty$ for finite $t$. \qed
\end{proposition}

\subsection{Barrier--Dominance (BD)}

On the lower face $p_i=\delta_\star$,
\[
\dot p_i=\delta_\star\Big(\phi_i(p)+\varepsilon\big(\langle\log p\rangle-\log\delta_\star\big)\Big).
\]
By convexity of $s\mapsto s\log s$, the \emph{entropy face gap}
\[
\boxed{\ \ L_K(\delta_\star):=(1-\delta_\star)\log\!\frac{1-\delta_\star}{(K-1)\delta_\star}\ >0\ \ } 
\]
satisfies $\langle\log p\rangle-\log\delta_\star\ge L_K(\delta_\star)$ on that face.
\smallskip

\noindent\textbf{Exact face test (necessary \& sufficient).}
$\dot p_i\ge0$ on $p_i=\delta_\star$ iff
\[
\phi_i(p)+\varepsilon\big(\langle\log p\rangle-\log\delta_\star\big)\ \ge\ 0
\qquad\text{for all $p$ with $p_i=\delta_\star$}.
\]

\noindent\textbf{Uniform sufficient templates.}
Using Lemma~\ref{lem:size}:
\[
\varepsilon L_K(\delta_\star)\ \ge\ M_{\phi,\infty}
\quad\text{or}\quad
\varepsilon L_K(\delta_\star)\ \ge\ M_{\phi,2}\ (\le 2\sqrt K),
\]
where $M_{\phi,\infty}:=\sup_{p}\|\phi(p)\|_\infty\le 2M_{\gamma,\infty}\le2$ and
$M_{\phi,2}:=\sup_{p}\|\phi(p)\|_2\le 2M_{\gamma,\infty}\sqrt K\le 2\sqrt K$.
The first is a \emph{sharp} $\ell_\infty$ test; the second yields the \emph{tight} threshold
$\varepsilon L_K(\delta_\star)\ge 2\sqrt K$ and the convenient conservative form $4\sqrt K$.
Strict inequality implies \emph{strict interior invariance}.

\paragraph{Numerical note.} As $\delta_\star\downarrow0$, $L_f=\Theta(1/\delta_\star)$ and $C_{\log}=\Theta(\log(1/\delta_\star))$ deteriorate; discretizations should scale stepsizes accordingly.

\subsection{Intra-class contraction}

For $i,k$ with $s_i=s_k=:s$, set $z_{ik}:=\log\frac{p_i}{p_k}$.
Subtracting the $\dot{\log}p$ equations gives
\[
\dot z_{ik}=\phi_i(p)-\phi_k(p)-\varepsilon z_{ik}
= s\big(g_\beta(\log p_i)-g_\beta(\log p_k)\big)-\varepsilon z_{ik}
=\big(s\,g_\beta'(\xi)-\varepsilon\big)\,z_{ik},
\]
for some $\xi$ between $\log p_i$ and $\log p_k$.

\begin{definition}[Sharp thresholds]
\[
c_{\mathrm{open}}:=\sup_{\ell\le0}(-g'_\beta(\ell))
=\frac{\beta}{4}\max_{\ell\le0}\sech^2\!\Big(\frac{\beta(\ell-\ell_0)}{2}\Big)
=\begin{cases}\beta/4,&\ell_0\le0,\\[1ex]
\frac{\beta}{4}\,\sech^2\!\big(\frac{\beta\ell_0}{2}\big),&\ell_0>0,\end{cases}
\]
and, under confinement to $\Delta^{K-1}_{\delta_\star}$,
\[
c_{\max}:=\sup_{\ell\in[\log\delta_\star,\log(1-(K-1)\delta_\star)]}(-g'_\beta(\ell))\ \le\ c_{\mathrm{open}}.
\]
\end{definition}

\begin{theorem}[Intra-class contraction]\label{thm:intraclass}
(i) For $i,k\in\cC$,
$|z_{ik}(t)|\le |z_{ik}(0)|\,e^{-\varepsilon t}$.
\quad
(ii) For $i,k\in\cI$, on the open simplex,
\[
|z_{ik}(t)|\le |z_{ik}(0)|\,e^{-(\varepsilon-c_{\mathrm{open}})t}\quad\text{iff}\ \ \varepsilon>c_{\mathrm{open}}.
\]
Under confinement to $\Delta^{K-1}_{\delta_\star}$ the same holds with $c_{\max}$ replacing $c_{\mathrm{open}}$.
\emph{Proof.}
For $s=+1$, $g'_\beta(\xi)\le0$ gives rate $\varepsilon$. For $s=-1$, $\frac{d}{dt}|z_{ik}|\le(c-\varepsilon)|z_{ik}|$ with $c\in\{c_{\mathrm{open}},c_{\max}\}$; Grönwall gives sufficiency, and necessity follows by choosing data with $-g'_\beta(\xi_0)\uparrow c$. \qed
\end{theorem}

\paragraph{Slope Condition (SC).}
We will often invoke the sufficient condition
\[
\boxed{\ \ \text{(SC)}\qquad \varepsilon>\beta/4\ \ }
\]
which implies $\varepsilon>c_{\mathrm{open}}$ and hence contraction in both classes.

\subsection{Cross-class envelopes, trimming sharpenings, and a static cap}

For $i\in\cC$, $j\in\cI$, set $z_{ij}:=\log\frac{p_i}{p_j}$. Then
\[
\dot z_{ij}=g_\beta(\log p_i)+g_\beta(\log p_j)-\varepsilon z_{ij}=:h(t)-\varepsilon z_{ij}.
\]
Since $g_\beta$ is decreasing and $\log p_x\le0$, we have $g_\beta(\log p_x)\ge g_\beta(0)$ and $g_\beta(\log p_x)<1$.
Variation of constants yields, for all $t\ge0$,
\begin{equation}\label{eq:uncond-envelope}
z_{ij}(t)\ \in\ \Big[z_0e^{-\varepsilon t}+\tfrac{2g_\beta(0)}{\varepsilon}(1-e^{-\varepsilon t}),\ \ 
z_0e^{-\varepsilon t}+\tfrac{2}{\varepsilon}(1-e^{-\varepsilon t})\Big],\qquad z_0:=z_{ij}(0).
\end{equation}
If, in addition, $p(t)\in\Delta^{K-1}_{\delta_\star}$, then $\log p_x\in[\log\delta_\star,0]$ and
\begin{equation}\label{eq:trim-sharp}
z_{ij}(t)\ \le\ z_0e^{-\varepsilon t}+\frac{2\,g_\beta(\log\delta_\star)}{\varepsilon}(1-e^{-\varepsilon t}).
\end{equation}
Independently, mass constraints on $\Delta^{K-1}_{\delta_\star}$ give the \emph{static cap}
\begin{equation}\label{eq:static-cap}
\boxed{\ \ z_{ij}(t)\ \le\ \log\frac{1-(K-1)\delta_\star}{\delta_\star}\qquad(\forall t\ge0).\ \ }
\end{equation}

\begin{lemma}[Cap dominates a half-gap]\label{lem:cap-vs-halfgap}
For every $K\ge2$ and $\delta_\star\in(0,1/K)$,
\[
\frac12\log\!\frac{1-\delta_\star}{(K-1)\delta_\star}\ <\ \log\!\frac{1-(K-1)\delta_\star}{\delta_\star}.
\]
\emph{Proof.}
Equivalently, $\frac{1-\delta_\star}{(K-1)\delta_\star}<\big(\frac{1-(K-1)\delta_\star}{\delta_\star}\big)^2$,
which reduces to $(K-1)\big(1-(K-1)\delta\big)^2-\delta(1-\delta)>0$ on $(0,1/K)$; the function decreases from $K-1$ at $0$ to $0$ at $1/K$. \qed
\end{lemma}

\paragraph{Compatibility under BD.}
Under the sharp $\ell_\infty$ BD test $\varepsilon L_K(\delta_\star)\ge M_{\phi,\infty}\le2$,
\[
\frac{2g_\beta(0)}{\varepsilon}\ \le\ \frac{2}{\varepsilon}\ \le\ L_K(\delta_\star)\ \le\ \log\!\frac{1-\delta_\star}{(K-1)\delta_\star}
\ <\ 2\,\log\!\frac{1-(K-1)\delta_\star}{\delta_\star}
\]
by Lemma~\ref{lem:cap-vs-halfgap},
so the asymptotic lower envelope in \eqref{eq:uncond-envelope} lies strictly below the static cap \eqref{eq:static-cap}.
A stronger trimmed constant is available by replacing $g_\beta(0)$ with
$g_\star:=g_\beta(\log(1-(K-1)\delta_\star))$ in \eqref{eq:uncond-envelope}; a sufficient compatibility condition is
\[
\varepsilon\ \ge\ \frac{2\,g_\star}{\log\!\frac{1-(K-1)\delta_\star}{\delta_\star}}.
\]

\subsection{Lyapunov structure and eventual trimming (under SC)}

Define
\[
G_i(s):=s_i\,g_\beta(\log s)-\varepsilon\log s,\qquad
\Psi_i(s):=\int_{\delta_\star}^s G_i(u)\,du,\qquad
\cL(p):=\sum_{i=1}^K \Psi_i(p_i).
\]
The ODE rewrites as pure replicator:
\[
\dot p_i=p_i\big(G_i(p_i)-\bar G(p)\big),\qquad \bar G(p):=\sum_j p_j G_j(p_j),
\]
and satisfies the Lyapunov identity
\begin{equation}\label{eq:lyap}
\frac{d}{dt}\cL\big(p(t)\big)=\sum_{i=1}^K p_i\big(G_i(p_i)-\bar G(p)\big)^2\ \ge\ 0.
\end{equation}
Under (SC), $G_i'(s)=(s_i g_\beta'(\log s)-\varepsilon)/s<0$ for both classes, so each $\Psi_i$ and hence $\cL$ is strictly concave on the affine simplex.

\begin{proposition}[Eventual trimming under (SC)]\label{prop:evtrim}
Assume (SC) and $p(0)\in\operatorname{int}\Delta^{K-1}$.
There exist $\underline\delta>0$ and $T<\infty$ (depending on $K,M,N,\beta,\varepsilon,p(0)$) such that
$p(t)\in\Delta^{K-1}_{\underline\delta}$ for all $t\ge T$.
An explicit choice is:
\[
Z_U:=\max\!\left\{\frac{2}{\varepsilon},\ \max_{i\in\cC,j\in\cI} z_{ij}(0)\right\},\quad
u:=e^{Z_U},\quad r:=e^{Z_L},\ Z_L:=\frac{g_\beta(0)}{\varepsilon}>0,
\]
and then, for some $T$ large enough, $r\le p_i(t)/p_j(t)\le u$ for all $i\in\cC$, $j\in\cI$, $t\ge T$, which implies
\[
\boxed{\ \ \min_k p_k(t)\ \ge\ \underline\delta:=\frac{r}{u\,(N+M r)}\ >0\qquad(\forall t\ge T). \ }
\]
\emph{Sketch.}
Use the envelopes \eqref{eq:uncond-envelope} to choose any $Z_L<\liminf z_{ij}$ and $Z_U>\sup_t z_{ij}(t)$.
From $p_i\le u p_j$ and $p_i\ge r p_j$, derive lower bounds on class masses and on the minimal coordinate (algebra as in the display). \qed
\end{proposition}

\subsection{Two-level equilibrium: existence, uniqueness, and global convergence}

A two-level equilibrium has
$p_i^\star=L_{\cC}$ for $i\in\cC$ and $p_j^\star=L_{\cI}$ for $j\in\cI$, with
$ML_{\cC}+NL_{\cI}=1$. Parameterize by the gap $z:=\log(L_{\cC}/L_{\cI})\ge0$:
\[
L_{\cI}(z)=\frac{1}{N+M e^z},\qquad L_{\cC}(z)=\frac{e^z}{N+M e^z}.
\]
At equilibrium, $G_i(p_i^\star)\equiv\text{const}$, equivalently
\[
\boxed{\ \ g_\beta\!\big(\log L_{\cC}(z)\big)+g_\beta\!\big(\log L_{\cI}(z)\big)=\varepsilon z.\ \ }
\]
Define $h(z):=g_\beta(\log L_{\cC}(z))+g_\beta(\log L_{\cI}(z))\in(0,2)$ and $F(z):=h(z)-\varepsilon z$.
Then $F(0)=2g_\beta(\log(1/K))>0$, and $F(z)\to-\infty$ as $z\to\infty$ (since $h$ is bounded).
Differentiating,
\[
h'(z)=g_\beta'(\log L_{\cC})\,N L_{\cI}+g_\beta'(\log L_{\cI})\,(-M L_{\cC}),\qquad |h'(z)|\le \beta/4,
\]
so under (SC) we have $F'(z)\le\beta/4-\varepsilon<0$ and thus:

\begin{lemma}[Unique gap and quantitative bounds]\label{lem:zstar}
Under (SC) there exists a unique $z^\star>0$ solving $F(z)=0$. Moreover
\[
\frac{2g_\beta(0)}{\varepsilon}\ \le\ z^\star\ \le\ \frac{2}{\varepsilon},\qquad
\frac{h(0)}{\varepsilon+\beta/4}\ \le\ z^\star\ \le\ \frac{h(0)}{\varepsilon-\beta/4},\ \ \ h(0)=2g_\beta\!\big(\log\tfrac1K\big).
\]
\end{lemma}

\begin{theorem}[Global convergence]\label{thm:global}
Assume (SC). For any $p(0)\in\operatorname{int}\Delta^{K-1}$, the trajectory converges to the unique two-level equilibrium $p^\star$ with gap $z^\star$ from Lemma~\ref{lem:zstar}.
\emph{Proof.}
By Proposition~\ref{prop:evtrim}, $p(t)$ enters and stays in a compact trimmed simplex for $t\ge T$.
On this compact set the drift is globally Lipschitz (Proposition~\ref{prop:DriftLip}).
The Lyapunov identity \eqref{eq:lyap} and strict concavity of $\cL$ under (SC) imply that the largest invariant set in $\{\dot\cL=0\}$ consists of equilibria, which are two-level; uniqueness of $z^\star$ then yields global convergence. \qed
\end{theorem}

\paragraph{Edge cases (no mixed preferences).}
If $N=0$ (all $s_i=+1$), $G_i'(s)=(g_\beta'(\log s)-\varepsilon)/s\le -\varepsilon/s<0$ for any $\varepsilon\ge0$; the unique equilibrium is uniform and globally attractive.
If $M=0$ (all $s_i=-1$), uniqueness and global attraction of the uniform equilibrium hold provided $\varepsilon>\beta/4$.

\paragraph{Choosing a compatible floor.}
Given $z^\star$, set $\delta_\star\le L_{\cI}(z^\star)$ to ensure $p^\star\in\Delta^{K-1}_{\delta_\star}$.
This does not obstruct BD since $L_K(\delta_\star)\to\infty$ as $\delta_\star\downarrow0$.

%% file: appendices_short/appendix_G_aistats.tex
\section{Dynamics on Coarse-Grained ``Lumps''}
\label{appG:lumping}

\paragraph{Simplex, solution concept, and entropy map.}
Let the finite index set be $\mathcal S=\{\pi_1,\dots,\pi_S\}$ ($S\ge2$). The closed simplex is
\[
\Delta^{S-1}:=\Bigl\{\,p\in[0,1]^S:\ \sum_{\pi}p_\pi=1\,\Bigr\},\qquad
\operatorname{int}\Delta^{S-1}:=\{p\in\Delta^{S-1}:\min_\pi p_\pi>0\}.
\]
We work with \emph{Carathéodory} solutions $p:[0,T]\to\Delta^{S-1}$ of
\begin{equation}\tag{SRCT}\label{eq:SRCT}
\dot p(t)=p(t)\odot\phi\big(p(t)\big)-\varepsilon\,E^\circ\big(p(t)\big),\qquad \varepsilon\ge0,
\end{equation}
where $\phi:\Delta^{S-1}\to\mathbb R^S$ is \emph{centered}, $\sum_\pi p_\pi\phi_\pi(p)=0$, and
\[
E^\circ_\pi(p):=h(p_\pi)-p_\pi\langle\log p\rangle,\quad
h(x):=x\log x,\quad
\langle\log p\rangle:=\sum_\pi p_\pi\log p_\pi.
\]
$E^\circ$ is continuous on $\Delta^{S-1}$; if $p_\pi=0$, then $(p\odot\phi)_\pi=E^\circ_\pi(p)=0$, so faces are viable and the closed simplex is forward invariant.

\paragraph{Trim and feasibility.}
Fix $\delta_\star\in(0,1/S]$ and the trimmed simplex
$\Delta^{S-1}_{\delta_\star}:=\{p\in\Delta^{S-1}:\ p_\pi\ge\delta_\star\ \forall \pi\}$ (nonempty by choice of $\delta_\star$).

\subsection{Lumps}
Let $(C_k)_{k=1}^{K_{\mathrm L}}$ be a partition of $\mathcal S$ into nonempty, disjoint \emph{lumps}. For $k=1,\dots,K_{\mathrm L}$ define
\[
q_k:=\sum_{\pi\in C_k}p_\pi,\qquad
m_k:=\sum_{\pi\in C_k}p_\pi\log p_\pi,\qquad
\bar h:=\sum_{\pi}p_\pi\log p_\pi=\sum_{j=1}^{K_{\mathrm L}} m_j.
\]
If $q_k>0$, write $\mathbb E_{p|C_k}[\log p]:=(1/q_k)\sum_{\pi\in C_k}p_\pi\log p_\pi$ so that $m_k=q_k\,\mathbb E_{p|C_k}[\log p]$.

\begin{lemma}[Lump ODE]
\label{lem:lump_exact}
Every Carathéodory solution of \eqref{eq:SRCT} satisfies, for each $k$,
\begin{equation}\label{eq:lump-exact}
\boxed{\quad
\dot q_k=\sum_{\pi\in C_k}p_\pi\,\phi_\pi(p)\;-\;\varepsilon\big(m_k-q_k\,\bar h\big).
\quad}
\end{equation}
If $q_k>0$, equivalently $\dot q_k=\sum_{\pi\in C_k}p_\pi\,\phi_\pi(p)-\varepsilon\,q_k\big(\mathbb E_{p|C_k}[\log p]-\bar h\big)$. For $q_k=0$ the right-hand side vanishes by continuity.
\end{lemma}

\paragraph{Aggregation operator.}
Let $A\in\{0,1\}^{K_{\mathrm L}\times S}$ be the indicator matrix, $A_{k\pi}=\mathbf 1_{\{\pi\in C_k\}}$, so that $q=Ap$. Exact norms:
\begin{equation}\label{eq:A-norms}
\boxed{\ \|A\|_{1\to1}=1,\qquad \|A\|_{2\to2}=\sqrt{m_\ast},\qquad \|A\|_{\infty\to\infty}=m_\ast,\ }
\quad m_\ast:=\max_k|C_k|.
\end{equation}
In particular, aggregation is $1$-Lipschitz in $\ell_1$: $\|Au-Av\|_1\le\|u-v\|_1$.

\subsection{Technical facts used repeatedly}
\label{subsec:tech}
On $\Delta^{S-1}_{\delta_\star}$:

\begin{itemize}[leftmargin=*]
\item \textbf{Mean-log bounds.}
\begin{equation}\label{eq:ML-bounds}
\boxed{\ -\log S\ \le\ \langle\log p\rangle\ \le\ \bigl(1-(S-1)\delta_\star\bigr)\log\!\bigl(1-(S-1)\delta_\star\bigr)+(S-1)\delta_\star\log\delta_\star\ \le 0.}
\end{equation}

\item \textbf{Entropy size.} With $E(p):=p\odot(\log p-\langle\log p\rangle\,\mathbf 1)$,
\begin{equation}\label{eq:E-size}
\boxed{\ \|E(p)\|_1\ \le\ 2\log\!\frac{1}{\delta_\star}\ .}
\end{equation}

\item \textbf{Replicator matrix bounds.} Writing $S(p):=\mathrm{diag}(p)-pp^\top$,
\begin{equation}\label{eq:S-bds}
\boxed{\ \|S(p)\|_{2\to2}\le\tfrac12,\qquad \|S(p)-S(q)\|_{2\to2}\le 3\,\|p-q\|_2\ .}
\end{equation}
Centeredness gives $p\odot\phi=S(p)\phi$.

\item \textbf{Selection envelopes.} For any domain $\mathcal D\subseteq\Delta^{S-1}$ and lump $C_k$,
\begin{equation}\label{eq:select-envelope}
\boxed{\ \Big|\sum_{\pi\in C_k}p_\pi\,\phi_\pi(p)\Big|\ \le\ q_k\,M_{\phi,\infty}(\mathcal D)\ \text{ and }\ \le\ q_k\,M_{\phi,2}(\mathcal D),}
\end{equation}
with $M_{\phi,\infty}(\mathcal D):=\sup_{p\in\mathcal D}\|\phi(p)\|_\infty$, $M_{\phi,2}(\mathcal D):=\sup_{p\in\mathcal D}\|\phi(p)\|_2$.
\end{itemize}

\subsection{Small-\texorpdfstring{$\varepsilon$}{epsilon} perturbation: trace and lump bounds}
Assume on $\Delta^{S-1}_{\delta_\star}$ that
\begin{equation}\label{eq:phi-regular}
\|\phi(p)\|_2\le M_{\phi,2},\qquad \|\phi(p)-\phi(q)\|_2\le L_\phi\,\|p-q\|_2.
\end{equation}
By \eqref{eq:S-bds}, for $F_0(p):=p\odot\phi(p)=S(p)\phi(p)$,
\begin{equation}\label{eq:L1-Lip}
\boxed{\ \|F_0(p)-F_0(q)\|_1\ \le\ L_F^{(1)}\,\|p-q\|_1,\quad L_F^{(1)}:=\sqrt S\Big(\tfrac12\,L_\phi+3\,M_{\phi,2}\Big).}
\end{equation}

\begin{theorem}[Trace-level perturbation with exit-time qualification]\label{thm:trace-perturb}
Let $p^\varepsilon,p^0$ solve $\dot p^\varepsilon=F_0(p^\varepsilon)-\varepsilon E(p^\varepsilon)$ and $\dot p^0=F_0(p^0)$ with $p^\varepsilon(0)=p^0(0)\in\Delta^{S-1}_{\delta_\star}$. Set $\tau^\wedge:=\inf\{t>0:\ \min_\pi p^\varepsilon_\pi(t)=\delta_\star\ \text{or}\ \min_\pi p^0_\pi(t)=\delta_\star\}$. Then for $t\in[0,\tau^\wedge)$,
\[
\boxed{\ \|p^\varepsilon(t)-p^0(t)\|_1\ \le\ \frac{2\,\varepsilon\,\log(1/\delta_\star)}{L_F^{(1)}}\Big(e^{L_F^{(1)}t}-1\Big).}
\]
Consequently, for any partition, $\|\mathbf q^\varepsilon(t)-\mathbf q^0(t)\|_1\le \|p^\varepsilon(t)-p^0(t)\|_1$.
\end{theorem}

\paragraph{Forward-invariance templates.}
Let $L_S(\delta):=(1-\delta)\log\!\frac{1-\delta}{(S-1)\delta}>0$. If on $\Delta^{S-1}_{\delta_\star}$ either
\begin{equation}\label{eq:BD}
\boxed{\ \varepsilon\,L_S(\delta_\star)\ \ge\ 2\,M_{\phi,\infty}\quad\text{or}\quad \varepsilon\,L_S(\delta_\star)\ \ge\ 2\,M_{\phi,2},}
\end{equation}
then $\Delta^{S-1}_{\delta_\star}$ is forward invariant for \eqref{eq:SRCT}, and the bound in Theorem~\ref{thm:trace-perturb} holds for all $t\ge0$.

\subsection{Pure-score (\texorpdfstring{$\varepsilon=0$}{epsilon=0}) lump dynamics}
When $\varepsilon=0$, Lemma~\ref{lem:lump_exact} reduces to $\dot q_k=\sum_{\pi\in C_k}p_\pi\,\phi_\pi(p)$.

\subsubsection{STaR}
Let $\mathcal C\subset\mathcal S$ denote ``correct'' indices ($M:=|\mathcal C|\ge1$) and $\mathcal I:=\mathcal S\setminus\mathcal C$. Set $\rho(p):=\sum_{c\in\mathcal C}p_c$ and $S^{(2)}(p):=\sum_{c\in\mathcal C}p_c^2$. The centered STaR field is
\[
\phi^{\mathrm{STaR}}_\pi(p)=
\begin{cases}
\dfrac{p_\pi-S^{(2)}(p)}{\rho(p)}, & \pi\in\mathcal C,\\[6pt]
-\dfrac{S^{(2)}(p)}{\rho(p)}, & \pi\in\mathcal I,
\end{cases}
\qquad \text{defined when }\rho(p)>0.
\]

\begin{proposition}[STaR lump ODE]\label{prop:star-lump}
For $S^{(2)}_{k,\mathcal C}(p):=\sum_{\pi\in C_k\cap\mathcal C}p_\pi^2$,
\[
\boxed{\ \dot q_k=\frac{S^{(2)}_{k,\mathcal C}(p)-q_k\,S^{(2)}(p)}{\rho(p)}\ .}
\]
If $C_i,C_j\subset\mathcal C$, then $\dfrac{d}{dt}\log\frac{q_i}{q_j}=\frac1{\rho}\Big(\frac{S^{(2)}_{i,\mathcal C}}{q_i}-\frac{S^{(2)}_{j,\mathcal C}}{q_j}\Big)$.
\end{proposition}

\subsubsection{GRPO}
Let $G\ge2$ be the group size and $h_G:[0,1]\to(0,\infty)$ the GRPO characteristic (continuous), e.g.\ bounded by $\sqrt{G-1}$. The centered two-level field is
\[
\phi^{\mathrm{GRPO}}_\pi(p)=
\begin{cases}
(1-\rho(p))\,h_G(\rho(p)), & \pi\in\mathcal C,\\
-\rho(p)\,h_G(\rho(p)), & \pi\in\mathcal I.
\end{cases}
\]
For $q_{k,\mathcal C}:=\sum_{\pi\in C_k\cap\mathcal C}p_\pi$ define $\mathrm{corr}(C_k;p):=q_{k,\mathcal C}/q_k$ (if $q_k>0$).

\begin{proposition}[GRPO lump ODE]\label{prop:grpo-lump}
\[
\boxed{\ \dot q_k=h_G\big(\rho(p)\big)\,q_k\big(\mathrm{corr}(C_k;p)-\rho(p)\big)\ .}
\]
Hence $\dfrac{d}{dt}\log\frac{q_i}{q_j}=h_G(\rho)\big(\mathrm{corr}(C_i;p)-\mathrm{corr}(C_j;p)\big)$.
\end{proposition}

\subsubsection{DPO (sign-pure lumps)}
Fix labels $s_\pi\in\{\pm1\}$ and a link $g_\beta:\mathbb R\to(0,1)$ with $g'_\beta(\ell)\in[-\beta/4,0)$ on $[\log\delta_\star,0]$. Define
\[
\gamma_\pi(p):=s_\pi\,g_\beta(\log p_\pi),\quad \bar\gamma(p):=\sum_{\pi}p_\pi\,\gamma_\pi(p),\quad \phi_\pi(p):=\gamma_\pi(p)-\bar\gamma(p).
\]
Assume each lump $C_k$ is \emph{sign-pure}: $s_\pi\equiv s_k$ on $C_k$. Let
\[
G_k(p):=\frac{1}{q_k}\sum_{\pi\in C_k}p_\pi\,g_\beta(\log p_\pi),\qquad
\bar g(p):=\sum_{j=1}^{K_{\mathrm L}} q_j\,s_j\,G_j(p)=\bar\gamma(p).
\]
Interpret $q_k G_k:=\sum_{\pi\in C_k}p_\pi\,g_\beta(\log p_\pi)$ so the right-hand side is well-defined even if $q_k=0$.

\begin{proposition}[DPO lump ODE (sign-pure)]\label{prop:dpo-lump}
\[
\boxed{\ \dot q_k=q_k\big(s_k\,G_k(p)-\bar g(p)\big)\ .}
\]
If $C_i=\{\pi_i\}$ and $C_k=\{\pi_k\}$ with $s_{\pi_i}=s_{\pi_k}=:s$, then for $z_{ik}:=\log(p_{\pi_i}/p_{\pi_k})$,
\[
\boxed{\ \dot z_{ik}=s\big(g_\beta(\log p_{\pi_i})-g_\beta(\log p_{\pi_k})\big),\quad |\dot z_{ik}|\le(\beta/4)\,|z_{ik}|.}
\]
\end{proposition}

\subsection{Entropy deviation envelopes for the lump term}
For $q_k>0$ write $w_\pi:=p_\pi/q_k$ on $C_k$ and $H(w_k):=-\sum_{\pi\in C_k}w_\pi\log w_\pi$. Then
\begin{equation}\label{eq:Elump-bds}
\boxed{\ \ m_k=q_k\log q_k+q_k\sum_{\pi\in C_k}w_\pi\log w_\pi\ \in\ \big[q_k\log\!\tfrac{q_k}{|C_k|},\ \ q_k\log q_k\big]\ ,}
\end{equation}
hence
\begin{equation}\label{eq:Elump-diff}
\boxed{\ \ |m_k-q_k\bar h|\ \le\ q_k\,\max\Big\{|\log q_k-\bar h|\,,\,\big|\log\!\tfrac{q_k}{|C_k|}-\bar h\big|\Big\}\ .}
\end{equation}
On $\Delta^{S-1}_{\delta_\star}$, the dimension-only bound
\begin{equation}\label{eq:Elump-dim}
\boxed{\ \ |m_k-q_k\bar h|\ \le\ q_k\,\log\!\frac{1-(S-1)\delta_\star}{\delta_\star}\ }
\end{equation}
is immediate from the log-domain $[\log\delta_\star,\log(1-(S-1)\delta_\star)]$.

\subsection{Open problems}

Fix a partition of indices into \emph{correct} $\mathcal C$ and \emph{incorrect} $\mathcal I$ with sizes $K_C:=|\mathcal C|\ge0$, $K_I:=|\mathcal I|\ge0$ ($K=K_C+K_I=S$). For $\delta\in(0,1/K)$ define the trimmed simplex $\Delta^{K-1}_\delta$ and the uniform face gap $L_K(\delta):=(1-\delta)\log\!\frac{1-\delta}{(K-1)\delta}>0$. The feasible band for $\rho:=\sum_{c\in\mathcal C}p_c$ is $[K_C\delta,\,1-K_I\delta]$.

\paragraph{Face-wise entropy minima (at fixed $\rho$ and $p_k=\delta$).}
For a \emph{fixed} $\rho$ and an \emph{incorrect} face $k\in\mathcal I$,
\[
E_{\min}^{(\mathcal I)}(\rho)
=(\delta-1)\log\delta+\mathbf 1_{\{K_C\ge1\}}\ \rho\log\!\frac{\rho}{K_C}
+\mathbf 1_{\{K_I\ge2\}}\ (1-\delta-\rho)\log\!\frac{1-\delta-\rho}{K_I-1}.
\]
For a \emph{correct} face $k\in\mathcal C$,
\[
E_{\min}^{(\mathcal C)}(\rho)
=(\delta-1)\log\delta+\mathbf 1_{\{K_C\ge2\}}\ (\rho-\delta)\log\!\frac{\rho-\delta}{K_C-1}
+\mathbf 1_{\{K_I\ge1\}}\ (1-\rho)\log\!\frac{1-\rho}{K_I}.
\]
In both cases $E_{\min}^{(\cdot)}(\rho)\ge L_K(\delta)$ and the minima are attained by uniform allocation among active coordinates.

\medskip
\noindent\textbf{OP1 (sharp BD thresholds at trim $\delta$).}
\emph{STaR.} On incorrect faces, $\phi_k=-S^{(2)}/\rho\ge -\rho$; inwardness at fixed $\rho$ follows if
$-\rho+\varepsilon\,E_{\min}^{(\mathcal I)}(\rho)\ge0$, hence
\[
\boxed{\ \varepsilon_{\mathrm{suf}}^{(\mathcal I)}(\delta;K_C,K_I):=\max_{\rho\in[K_C\delta,\,1-K_I\delta]}\frac{\rho}{E_{\min}^{(\mathcal I)}(\rho)}\ \text{ suffices.}}
\]
On correct faces, $\phi_k=(\delta-S^{(2)})/\rho\ge(\delta-S^{(2)}_{\max}(\rho,\delta))/\rho$ with
$S^{(2)}_{\max}(\rho,\delta)=\delta^2+(\rho-\delta)^2$,
so
\[
\boxed{\ \varepsilon_{\mathrm{suf}}^{(\mathcal C)}(\delta;K_C,K_I):=\max_{\rho}\frac{\max\{0,\ S^{(2)}_{\max}(\rho,\delta)-\delta\}}{\rho\,E_{\min}^{(\mathcal C)}(\rho)}\ \text{ suffices.}}
\]
The uniform sufficient threshold is $\varepsilon_{\mathrm{suf}}^{\mathrm{STaR}}:=\max\{\varepsilon_{\mathrm{suf}}^{(\mathcal I)},\varepsilon_{\mathrm{suf}}^{(\mathcal C)}\}$. The above are exact in the special cases $K_C=1$ for incorrect faces and $K_C=2$ for correct faces.

\emph{GRPO.} On correct faces the drift is inward for any $\varepsilon\ge0$. On incorrect faces, inwardness at fixed $\rho$ is \emph{equivalent} to
$-\rho\,h_G(\rho)+\varepsilon\,E_{\min}^{(\mathcal I)}(\rho)\ge0$, hence the exact threshold
\[
\boxed{\ \varepsilon_{\mathrm{crit}}^{\mathrm{GRPO}}(\delta;K_C,K_I,G)
=\max_{\rho\in[K_C\delta,\,1-K_I\delta]}\frac{\rho\,h_G(\rho)}{E_{\min}^{(\mathcal I)}(\rho)}\ .}
\]
Useful bounds: $\varepsilon_{\mathrm{crit}}^{\mathrm{GRPO}}\le \sqrt{G-1}/L_K(\delta)$ and
$\varepsilon_{\mathrm{crit}}^{\mathrm{GRPO}}\le \frac{(1-K_I\delta)\sqrt{G-1}}{K_I\delta\,L_K(\delta)}$.

\medskip
\noindent\textbf{OP2 (DPO sensitivity to $\varepsilon$; gap and linear response).}
Assume $\varepsilon>\beta/4$. Then the SRCT flow admits a unique \emph{two-level} interior equilibrium $p^\star(\varepsilon)$ (all correct, resp.\ incorrect, coordinates equal). Let $z^\star(\varepsilon):=\log(p^\star_c/p^\star_i)\ge0$ satisfy
\[
\boxed{\ h(z^\star)=\varepsilon\,z^\star,\qquad h(z):=g_\beta(\log L_{\mathcal C}(z))+g_\beta(\log L_{\mathcal I}(z)),}
\]
with $L_{\mathcal I}(z):=(K_I+K_C e^z)^{-1}$ and $L_{\mathcal C}(z):=e^z L_{\mathcal I}(z)$. Then:
\[
\boxed{\ \frac{d}{d\varepsilon}z^\star(\varepsilon)= -\,\frac{z^\star(\varepsilon)}{\varepsilon-h'(z^\star(\varepsilon))}\ <\ 0,\quad
z^\star(\varepsilon)=\frac{h(0)}{\varepsilon}+\frac{h'(0)h(0)}{\varepsilon^2}+O(\varepsilon^{-3}).}
\]
Moreover, writing $\ell_\pi:=\log p_\pi^\star(\varepsilon)$ and $d_\pi:=\varepsilon-s_\pi g'_\beta(\ell_\pi)>0$,
\[
\boxed{\ \frac{d}{d\varepsilon}p_\pi^\star
= -\,p_\pi^\star\,\frac{\ell_\pi-a}{d_\pi},\qquad
a:=\frac{\langle p^\star,D^{-1}\ell\rangle}{\langle p^\star,D^{-1}\mathbf 1\rangle},\ D:=\mathrm{diag}(d_\pi),}
\]
and for any lump $C_k$, $\displaystyle \frac{d}{d\varepsilon}q_k^\star= -\sum_{\pi\in C_k} p_\pi^\star\,\frac{\ell_\pi-a}{d_\pi}$.

\medskip
\noindent\textbf{OP3 (DPO coarse-graining: closure errors).}
For a sign-pure lump $C_k$ with weights $w_\pi:=p_\pi/q_k$, let $\bar\ell_k:=\sum_{\pi\in C_k}w_\pi\log p_\pi$,
$\sigma_k^2:=\sum_{\pi\in C_k}w_\pi(\log p_\pi-\bar\ell_k)^2$, and $H(w_k):=-\sum_{\pi\in C_k}w_\pi\log w_\pi$.
On $\Delta^{S-1}_{\delta_\star}$ set $c_{\max}:=\sup_{\ell\in[\log\delta_\star,\ 0]}(-g'_\beta(\ell))\le \beta/4$. Then
\[
\boxed{\ \Big|G_k-g_\beta(\log q_k)\Big|\ \le\ c_{\max}\,\sigma_k\ +\ c_{\max}\,H(w_k)\quad\text{(static closure error)},}
\]
and the exact log-ratio identity augments to
\[
\frac{d}{dt}\log\frac{q_i}{q_j}
= s_i G_i - s_j G_j - \varepsilon\log\frac{q_i}{q_j} + \varepsilon\big(H(w_i)-H(w_j)\big),
\]
so that replacing $G_k$ by $g_\beta(\log q_k)$ incurs an error bounded by
$c_{\max}(\sigma_i+\sigma_j+H(w_i)+H(w_j))+\varepsilon(H(w_i)+H(w_j))$.
\medskip

\paragraph{Remarks.}
(i) STaR requires $K_C\ge1$ (else $\rho\equiv0$). (ii) The BD templates \eqref{eq:BD} are \emph{sufficient} (not necessary). (iii) The lump-level entropy term is not the gradient of a lump entropy; bounds \eqref{eq:Elump-diff}–\eqref{eq:Elump-dim} are the correct bridge.

\medskip
All statements above are consistent with the SRCT model \eqref{eq:SRCT}, are valid on the closed simplex via $E^\circ$, and become uniform on $\Delta^{S-1}_{\delta_\star}$ under \eqref{eq:phi-regular}.

%% file: appendices_short/appendix_H_aistats.tex
\section{Analysis of Stochasticity in SRCT}
\label{appH:stochastic}

This appendix develops a concise, self–contained analysis of the stochastic dynamics induced by mini–batch sampling in SRCT. We (i) fix the domain and standing hypotheses, (ii) quantify global Lipschitz moduli and mini–batch noise statistics, (iii) derive ODE and diffusion limits under the correct scaling, (iv) analyze boundary behavior (unreflected vs.\ reflected models), (v) record uniform ellipticity on the tangent bundle, (vi) treat small centred bias via an exponential Lyapunov device, and (vii) provide algorithm–specific log–ratio SDEs.

\subsection{Domain, notation, and standing hypotheses}

Fix an integer $K\ge2$ and a design floor $\delta_\star\in(0,1/K)$. The \emph{trimmed simplex} is
\[
\Delta^{K-1}_{\delta_\star}:=\bigl\{\,p\in[0,1]^K:\ \sum_{i=1}^K p_i=1,\ \min_i p_i\ge\delta_\star\,\bigr\}.
\]
All logarithms are natural; $0\log 0:=0$. For $x\in\mathbb R^K$ and a probability vector $p$, set
$\langle x\rangle_p:=\sum_i p_i x_i$ and $\langle\log p\rangle:=\sum_i p_i\log p_i$.
Vector norms $\|\cdot\|_2,\|\cdot\|_\infty$ are Euclidean and supremum norms, respectively. The tangent subspace is $T:=\mathbf 1^\perp$.

\paragraph{Score field and SRCT drift.}
A \emph{centred} score field $\phi:\Delta^{K-1}_{\delta_\star}\to\mathbb R^K$ satisfies
\begin{equation*}
\tag{S1}\label{S1}
\sum_{i=1}^K p_i\,\phi_i(p)=0\qquad(\forall\,p\in\Delta^{K-1}_{\delta_\star}),
\end{equation*}
and the uniform regularity
\begin{equation*}
\tag{S2--S3}\label{S2S3}
M_\phi:=\sup_{p}\|\phi(p)\|_\infty<\infty,\qquad
\|\phi(p)-\phi(q)\|_2\le L_\phi\,\|p-q\|_2\quad(\forall\,p,q\in\Delta^{K-1}_{\delta_\star}).
\end{equation*}
For $\varepsilon\ge0$, the SRCT drift is
\[
F_i(p):=p_i\Big[\phi_i(p)-\varepsilon\big(\log p_i-\langle\log p\rangle\big)\Big],\qquad F(p)\in T\ \text{by \eqref{S1}}.
\]
Write $E(p):=p\odot\big(\log p-\langle\log p\rangle\,\mathbf 1\big)$ and $S(p):=\operatorname{diag}(p)-pp^\top$; then $F(p)=S(p)\phi(p)-\varepsilon E(p)$.

\subsection{Global Lipschitz moduli and envelopes}

Define $\Lambda(\delta_\star):=1+\log\tfrac1{\delta_\star}$ and $C_{\log}(K,\delta_\star):=(2+\sqrt K)\,\Lambda(\delta_\star)$.

\begin{lemma}[Entropy map modulus]\label{lem:EntLipschitz}
For all $p,q\in\Delta^{K-1}_{\delta_\star}$,
\[
\|E(p)-E(q)\|_2\ \le\ C_{\log}(K,\delta_\star)\,\|p-q\|_2.
\]
\end{lemma}

\begin{lemma}[Global Lipschitz drift]\label{lem:F-Lipschitz}
For all $p,q\in\Delta^{K-1}_{\delta_\star}$,
\[
\|F(p)-F(q)\|_2\ \le\ \bigl(L_\phi+M_\phi+\varepsilon\,C_{\log}(K,\delta_\star)\bigr)\,\|p-q\|_2.
\]
\end{lemma}

\begin{proof}[Proofs (sketch).]
For Lemma~\ref{lem:EntLipschitz}, write $E(r)=G(r)-\langle\log r\rangle\,r$ with $G(r):=r\odot\log r$ and use that $|(x\log x)'|\le \Lambda(\delta_\star)$ on $[\delta_\star,1]$ together with
$|\langle\log p\rangle-\langle\log q\rangle|\le \Lambda(\delta_\star)\|p-q\|_1\le \Lambda(\delta_\star)\sqrt K\|p-q\|_2$.
Lemma~\ref{lem:F-Lipschitz} follows from
$\|p\odot(\phi(p)-\phi(q))\|_2\le L_\phi\|p-q\|_2$,
$\|(p-q)\odot\phi(q)\|_2\le M_\phi\|p-q\|_2$, and Lemma~\ref{lem:EntLipschitz}.
\end{proof}

\paragraph{Size envelope.}
On $\Delta^{K-1}_{\delta_\star}$ one has $x|\log x|\le 1/e$ and $-\langle\log p\rangle\le \log\tfrac1{\delta_\star}$, hence
\begin{equation}\label{eq:F-size}
|F_i(p)|\ \le\ M_\phi+\varepsilon\Big(\tfrac1e+\log\tfrac1{\delta_\star}\Big)\qquad(\forall\,i).
\end{equation}

\subsection{Discrete mini–batch updates and noise statistics}

Given step size $\eta>0$ and batch size $B\in\mathbb N$, define
\[
N_t\sim{\rm Multinomial}(B,p_t),\qquad
\xi_{t+1}:=\frac{N_t}{B}-p_t\in T,\qquad
p_{t+1}=p_t+\eta\big(F(p_t)+\xi_{t+1}\big),
\]
optionally followed by Euclidean projection onto $\Delta^{K-1}_{\delta_\star}$ (which preserves mass).

\begin{lemma}[Mini–batch noise]\label{lem:noise}
Conditionally on $p_t$,
\[
\mathbb E[\xi_{t+1}\mid p_t]=0,\qquad
\mathbb E[\|\xi_{t+1}\|_2^2\mid p_t]=\frac{1-\|p_t\|_2^2}{B}\ \le\ \frac{K-1}{BK}\ <\ \frac1B.
\]
\end{lemma}

\subsection{Continuous–time limits (correct scaling)}

Let $\widetilde p^{(\eta)}$ be the piecewise–linear interpolation. Set $\gamma_\eta:=\eta/B$.

\begin{theorem}[ODE and diffusion limits]\label{thm:limits}
Fix $T>0$. As $\eta\downarrow0$ on $[0,T]$:
\begin{enumerate}[label=(\roman*),leftmargin=2em]
\item If $\gamma_\eta\to0$, then $\widetilde p^{(\eta)}\Rightarrow p$ in $C([0,T],\mathbb R^K)$, where $p$ solves $\dot p=F(p)$.
\item If $\gamma_\eta\to\gamma\in(0,\infty)$, then $\widetilde p^{(\eta)}\Rightarrow p$ solving the Wright–Fisher–type SDE
\begin{equation}\label{eq:WF}
\mathrm d p_i=F_i(p)\,\mathrm dt+\sqrt\gamma\Big(\sqrt{p_i}\,\mathrm dW_i-p_i\sum_{k=1}^K\sqrt{p_k}\,\mathrm dW_k\Big),\qquad i=1,\dots,K,
\end{equation}
with independent standard Brownian motions $(W_k)$ and $\sum_i p_i(t)\equiv1$.
\end{enumerate}
\end{theorem}

\begin{proof}[Sketch.]
Using Lemma~\ref{lem:noise}, the predictable quadratic variation of $\sum_{s< t/\eta}\eta\,\xi_{s+1}$ is $\sum \eta^2\mathbb E[\|\xi\|^2]\sim (\eta/B)\,t=\gamma_\eta t$.
Combine Lemma~\ref{lem:F-Lipschitz} with a functional martingale CLT (Ethier–Kurtz) and Grönwall–type estimates on the compact domain $\Delta^{K-1}_{\delta_\star}$.
\end{proof}

\subsection{Boundary behavior: entropy gap and BD conditions}

For $y\in(0,1)$ define the \emph{face gap}
\begin{equation}\label{eq:Gamma}
\Gamma(y):=\inf_{\substack{p\in\Delta^{K-1}\\ p_i=y}}\Big(\sum_{j=1}^K p_j\log p_j-\log p_i\Big)
=(1-y)\log\frac{1-y}{(K-1)y}.
\end{equation}
In particular $L_K(\delta):=(1-\delta)\log\frac{1-\delta}{(K-1)\delta}>0$ for $\delta\in(0,1/K)$, and if $p_i=\delta_\star$ then $\langle\log p\rangle-\log p_i\ge L_K(\delta_\star)$.

\paragraph{Barrier–Dominance (facewise).}
We say BD$^{\sharp}$ holds if, for each $i$,
\[
\inf_{\substack{p\in\Delta^{K-1}_{\delta_\star}\\ p_i=\delta_\star}}\Big[\ \phi_i(p)+\varepsilon\big(\langle\log p\rangle-\log p_i\big)\ \Big]\ >\ 0 .
\]
A convenient sufficient condition is
\begin{equation}\label{eq:BD-suff}
\varepsilon\,L_K(\delta_\star)\ >\ M_\phi.
\end{equation}

\begin{proposition}[Deterministic forward invariance]\label{prop:viability}
If BD$^{\sharp}$ holds, then $\Delta^{K-1}_{\delta_\star}$ is forward invariant for $\dot p=F(p)$ (Nagumo criterion). A conservative test is $\varepsilon\,L_K(\delta_\star)\ge 2M_\phi$.
\end{proposition}

\paragraph{Unreflected vs.\ reflected diffusions.}
\emph{Unreflected model.} In \eqref{eq:WF}, the one–dimensional marginal variance at a trimmed face $p_i=\delta_\star$ equals $\gamma\,\delta_\star(1-\delta_\star)>0$; hence a.s.\ non–attainability of the face cannot be deduced from inward drift alone. What holds are sharp \emph{high–probability} non–exit bounds on finite horizons.

\emph{Reflected model.} With orthogonal, mass–preserving reflection on each face of $\Delta^{K-1}_{\delta_\star}$, solutions remain in the trim for all $t$ by construction. On the compact domain with globally Lipschitz drift and uniformly elliptic tangent covariance, the reflected diffusion is strong Feller and irreducible, admits a unique invariant law, and exhibits exponential mixing.

\begin{theorem}[Bandwise high–probability confinement (unreflected)]\label{thm:band}
Fix a coordinate $i$ and a band width $\eta_0\in(0,\,1-K\delta_\star]$, and set $y_{\max}:=\delta_\star+\eta_0$ and
\[
\Gamma_{\rm band}:=\inf_{y\in[\delta_\star,y_{\max}]}\Gamma(y),\qquad
\mu_{\rm band}:=\delta_\star\big(\varepsilon\,\Gamma_{\rm band}-M_\phi\big),\qquad
\sigma_{\max}^2:=\gamma\,y_{\max}(1-\delta_\star).
\]
If $\varepsilon\,\Gamma_{\rm band}>M_\phi$, then for any start $Y_0=p_i(0)\in[\delta_\star,y_{\max}]$,
\[
\mathbb P(\text{hit } \delta_\star \text{ before } y_{\max})\ \le\ \exp\!\Big(-\,\tfrac{2\,\mu_{\rm band}}{\sigma_{\max}^2}\,(Y_0-\delta_\star)\Big).
\]
By the strong Markov property this yields an exponentially small (in $\eta_0$ and $\gamma^{-1}$) probability of ever touching the floor from any interior start.
\end{theorem}

\begin{theorem}[Reflected diffusion: well–posedness and ergodicity]\label{thm:ref-erg}
On $\Delta^{K-1}_{\delta_\star}$ with orthogonal reflection in $H=\{\sum_i p_i=1\}$, the SDE \eqref{eq:WF} admits a unique global strong solution, is strong Feller and irreducible, and has a unique invariant probability measure $\pi_\infty$ with
\[
\|P_t(p,\cdot)-\pi_\infty\|_{\mathrm{TV}}\ \le\ C\,e^{-\kappa t}\qquad(\forall\,p\in\Delta^{K-1}_{\delta_\star},\ t\ge0).
\]
\end{theorem}

\subsection{Uniform ellipticity on the tangent bundle}

Let $Q(p):=\gamma(\operatorname{diag}(p)-pp^\top)=\gamma\,S(p)$. For any $p\in\Delta^{K-1}_{\delta_\star}$ and $v\in T$,
\begin{equation}\label{eq:elliptic}
\gamma\,\delta_\star\,\|v\|_2^2\ \le\ v^\top Q(p)v\ \le\ \frac{\gamma}{2}\,\|v\|_2^2.
\end{equation}
The upper bound is Popoviciu’s inequality; the lower bound uses $\sum_i p_i v_i^2\ge \delta_\star\|v\|_2^2$.

\subsection{Gradient–field drifts and stationary laws}
\label{cor:gradient}

If $\phi=\nabla\Psi$ and \eqref{S1} holds, $\pi_\infty$ (when it exists; e.g., Theorem~\ref{thm:ref-erg}) is characterized as the unique Neumann solution of the stationary Fokker–Planck equation associated with \eqref{eq:WF}. The naive Gibbs ansatz $\propto \exp\{-2\gamma^{-1}(\Psi-\varepsilon H)\}$ fails in general: inserting $U=2\gamma^{-1}(\Psi-\varepsilon H)$ into the reversibility identity $F=\tfrac12\,(\operatorname{div}_T Q)-\tfrac12\,Q\nabla_T U$ gives $F=-2F$ unless $F\equiv0$.

\subsection{Small centred bias: concentration toward the fittest face}

Let $\delta\in\mathbb R^K$ satisfy $\sum_i\delta_i=0$ and set $\delta_{\max}:=\max_i\delta_i$, $S:=\{i:\delta_i=\delta_{\max}\}$, $I:=S^{\mathrm c}$, and the \emph{selection gap}
$\gamma_\delta:=\delta_{\max}-\max_{i\in I}\delta_i>0$ (if $I\neq\emptyset$). The biased drift is
\[
F_i^\delta(p):=p_i\Big[\phi_i(p)+\delta_i-\sum_j p_j\delta_j-\varepsilon\big(\log p_i-\langle\log p\rangle\big)\Big].
\]

\paragraph{Exponential Lyapunov device (reflected model).}
Let $m(p):=\sum_j \delta_j p_j$ and $V(p):=\sum_j p_j(\delta_j-m(p))^2$ (variance of $\delta$ under $p$). For $\lambda>0$ define $U(p):=e^{\lambda m(p)}$.

\begin{lemma}[Lyapunov inequality]\label{lem:Lyap}
For the reflected diffusion with generator $\mathcal L^\delta$ and any $p\in\Delta^{K-1}_{\delta_\star}$,
\[
\mathcal L^\delta U(p)\ \ge\ U(p)\Big(\lambda\,V(p)-\lambda\,\|\delta\|_\infty\big(M_\phi+\varepsilon C_{\log}\big)\Big).
\]
In particular, with $\displaystyle \lambda:=\big(2\|\delta\|_\infty(M_\phi+\varepsilon C_{\log})\big)^{-1}$,
\[
\mathcal L^\delta U\ \ge\ U\big(\lambda V-\tfrac12\big).
\]
\end{lemma}

\begin{proof}
$\nabla U=\lambda U\,\delta$, $\nabla^2U=\lambda^2 U\,\delta\delta^\top$; the diffusion contribution is non–negative. For the drift, use $\sum_j p_j\delta_j(\delta_j-m)=V$ and the envelopes
$\sum_j p_j|\phi_j|\le M_\phi$, $\sum_j p_j|\log p_j-\langle\log p\rangle|\le C_{\log}$.
\end{proof}

\begin{theorem}[Stationary concentration near the fittest face]\label{thm:concentration}
Let $\pi_\infty$ be the invariant law of the reflected biased diffusion. Then
\[
\mathbb E_{\pi_\infty}[V]\ \le\ \frac{e^{\,2\lambda\|\delta\|_\infty}}{2\lambda}
\qquad\text{with}\qquad
\lambda=\frac{1}{2\|\delta\|_\infty(M_\phi+\varepsilon C_{\log})}.
\]
Since $V(p)\ge \gamma_\delta^2\,L(p)\big(1-L(p)\big)$ with $L(p):=\sum_{i\in I}p_i$, this implies the symmetric band estimate, for any $\theta\in(0,\tfrac12]$,
\[
\pi_\infty\big\{\ \theta\le L(p)\le 1-\theta\ \big\}\ \le\
\frac{e^{\,1/(M_\phi+\varepsilon C_{\log})}\ \|\delta\|_\infty(M_\phi+\varepsilon C_{\log})}{\gamma_\delta^2\,\theta(1-\theta)}.
\]
\end{theorem}

\paragraph{Remark (no fixation under a positive floor).}
If $\delta_\star>0$ then $\sum_{i\in I}p_i(t)\ge |I|\,\delta_\star$ for all $t$; thus one has \emph{concentration toward} (not fixation on) the fittest face. A bona fide fixation statement appears only in the vanishing–floor limit $\delta_\star\downarrow0$.

\subsection{Log–ratio SDEs (algorithm–specific)}

For $z_{ij}:=\log(p_i/p_j)$, Itô’s formula applied to \eqref{eq:WF} yields the exact identity
\begin{equation}\label{eq:logratio}
\mathrm d z_{ij}
=\big(\phi_i(p)-\phi_j(p)\big)\,\mathrm dt
-\varepsilon\,z_{ij}\,\mathrm dt
-\frac{\gamma}{2}\Big(\frac{1-p_i}{p_i}-\frac{1-p_j}{p_j}\Big)\mathrm dt
+\sqrt\gamma\Big(\frac{\mathrm dW_i}{\sqrt{p_i}}-\frac{\mathrm dW_j}{\sqrt{p_j}}\Big).
\end{equation}

\paragraph{GRPO (within–class).}
If all correct traces share the same centred score, $\phi_i=\phi_j$ within the class, then \eqref{eq:logratio} reduces to
\[
\mathrm d z_{ij}=-\varepsilon\,z_{ij}\,\mathrm dt
-\frac{\gamma}{2}\Big(\frac{1-p_i}{p_i}-\frac{1-p_j}{p_j}\Big)\mathrm dt
+\sqrt\gamma\Big(\frac{\mathrm dW_i}{\sqrt{p_i}}-\frac{\mathrm dW_j}{\sqrt{p_j}}\Big).
\]

\paragraph{STaR (within–class).}
If $\phi_i-\phi_j=(p_i-p_j)/\rho$ with $\rho:=\sum_{c\in\mathcal C}p_c$, then
\[
\mathrm d z_{ij}=\Big(\frac{p_i-p_j}{\rho}-\varepsilon\,z_{ij}\Big)\mathrm dt
-\frac{\gamma}{2}\Big(\frac{1-p_i}{p_i}-\frac{1-p_j}{p_j}\Big)\mathrm dt
+\sqrt\gamma\Big(\frac{\mathrm dW_i}{\sqrt{p_i}}-\frac{\mathrm dW_j}{\sqrt{p_j}}\Big).
\]
On $\Delta^{K-1}_{\delta_\star}$ one has $|p_i-p_j|/\rho\le \frac{1-(K-1)\delta_\star}{|\mathcal C|\,\delta_\star}\,|z_{ij}|$.

\paragraph{DPO (same–sign pair).}
With $s_i\in\{\pm1\}$ and $\phi_i(p)=s_i\,g_\beta(\log p_i)-\sum_k p_k s_k g_\beta(\log p_k)$, $g_\beta'(x)\in[-\beta/4,0)$; for $i,k$ with $s_i=s_k$ and $p_i\approx p_k$,
\[
\mathrm d z_{ik}\approx \big(s\,g_\beta'(\xi)-\varepsilon\big)\,z_{ik}\,\mathrm dt\ +\ \text{(Itô \& noise as in \eqref{eq:logratio})}.
\]
Intra–class log–ratios contract if $\varepsilon>\sup(-g_\beta')$ (e.g.\ $\varepsilon>\beta/4$).

\subsection{Regime dictionary (concise)}

Let $r:=\sigma^2/\lambda_{\rm eff}$ with $\sigma^2:=\gamma$ the diffusion variance scale and $\lambda_{\rm eff}$ a local contraction modulus of $F$ on $T$ (for log–ratios, $\lambda_{\rm eff}\gtrsim\varepsilon$). Under BD$^{\sharp}$:
\begin{itemize}[leftmargin=1.6em]
\item $r\ll1$ (low noise): tight interior concentration; ${\rm Var}(z_{ij})=O(\sigma^2/\varepsilon)$.
\item $r\asymp1$ (balanced): moderate interior spread; unique invariant law.
\item $r\gg1$ (noise–dominated but interior): broad interior law; faces are still repelling.
\end{itemize}
If BD$^{\sharp}$ fails, boundary approach and absorption may occur; interior concentration statements do not apply.

\medskip
\noindent\textbf{Summary.}
On the trimmed simplex, the SRCT drift is globally Lipschitz with an explicit modulus; mini–batch noise is centred with variance $O(1/B)$. The correct continuous–time limits are the ODE ($\eta/B\to0$) and a Wright–Fisher–type diffusion ($\eta/B\to\gamma$). The entropy face gap $L_K(\delta_\star)$ quantifies inward normal speed; BD$^{\sharp}$ yields ODE invariance and, for the unreflected SDE, high–probability confinement on finite horizons; the reflected diffusion is strictly invariant and exponentially ergodic. A small centred bias admits an exponential Lyapunov control that quantifies stationary concentration toward the fittest face. Exact log–ratio SDEs provide algorithm–specific envelopes (GRPO, STaR, DPO).

%% file: appendices_short/appendix_I_aistats.tex
\section{Kernel Design Strategies for SRCT}
\label{app:kernel_strategies_srct}

This appendix gives a self–contained, concise treatment of kernel design and analysis for SRCT. Part~\S\ref{sec:ideal} establishes an exact two–level stationarity condition, curvature (uniqueness/interiority), a tight log–ratio envelope with a dynamic floor, exponential convergence rates, a uniform suppression guarantee, and a block–constant PSD construction that realizes a prescribed class gap with controlled norms. Part~\S\ref{sec:practical} turns to practically learned kernels, including a gated effective kernel, exact suppression ratios, a support–function identity that quantifies diversity pressure, and an explicit global Lipschitz modulus for the SRCT drift.

\vspace{-1ex}
\paragraph{Setting, notation, and standing assumptions.}
Let $\mathcal S=\{\pi_1,\ldots,\pi_S\}$, $S\!\ge\!2$, and
$\Delta^{S-1}:=\{p\in[0,1]^S:\sum_{i=1}^S p_i=1\}$.
All logs are natural; $0\log 0:=0$.
Fix a partition $\mathcal S=\mathcal C\cup\mathcal I$ with $\mathcal C\cap\mathcal I=\varnothing$,
sizes $M:=|\mathcal C|\ge1$, $N:=|\mathcal I|=S-M$,
and utilities $U_i:=\mathbf 1_{\{i\in\mathcal C\}}\in\{0,1\}$.
Kernels are symmetric PSD: $K=K^\top\succeq 0$.
Vector norms $\|\cdot\|_2,\|\cdot\|_\infty$; operator norms
$\|A\|_{2\to 2}$ (spectral),
$\|A\|_{\infty\to\infty}:=\max_i\sum_j|A_{ij}|$,
$\|A\|_{\max}:=\max_{i,j}|A_{ij}|$.
Let $T:=\mathbf 1^\perp$ (tangent subspace) and $\Pi_T:=I-\tfrac{1}{S}\mathbf 1\mathbf 1^\top$.

\paragraph{SRCT objective, Shahshahani flow, and gauge.}
For $\lambda,\beta\ge0$ and entropy strength $A>0$ define
\[
\widetilde J(p):=U^\top p-\lambda\beta\,p^\top Kp+A\,H[p],
\qquad H[p]:=-\sum_{i=1}^S p_i\log p_i.
\]
Variational derivative (on $\operatorname{int}\Delta^{S-1}$):
\[
F_i(p)=\frac{\delta\widetilde J}{\delta p_i}=U_i-2\lambda\beta\,(Kp)_i-A\,(1+\log p_i),
\quad
\bar F(p):=\sum_j p_jF_j(p).
\]
The Shahshahani (replicator) flow is
\[
\dot p_i=p_i\big(F_i(p)-\bar F(p)\big),\qquad \sum_i \dot p_i=0.
\]
Adding a constant to $F$ leaves the vector field invariant (gauge invariance); thus the ``$+1$'' in $-A(1+\log p_i)$ can be absorbed into the KKT multiplier at stationarity.

\subsection{Idealized Kernel for a Two–Level Equilibrium}
\label{sec:ideal}

\paragraph{Two–level target.}
Fix $\delta_\star\in(0,1)$ with $N\delta_\star<1$ and set
\[
p_i^\star:=\delta_\star\quad(i\in\mathcal I),
\qquad
p_c^\star=:p_C:=\frac{1-N\delta_\star}{M}>0\quad(c\in\mathcal C),
\]
and write
$V_C:=(Kp^\star)_c$ (all $c\in\mathcal C$),
$V_I:=(Kp^\star)_i$ (all $i\in\mathcal I$).

\begin{proposition}[KKT $\iff$ classwise constancy $+$ gap]\label{prop:kkt-gap}
Under the two–level ansatz above, $p^\star$ is a stationary point of the Shahshahani flow if and only if
\begin{enumerate}[label=(\roman*), itemsep=2pt, topsep=2pt]
\item \emph{Classwise constancy:} $(Kp^\star)_c\equiv V_C$ for all $c\in\mathcal C$ and $(Kp^\star)_i\equiv V_I$ for all $i\in\mathcal I$.
\item \emph{Gap identity:}
\[
1-2\lambda\beta\,(V_C-V_I)-A\log\!\frac{p_C}{\delta_\star}=0.
\]
\end{enumerate}
\emph{Proof.}
Subtract the KKT equations for two indices in the same class to force classwise constancy;
subtract a correct–incorrect pair and use $U_c-U_i=1$ and
$\log p_c^\star-\log p_i^\star=\log(p_C/\delta_\star)$ to obtain the gap.
The converse is immediate by inspection. \qed
\end{proposition}

\paragraph{Curvature, strict concavity, uniqueness, interiority.}
Let
$\kappa_T:=\lambda_{\min}\big((\Pi_T K\Pi_T)|_T\big)\ge0$.
For any $v\in T$,
\[
\langle \nabla^2\widetilde J(p)v,v\rangle
=-A\sum_i \frac{v_i^2}{p_i}-2\lambda\beta\,v^\top K v
\le -(A+2\lambda\beta\,\kappa_T)\,\|v\|_2^2.
\]
Hence $\widetilde J$ is $A$–strongly concave on the affine simplex; in particular, the maximizer is unique and (by the steepness of $A\,H[p]$) interior.

\paragraph{Log–ratio dynamics, operator–norm envelope, dynamic floor.}
Let $z_{ij}:=\log\frac{p_i}{p_j}$.
Along trajectories,
\[
\dot z_{ij}=(U_i-U_j)-2\lambda\beta\big((Kp)_i-(Kp)_j\big)-A\,z_{ij}.
\]
For all $p\in\Delta^{S-1}$ and $i\neq j$,
\[
\big|(Kp)_i-(Kp)_j\big|
=\big|(K_{i\cdot}-K_{j\cdot})^\top p\big|
\le \Delta_K,
\]
where one may take any of the following (use the tightest available):
\[
\Delta_K\in\left\{\sqrt{2}\,\|K\|_{2\to 2},\ \ 2\,\|K\|_{\infty\to\infty},\ \ 2\,\|K\|_{\max},\ \
\max_{i\neq j}\|K_{i\cdot}-K_{j\cdot}\|_\infty\right\}.
\]
With $B_\sharp:=|U_i-U_j|+2\lambda\beta\,\Delta_K\le 1+2\lambda\beta\,\Delta_K$,
variation of constants yields
\[
|z_{ij}(t)|\le |z_{ij}(0)|e^{-At}+\frac{B_\sharp}{A}\,(1-e^{-At}).
\]
Let
\[
M_\sharp:=\max\!\Big\{\max_{k\ne \ell}|z_{k\ell}(0)|,\ \frac{B_\sharp}{A}\Big\},
\qquad
\delta:=S^{-1}e^{-M_\sharp}.
\]
Then, for all $t\ge0$ and all $i$, $\displaystyle \delta\le p_i(t)\le \frac{e^{M_\sharp}}{S}$,
so the ODE is globally well–posed and $\Delta_\delta:=\{p\in\Delta^{S-1}:\min_i p_i\ge\delta\}$ is forward–invariant.

\paragraph{Exponential convergence.}
Let $a(p):=F(p)-\langle p,F(p)\rangle\mathbf 1$.
Along trajectories,
$\frac{d}{dt}\widetilde J(p_t)=\sum_i p_i\,a_i(p_t)^2\ge \delta\|a(p_t)\|_2^2$ on $\Delta_\delta$.
Since $\widetilde J$ is $A$–strongly concave on the affine simplex,
$\widetilde J(p^\star)-\widetilde J(p)\le \frac{1}{2A}\|a(p)\|_2^2$.
Therefore, for all $t\ge0$,
\[
\widetilde J(p^\star)-\widetilde J(p_t)\le \big(\widetilde J(p^\star)-\widetilde J(p_0)\big)e^{-2A\delta\,t},
\qquad
\|p_t-p^\star\|_2
\le \sqrt{\tfrac{2}{A}\big(\widetilde J(p^\star)-\widetilde J(p_0)\big)}\,e^{-A\delta\,t}.
\]
Moreover, since $-\nabla^2\widetilde J(p)\succeq A\,\mathrm{diag}(1/p)$,
$\widetilde J$ is $A$–strongly concave in the Shahshahani metric $g_p(u,u)=\sum_i u_i^2/p_i$,
and the Riemannian PL inequality with the Lyapunov identity gives the $\delta$–free rate
\[
\widetilde J(p^\star)-\widetilde J(p_t)\le \big(\widetilde J(p^\star)-\widetilde J(p_0)\big)e^{-2A t}.
\]

\paragraph{Stationary structure and uniform suppression.}
At any equilibrium $p^\star$, subtracting KKT equations with the same utility yields, for $U_a=U_b$,
\[
\log\frac{p_a^\star}{p_b^\star}
=-\frac{2\lambda\beta}{A}\Big((Kp^\star)_a-(Kp^\star)_b\Big).
\]
For $c\in\mathcal C$, $i\in\mathcal I$,
\[
\log\frac{p_i^\star}{p_c^\star}
=-\frac{1}{A}\Big(1-2\lambda\beta\big((Kp^\star)_c-(Kp^\star)_i\big)\Big).
\]
A $p$–independent sufficient condition ensuring $p_i^\star<p_c^\star$ for all such pairs is
\[
\boxed{\,2\lambda\beta\,\Delta_K<1\,}
\quad\text{(use any $\Delta_K$ bound above; the $\ell_\infty$ row–difference is tight)}.
\]

\paragraph{Block–constant kernels: PSD, norms, gap realization, low–norm choice.}
Consider
\[
K_{ij}=
\begin{cases}
\kappa_{CC}, & i,j\in\mathcal C,\\
\kappa_{II}, & i,j\in\mathcal I,\\
\kappa_{CI}, & \text{otherwise.}
\end{cases}
\]
Let $B:=\begin{psmallmatrix}\kappa_{CC}&\kappa_{CI}\\ \kappa_{CI}&\kappa_{II}\end{psmallmatrix}$ and
$T:\mathbb R^2\to\mathbb R^S$, $T(a,b)=a\,\mathbf 1_{\mathcal C}+b\,\mathbf 1_{\mathcal I}$,
so $K=TBT^\top$ and $\mathrm{rank}(K)\le 2$.
Then $K\succeq0\iff B\succeq0$, i.e., $\kappa_{CC}\ge0$, $\kappa_{II}\ge0$, $\kappa_{CC}\kappa_{II}\ge\kappa_{CI}^2$.
Norm controls:
$\|K\|_{2\to2}\le \max\{M,N\}\,\|B\|_{2\to2}$ and
$\|K\|_{\infty\to\infty}=\max\{M|\kappa_{CC}|+N|\kappa_{CI}|,\ M|\kappa_{CI}|+N|\kappa_{II}|\}$.
With the two–level $p^\star$,
\[
(Kp^\star)_c-(Kp^\star)_i
=(\kappa_{CC}-\kappa_{CI})\,(1-N\delta_\star)+(\kappa_{CI}-\kappa_{II})\,N\delta_\star,
\]
so the gap identity of Proposition~\ref{prop:kkt-gap} becomes
\[
(1-N\delta_\star)(\kappa_{CC}-\kappa_{CI})+N\delta_\star(\kappa_{CI}-\kappa_{II})
=\frac{1-A\log(p_C/\delta_\star)}{2\lambda\beta}\;=:\;X.
\]
A low–norm constructive choice sets $\kappa_{CI}=0$ and then
\[
\kappa_{II}^{\min}=\max\Big\{0,\ -\frac{X}{N\delta_\star}\Big\}\quad(N\ge1),\qquad
\kappa_{CC}=\frac{X+N\delta_\star\,\kappa_{II}^{\min}}{1-N\delta_\star},
\]
minimizing $\|K\|_{\infty\to\infty}=\max\{M\kappa_{CC},\,N\kappa_{II}\}$ under PSD.
\emph{Edge case $N=0$:} the gap is void; maximizing $-\,\lambda\beta\,p^\top Kp+A\,H[p]$ yields a unique interior solution for $A>0$.

\subsection{Practical Design with a Learnable Semantic Kernel}
\label{sec:practical}

\paragraph{Gated effective kernel and objective.}
Let $k_{\mathrm{sem}}=k_{\mathrm{sem}}^\top\succeq0$ be a learnable semantic kernel and
$R\in\{0,1\}^S$ a binary verifier with $\mathcal C=\{i:R_i=1\}$, $\mathcal I=\{i:R_i=0\}$.
Define the \emph{effective} kernel
\[
K_{\mathrm{eff}}:=\mathrm{Diag}(R)\,k_{\mathrm{sem}}\,\mathrm{Diag}(R)\succeq 0.
\]
Consider the objective
\[
\mathcal J(p)=U^\top p+\lambda\big(\alpha\,H[p]-\beta\,p^\top K_{\mathrm{eff}}p\big),
\qquad \lambda,\alpha,\beta\ge0,
\]
and let the \emph{effective entropy coefficient} be
\[
\varepsilon_{\mathrm{tot}}:=\varepsilon_{\mathrm{base}}+\lambda\alpha,\qquad \varepsilon_{\mathrm{base}}>0.
\]
The SRCT flow uses the score $\phi_i(p)=U_i-2\lambda\beta\,(K_{\mathrm{eff}}p)_i$ and reads
\[
\dot p_i=p_i\big(\phi_i(p)-\bar\phi(p)\big)-\varepsilon_{\mathrm{tot}}\,p_i\big(\log p_i-\langle\log p\rangle\big),
\quad
\bar\phi(p):=\sum_j p_j\phi_j(p),\ \ \langle\log p\rangle:=\sum_j p_j\log p_j.
\]
Stationary points $p^\star\in\operatorname{int}\Delta^{S-1}$ satisfy the KKT system
\[
U_i-2\lambda\beta\,(K_{\mathrm{eff}}p^\star)_i-\varepsilon_{\mathrm{tot}}\big(1+\log p_i^\star\big)=\lambda_0,
\]
with the “$+1$” and $\lambda_0$ eliminated by taking differences.

\paragraph{Incorrect suppression and equalization among correct traces.}
Since $K_{\mathrm{eff}}(i,\cdot)\equiv 0$ for $i\in\mathcal I$,
$(K_{\mathrm{eff}}p^\star)_i=0$ and, for any $c\in\mathcal C$,
\[
\boxed{\quad
\frac{p_i^\star}{p_c^\star}
=\exp\!\left(-\frac{\,1-2\lambda\beta\,(K_{\mathrm{eff}}p^\star)_c\,}{\varepsilon_{\mathrm{tot}}}\right).
\quad}
\]
Thus strong suppression ($p_i^\star\!\ll\! p_c^\star$) is promoted by small $\varepsilon_{\mathrm{tot}}$ and moderate $\lambda\beta\,(K_{\mathrm{eff}}p^\star)_c$.
For $a,b\in\mathcal C$,
\[
\varepsilon_{\mathrm{tot}}\log\frac{p_a^\star}{p_b^\star}
=2\lambda\beta\Big((K_{\mathrm{eff}}p^\star)_b-(K_{\mathrm{eff}}p^\star)_a\Big),
\]
so larger $\varepsilon_{\mathrm{tot}}$ enhances equalization when the correct–side kernel averages are close.

\paragraph{Support–function identity (diversity pressure).}
For any $A\in\mathbb R^{S\times S}$ and distinct $i,j$,
\[
\sup_{p\in\Delta^{S-1}}\big|(Ap)_i-(Ap)_j\big|
=\sup_{p\in\Delta^{S-1}}\big|(A_{i\cdot}-A_{j\cdot})^\top p\big|
=\|A_{i\cdot}-A_{j\cdot}\|_\infty.
\]
(\emph{Proof:} $\Delta^{S-1}$ is the convex hull of basis vectors; the support function in direction $a$ equals $\max_k a_k$; take absolute values.)

Applying this to $A=K_{\mathrm{eff}}$ shows that the maximal instantaneous disparity of kernel averages across two correct indices is exactly the $\ell_\infty$ row–difference; when $k_{\mathrm{sem}}$ is semantically coherent, this term is larger across distinct semantic lumps, enforcing diversity via the $-\beta\,p^\top K_{\mathrm{eff}}p$ penalty.

\paragraph{Global Lipschitz modulus of the SRCT drift on a trimmed simplex.}
Let $\Delta^{S-1}_{\delta_\star}:=\{p\in\Delta^{S-1}:p_i\ge\delta_\star\ \forall i\}$ and
$\Lambda(\delta_\star):=1+\log(1/\delta_\star)$.
Write $S(p):=\mathrm{diag}(p)-pp^\top$ and
$E(p):=p\odot\big(\log p-\langle\log p\rangle\,\mathbf 1\big)$,
so the drift is $F(p)=S(p)\phi(p)-\varepsilon_{\mathrm{tot}}E(p)$ with $\phi(p)=U-2\lambda\beta\,K_{\mathrm{eff}}p$.
On $\Delta^{S-1}_{\delta_\star}$,
\[
\|S(p)\|_{2\to2}\le \tfrac12,\qquad \|S(p)-S(q)\|_{2\to2}\le 3\,\|p-q\|_2,
\]
\[
L_\phi^{(2)}:=2\lambda\beta\,\|K_{\mathrm{eff}}\|_{2\to2},\qquad
\|\phi(p)\|_2\le \sqrt M+2\lambda\beta\,\|K_{\mathrm{eff}}\|_{2\to2}=:M_{\phi,2},
\]
\[
\|E(p)-E(q)\|_2\le \Lambda(\delta_\star)\,(2+\sqrt S)\,\|p-q\|_2.
\]
Combining,
\[
\boxed{\quad
\|F(p)-F(q)\|_2\ \le\
\Big(\tfrac12\,L_\phi^{(2)}+3\,M_{\phi,2}\ +\ \varepsilon_{\mathrm{tot}}\,\Lambda(\delta_\star)\,(2+\sqrt S)\Big)\ \|p-q\|_2.
\quad}
\]
Hence the ODE is globally Lipschitz on $\Delta^{S-1}_{\delta_\star}$ with an explicit modulus.

\paragraph{Tuning guidance (concise).}
Smaller $\varepsilon_{\mathrm{tot}}$ (i.e., smaller $\lambda\alpha$ given $\varepsilon_{\mathrm{base}}$) yields exponentially stronger incorrect suppression but weaker equalization; larger $\varepsilon_{\mathrm{tot}}$ does the opposite. The coefficient $\lambda\beta$ regulates semantic diversity pressure via $K_{\mathrm{eff}}$ and should be chosen to spread mass across genuinely distinct correct lumps without excessively penalizing semantically coherent high–utility traces.

\vspace{-1ex}
\paragraph{Design–to–guarantee checklist (explicit constants).}
\begin{enumerate}[itemsep=1pt, topsep=2pt]
\item \emph{Target \& gap.}
$X=\dfrac{1-A\log(p_C/\delta_\star)}{2\lambda\beta}$ with $p_C=\dfrac{1-N\delta_\star}{M}$.
\item \emph{Kernel.} Choose symmetric PSD $K$ realizing the gap; for block–constant $K$, the low–norm choice is $\kappa_{CI}=0$ and
$\kappa_{II}=\kappa_{II}^{\min}$, $\kappa_{CC}=\dfrac{X+N\delta_\star\,\kappa_{II}^{\min}}{1-N\delta_\star}$.
\item \emph{Curvature (uniqueness/interiority).} Ensure $A>0$ (then the maximizer is unique and interior).
\item \emph{Log–ratio floor.} With any $\Delta_K$ option above, set
$B_\sharp=1+2\lambda\beta\,\Delta_K$,
$M_\sharp=\max\{\max_{i\ne j}|z_{ij}(0)|,\ B_\sharp/A\}$,
$\delta=S^{-1}e^{-M_\sharp}$; then $p_i(t)\in[\delta,e^{M_\sharp}/S]$ for all $t$.
\item \emph{Rates.} Euclidean–PL on $\Delta_\delta$:
$\|p_t-p^\star\|_2\le \sqrt{\tfrac{2}{A}\big(\widetilde J(p^\star)-\widetilde J(p_0)\big)}\,e^{-A\delta t}$; \quad
metric–PL ($\delta$–free): $\widetilde J(p^\star)-\widetilde J(p_t)\le(\widetilde J(p^\star)-\widetilde J(p_0))e^{-2At}$.
\item \emph{Suppression.} A uniform sufficient condition for $p_i^\star<p_c^\star$ is $2\lambda\beta\,\Delta_K<1$.
\end{enumerate}

\vspace{-1ex}
\paragraph{Notation hygiene and edge cases.}
Symbol $\delta_\star$ denotes the \emph{prescribed} target floor in the two–level ansatz,
while $\delta=S^{-1}e^{-M_\sharp}$ is the \emph{dynamic} floor from the log–ratio envelope.
When $N=0$, the cross–class gap is void; all curvature, floor, and convergence statements remain valid with $A>0$.

%% file: appendices_short/appendix_J_aistats.tex
\newcommand{\Fix}{\mathrm{Fix}}
\newcommand{\Jp}{J_p}

\section{Insight Experiments}
\label{app:insight-exp}

This appendix complements the main paper with simple experiments to validate parts of the theory. Unless stated otherwise: lines are means across five seeds and ribbons show $\pm$1\,s.d; the vertical line at step~200 indicates the event–detection smoothing floor. Metrics used throughout are the entropy $\Hp=-\sum_i p_i\log p_i$, fixation index $\Fix=\sum_i p_i^2$, cluster Gini (inequality over masses of the three correct–strategy clusters), incorrect mass (total probability on incorrect traces), and the objective proxy
\[
\Jp \;\coloneqq\; \text{utility mass} \;+\; \lambda\alpha\,\Hp \;-\; \lambda\beta\,p^\top K_{\mathrm{eff}}p.
\]

\subsection{Experimental Implementation and Reproducibility}
\label{app:exp-impl}

\vspace{0.5em}
\noindent\textbf{Synthetic trace universe.}
All experiments share the same finite “trace universe” with \(S=12\) traces. Eight traces are \emph{correct} and partitioned into three semantic clusters (strategies) \(A,B,C\) of sizes \(3,3,2\); the remaining four are \emph{incorrect}. Let \(\cC\subset\{1,\dots,12\}\) be the set of correct traces and \(\cI=\{1,\dots,12\}\setminus\cC\) the incorrect traces. A policy is a probability vector \(p\in\Delta^{S-1}\), with numerical clipping \(p_i\gets\max(p_i,10^{-12})\) before any \(\log\) is evaluated. Cluster membership is used only for analysis and, in Study~B, for the creativity kernel.

\paragraph{Verifier and rewards.}
Correctness is deterministic: \(U(i)=1\) for \(i\in\cC\), \(U(i)=0\) for \(i\in\cI\).
In Study~B, we additionally use base rewards \(r(i)=1.0\) for \(i\in\cC\) and \(r(i)=0.2\) for \(i\in\cI\).

\paragraph{Mini‑batch sampling and noise.}
Each update step draws a multinomial mini‑batch of size \(B\) from the current policy \(p\), yielding counts \(\mathbf{n}\sim\mathrm{Multinomial}(B,p)\) and empirical frequencies \(\hat p=\mathbf{n}/B\). All fitness/payoff computations that require batch statistics use \(\hat p\) (not the full \(p\)) so that finite‑batch noise is the only source of stochasticity.

\vspace{0.25em}
\noindent\textbf{Common metrics and event detection.}
At fixed intervals we log:
\begin{itemize}[leftmargin=1.25em,itemsep=2pt,topsep=2pt]
\item \emph{Entropy:} \(H[p]=-\sum_{i}p_i\log p_i\).
\item \emph{Fixation index:} \(\mathrm{Fix}=\sum_{i}p_i^2\) (monoculture \(\to 1\)).
\item \emph{Cluster masses:} \(m_A,m_B,m_C\) (probability within each correct cluster).
\item \emph{Cluster inequality:} Gini\((m_A,m_B,m_C)\).
\item \emph{Incorrect mass:} \(M_{\mathrm{inc}}=\sum_{i\in\cI}p_i\).
\item \emph{Objective proxy (Study~B):}
\(J_{\!p}=\textstyle\sum_{i\in\cC}p_i+\lambda\alpha\,H[p]-\lambda\beta\,p^\top K_{\text{eff}}p\),
where \(K_{\text{eff}}\) is the gated creativity kernel described below.
\end{itemize}
Events are detected on 50‑step moving averages with a 200‑step floor: (i) \emph{fixation} (STaR/GRPO) when \(\max_i p_i\ge 0.75\) and \(\max\{m_A,m_B,m_C\}\ge 0.9\); (ii) \emph{homogenization} (DPO) when the smoothed cluster Gini \(\le 0.10\) and all nonzero cluster masses \(\ge 0.15\).
Unless noted, runs use \(T=5000\) steps and five seeds \(\{101,202,303,404,505\}\); lines show seed means and ribbons \(\pm 1\) s.d.

\vspace{0.25em}
\noindent\textbf{Theoretical (replicator) update used in Studies A and A\textsuperscript{+}.}
All “theory” tracks use the same exponentiated‑gradient (replicator) step
\[
\tilde p_i \leftarrow p_i \exp\!\big(\eta\,[\phi_i-\varepsilon\log p_i]\big),\quad
p \leftarrow \tilde p / \|\tilde p\|_1,
\]
with learning rate \(\eta=0.15\) and barrier \(\varepsilon\in\{0,3\times10^{-4}\}\).
The method‑specific fitness \(\phi_i\) is:
\begin{align*}
\textbf{STaR:}&\quad \phi_i=\hat p_i/\hat\rho \ \text{if } i\in\cC,\ \text{else }0,
\quad\hat\rho=\sum_{c\in\cC}\hat p_c;\\
\textbf{GRPO:}&\quad \phi_i=\mathbf{1}\{i\in\cC\};\\
\textbf{DPO:}&\quad \phi_i=-\log\big(\max(\hat p_i,10^{-12})\big)\ \text{if } i\in\cC,\ \text{else }0.
\end{align*}

\vspace{0.25em}
\noindent\textbf{Algorithm‑faithful (procedural) updates used in Study A\textsuperscript{+}.}
In parallel to the “theory” track, we run \emph{algorithm‑faithful} procedures on logits \(\theta\) with \(p=\mathrm{softmax}(\theta)\):
\begin{itemize}[leftmargin=1.25em,itemsep=2pt,topsep=2pt]
\item \textbf{STaR (sequential reinforcement).} Sample up to \(L\) traces i.i.d.; on the first correct \(c\) apply
\(\theta\leftarrow \theta+\eta_{\text{star}}(\mathbf{e}_c-p)\).
If none is correct, no‑op that step. \(L\in\{16,64\}\) co‑varies with \(B\).
\item \textbf{GRPO (group REINFORCE with baseline).} Sample a group of size \(m\); with centered advantages \(a_j=r_j-\bar r\),
\(\theta \leftarrow \theta + \frac{\eta_{\text{grpo}}}{m}\sum_j a_j(\mathbf{e}_{i_j}-p)\);
\(m\in\{8,16,32\}\) depending on \(B\).
\item \textbf{DPO (pairwise preferences, Davidson ties).} For pairs \((i,j)\) drawn from the batch, compute the Davidson log‑likelihood with tie parameter \(\nu\) and take a gradient step \(\theta\leftarrow\theta+\eta_{\text{dpo}}\nabla_\theta \ell\). We use batched pairs and adaptive scaling to match one‑step norms to the theory track.
\end{itemize}
For each method and \(B\), \(\eta_{\text{proc}}\) (and, for DPO, pairs‑per‑step and \(\nu\)) is calibrated on a small set of \emph{anchor states} to maximize the mean cosine between one‑step \(\Delta p\) from the procedural and theory tracks while keeping the norm ratio close to \(1\).

\vspace{0.25em}
\noindent\textbf{DCR objective and kernel (Study B).}
Study~B augments a GRPO‑like base with a diversity energy \(\lambda(\alpha H[p]-\beta\,Q[p])\),
and folds the entropic term into the effective barrier:
\(\varepsilon \leftarrow \varepsilon_{\text{barrier}}+\lambda\alpha\) with \(\varepsilon_{\text{barrier}}=10^{-4}\).
The gated kernel is
\[
K_{\text{eff}} = R\,K_{\text{sem}}\,R,\qquad
R_{ii}=\mathbf{1}\{i\in\cC\},
\]
and \(K_{\text{sem}}(i,j)=1\) if \(i,j\) are correct and in the same cluster, else \(0\).
The fitness used in the replicator step is
\[
\phi_i \;=\; r(i)\;-\;2\lambda\beta\,(K_{\text{eff}}\hat p)_i,
\]
so that the quadratic penalty \(-\lambda\beta\,p^\top K_{\text{eff}}p\) discourages concentration on \emph{similar correct} traces only.
We sweep \(\alpha\in\{0.02,0.05,0.10\}\), \(\beta\in\{0.10,0.25,0.50,0.75\}\), with \(\lambda=1\), \(B=128\), \(\eta=0.15\).
Two ablations are reported: \emph{Entropy‑only} (\(\beta=0\)) and \emph{Ungated} (apply \(K\) to all traces).

\vspace{0.25em}
\noindent\textbf{Time horizons, seeds, and smoothing.}
Unless stated otherwise: \(T=5000\) steps; seeds \(\{101,202,303,404,505\}\);
50‑step moving averages and a 200‑step event floor are used for all event times and overlaid ribbons.

We run all experiments on a single \texttt{NVIDIA RTX 6000} with 49GB of VRAM. 

\subsection{Strategy–simplex overview (Fig.~\ref{fig:strategy-simplex})}
Figure~\ref{fig:strategy-simplex} provides a qualitative, distributional view of training on the three–strategy simplex (clusters A/B/C): STaR flows to a corner (monoculture), GRPO meanders along a neutral manifold before noise–driven fixation, DPO equalizes mass within the correct set, and DCR converges to a unique interior equilibrium with multi–strategy support. These panels summarize the high–level modes that are quantitatively confirmed in the subsequent figures.

\begin{figure*}[t]
  \centering
  \includegraphics[width=.85\linewidth]{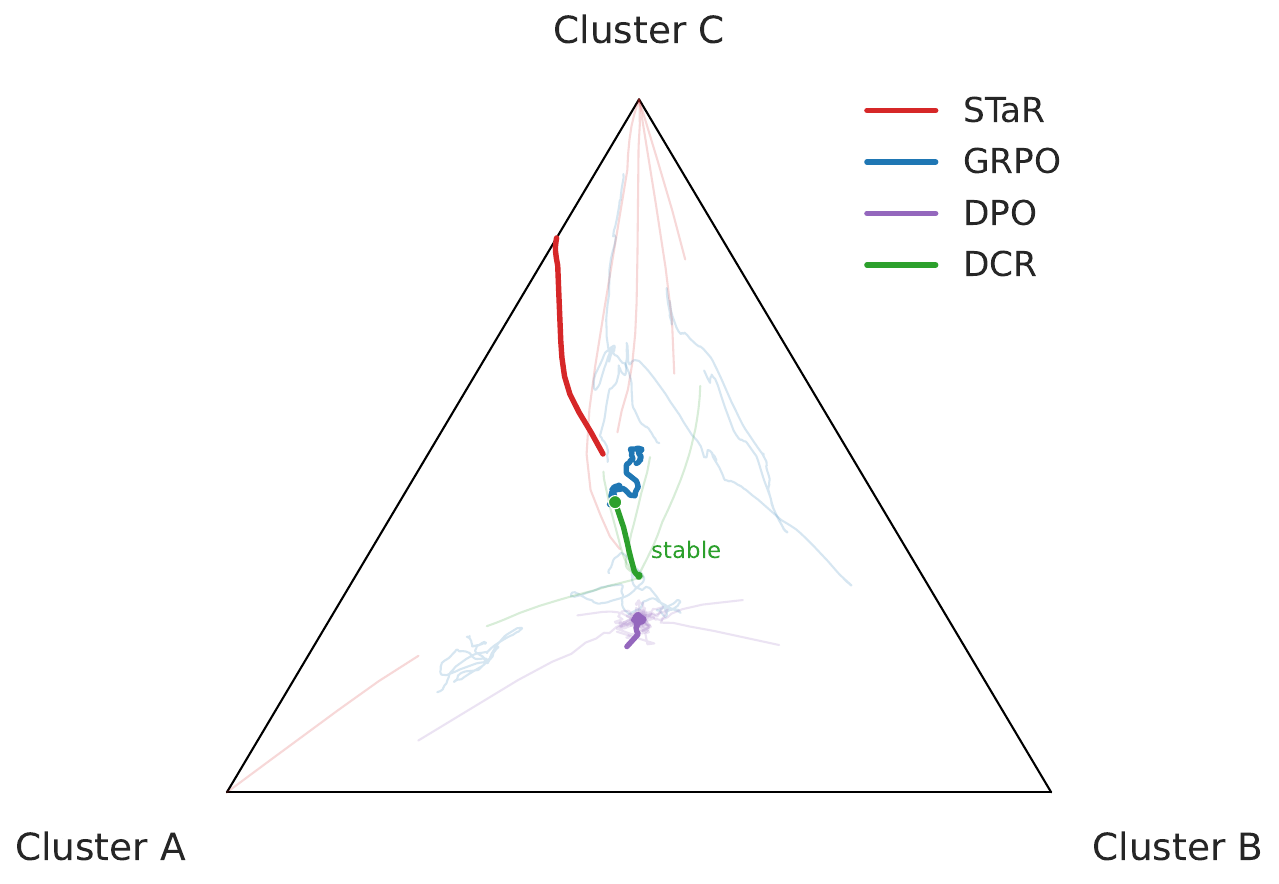}
  \caption{\textbf{Strategy–simplex dynamics.} Representative trajectories of cluster masses $(m_A,m_B,m_C)$ under STaR, GRPO, DPO, and DCR. STaR collapses to a vertex; GRPO drifts along the face; DPO equalizes on the face; DCR reaches a stable interior point retaining all clusters. Early (step~200) and late (step~5000) states are marked.}
  \label{fig:strategy-simplex}
\end{figure*}

\subsection{Study~A: scalar–objective dynamics (Fig.~\ref{fig:studyA})}
Figure~\ref{fig:studyA} aggregates the time evolution of $\Hp$, $\Fix$, cluster Gini, and incorrect mass for STaR, GRPO, and DPO. STaR collapses essentially immediately ($\Hp\!\to\!0$, $\Fix\!\to\!1$); GRPO exhibits slow, batch–size–dependent drift (median fixation $\approx$4.7k steps at $B{=}16$; no fixation by 5k at $B{=}64$); DPO homogenizes correct strategies early while maintaining zero incorrect mass.

\begin{figure*}[t]
  \centering
  \includegraphics[width=.98\linewidth]{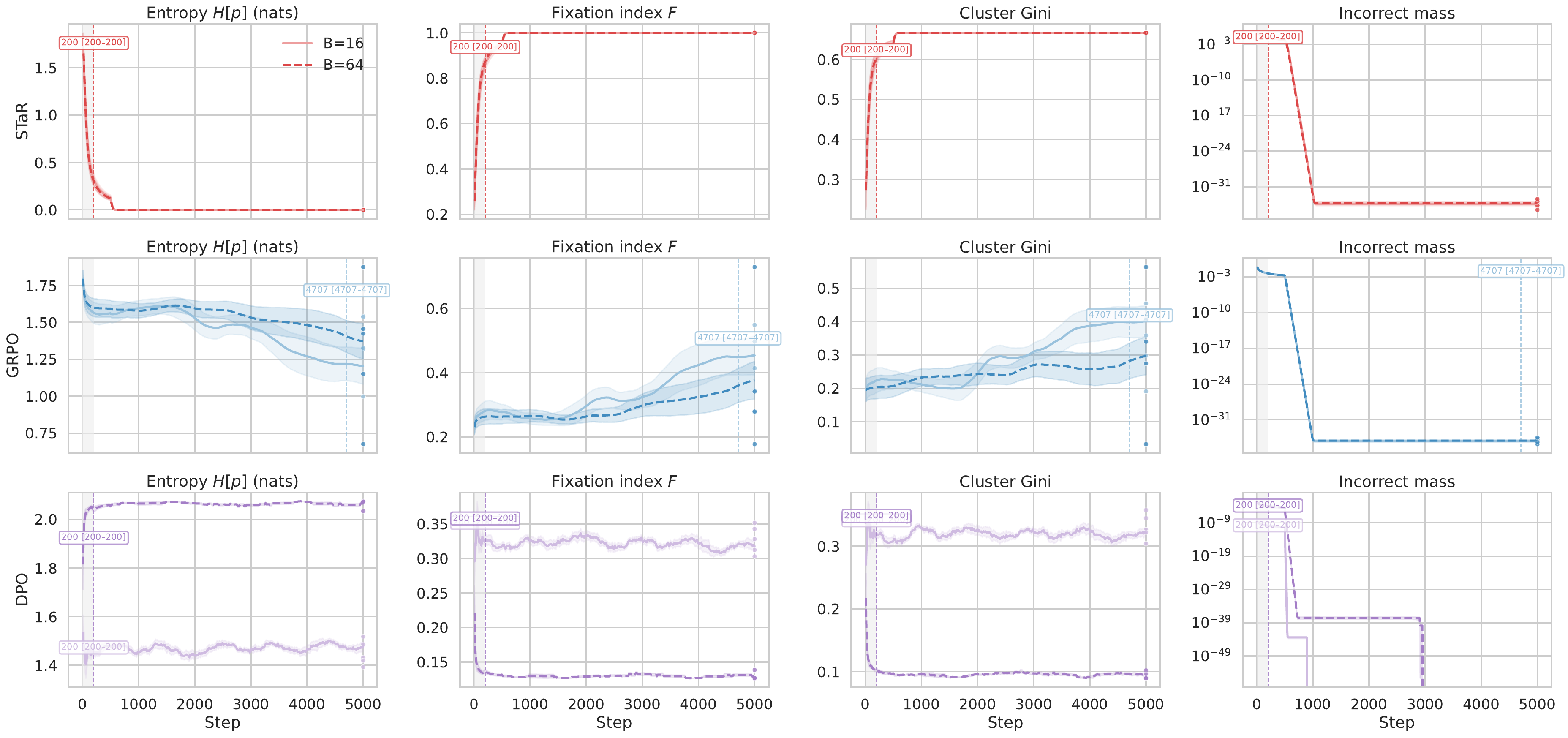}
  \caption{\textbf{Study~A: collapse modes.} Rows: STaR (top), GRPO (middle), DPO (bottom). Columns: entropy $\Hp$, fixation index $\Fix$, cluster Gini, incorrect mass (log scale). STaR deterministically fixates; GRPO drifts with speed increasing at smaller batch; DPO equalizes among correct traces while keeping incorrect mass at~0.}
  \label{fig:studyA}
\end{figure*}

\subsection{Study~B: overlays and alignment diagnostics (Figs.~\ref{fig:overlays}, \ref{fig:studyAplus}, \ref{fig:align-summary})}
The overlays in Fig.~\ref{fig:overlays} compare the replicator “theory” track and the algorithm–faithful procedural track for a common seed: STaR nearly coincides; GRPO shows small–magnitude neutral steps; DPO matches event timing but sustains higher entropy due to paired–comparison (Davidson ties) and the $\theta\!\mapsto\!p$ geometry. 

Per–step alignment in Fig.~\ref{fig:studyAplus} shows (i) high sign agreement for DPO with modest cosine (geometry mismatch), (ii) near–neutral GRPO behavior, and (iii) high STaR cosine with zero event–gap. Batch–size summaries in Fig.~\ref{fig:align-summary} confirm that, despite low cosines at larger $B$, the one–step JS divergence shrinks and event timing synchronizes. 

\begin{figure*}[t]
  \centering
  \includegraphics[width=.98\linewidth]{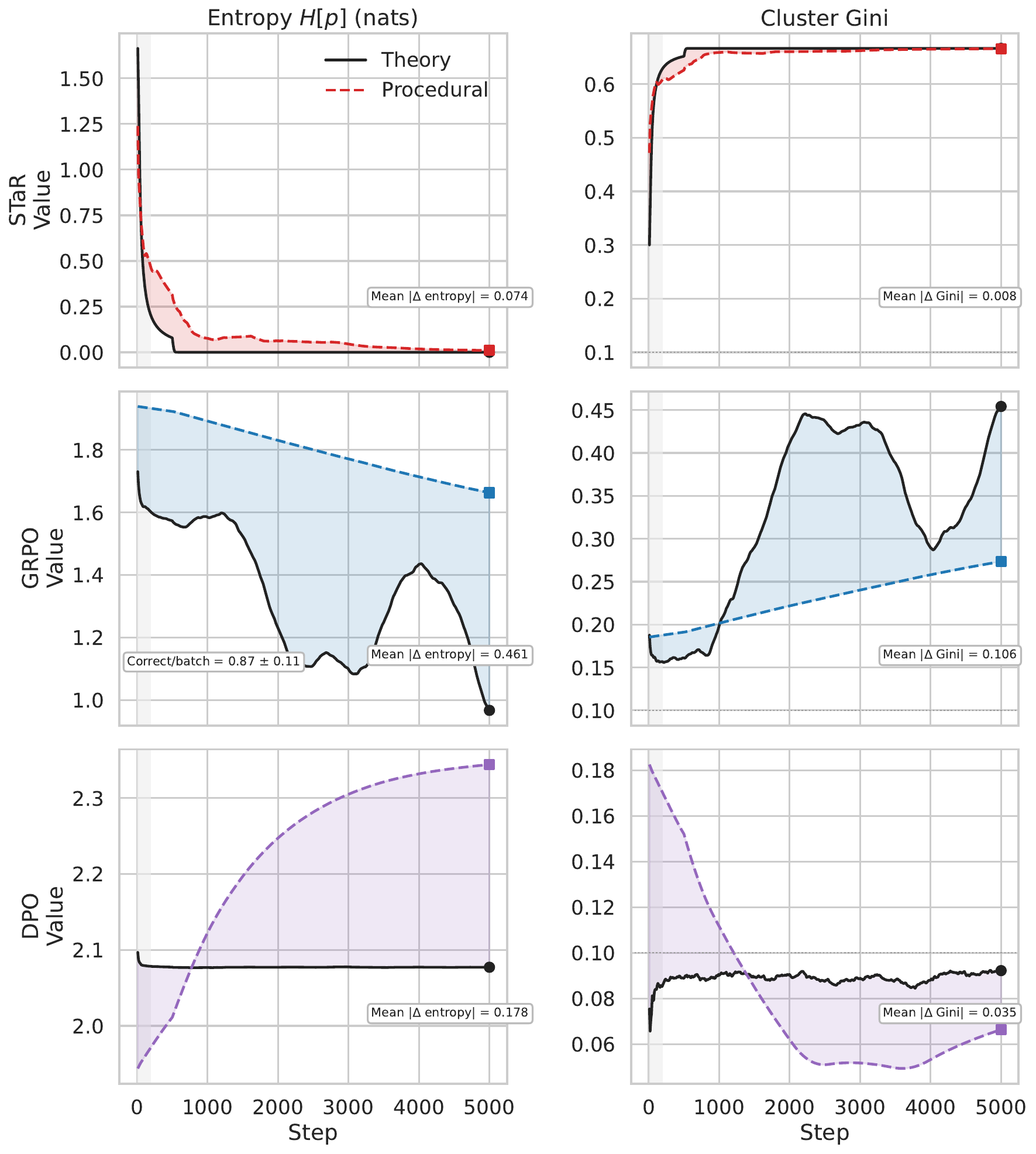}
  \caption{\textbf{Theory vs.\ procedural overlays (single seed).} Entropy and cluster–Gini trajectories for STaR, GRPO, and DPO. Procedural updates (sequential STaR, group REINFORCE, Davidson–ties DPO) track theory closely in \emph{events}; instantaneous directions differ most for DPO.}
  \label{fig:overlays}
\end{figure*}

\begin{figure*}[t]
  \centering
  \includegraphics[width=.98\linewidth]{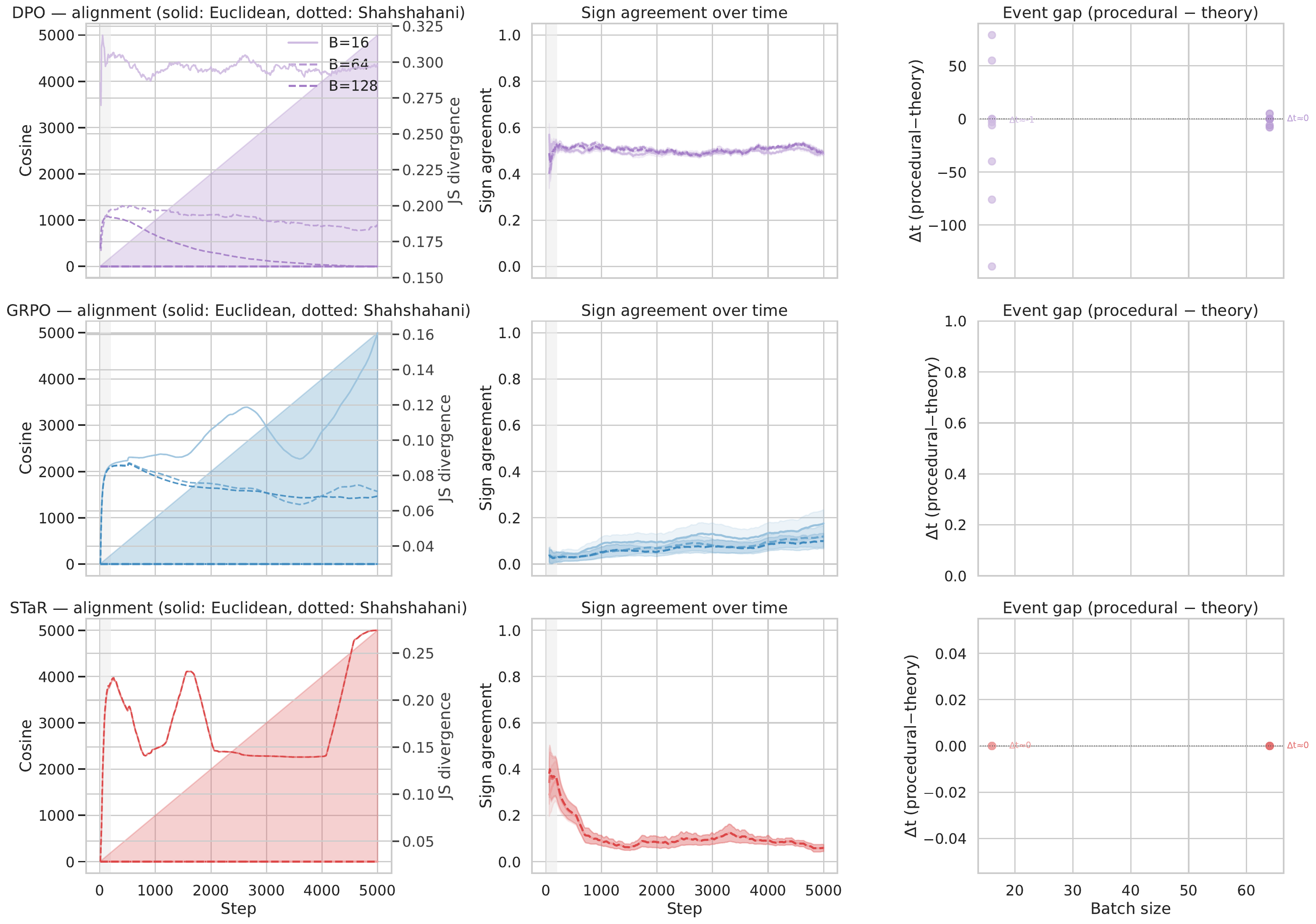}
  \caption{\textbf{Alignment vs.\ theory over time.} For each method: cosine of $\Delta p$ (solid: Euclidean; dotted: Shahshahani), sign agreement of log–ratio slopes, and event–time gap (procedural $-$ theory). DPO: low cosine, near–perfect signs; GRPO: near–neutral; STaR: high cosine, zero gap.}
  \label{fig:studyAplus}
\end{figure*}

\begin{figure*}[t]
  \centering
  \includegraphics[width=.9\linewidth]{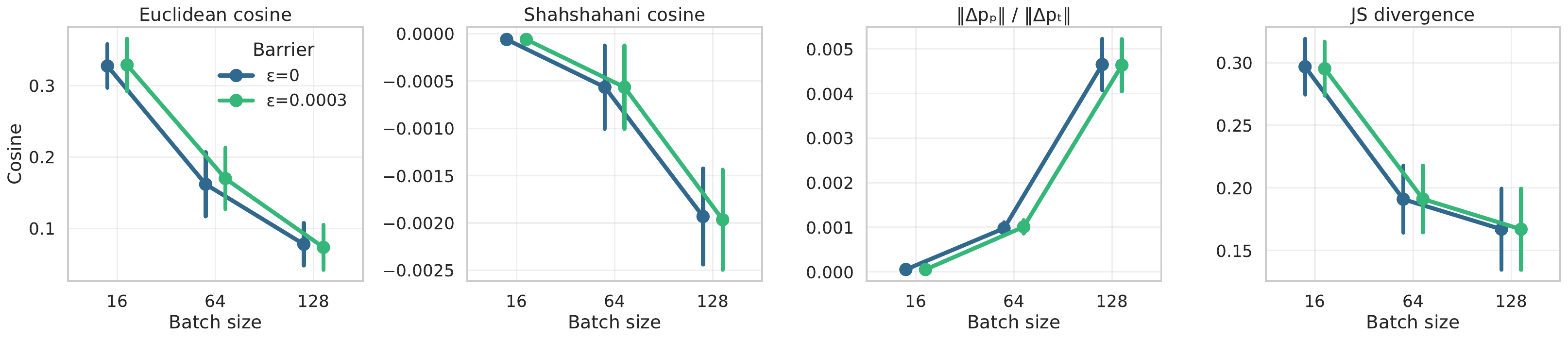}
  \caption{\textbf{Alignment summary vs.\ batch size.} Euclidean/Shahshahani cosine and one–step JS divergence as functions of $B$ (markers: mean; bars: s.d.). Cosine decreases with $B$ for DPO while JS concurrently decreases, indicating increasingly synchronous trajectories despite metric/parameterization mismatch. }
  \label{fig:align-summary}
\end{figure*}

\subsection{Study~C: DCR phase diagrams (Fig.~\ref{fig:phase}) and ablations (Fig.~\ref{fig:ablations})}
Figure~\ref{fig:phase} sweeps $(\alpha,\beta)$ and reports: incorrect mass, minimum cluster mass, between–seed JSD, and correct mass. A broad band achieves near–zero incorrect mass, full coverage, and negligible between–seed JSD—an empirical signature of a unique, interior, diverse equilibrium.

Figure~\ref{fig:ablations} compares \textsc{DCR}, \textsc{Entropy–only}, and \textsc{Ungated}. While coverage saturates at~3 for all, \textsc{DCR} reduces kernel energy (structured diversity) and maintains large positive safety margins; \textsc{Entropy–only} lacks targeted distinctiveness; \textsc{Ungated} penalizes incorrect–incorrect similarity, degrading safety despite larger proxy gains.

\begin{figure*}[t]
  \centering
  \includegraphics[width=.98\linewidth]{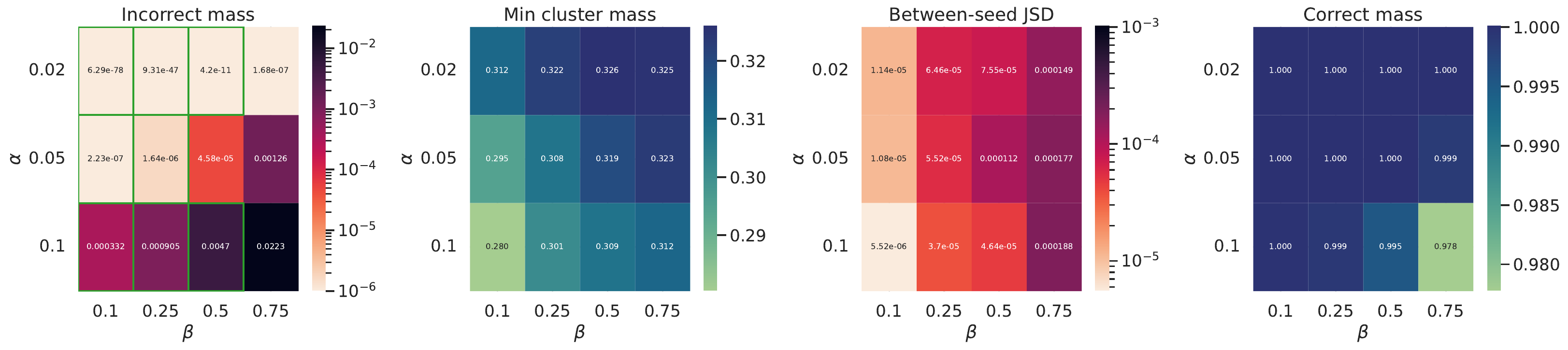}
  \caption{\textbf{DCR phase diagrams over $(\alpha,\beta)$.} From left to right: incorrect mass (log scale), minimum cluster mass, between–seed JSD, and correct mass. A contiguous band shows near–zero error, high structured diversity, and a unique terminal distribution.}
  \label{fig:phase}
\end{figure*}

\begin{figure*}[t]
  \centering
  \includegraphics[width=.98\linewidth]{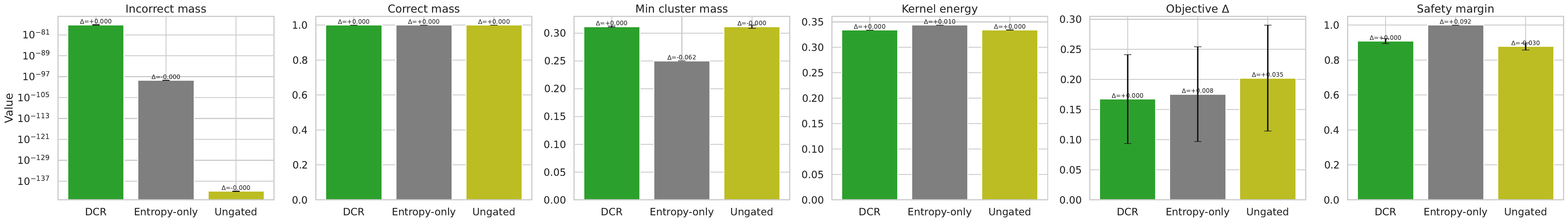}
  \caption{\textbf{DCR vs.\ ablations.} Bars (mean$\pm$sd) for incorrect mass (log axis), coverage, kernel energy, objective $\Delta\Jp$, and safety margin. \textsc{DCR} achieves the best trade–off (low error, full coverage, lower kernel energy, strong safety). \textsc{Entropy–only} preserves breadth without distinctiveness; \textsc{Ungated} reduces safety by penalizing similarity outside the correct set. }
  \label{fig:ablations}
\end{figure*}

\subsection{Objective and safety trajectories (Fig \ref{fig:objective-overlay})}
Figure~\ref{fig:objective-overlay} shows trajectories: \textsc{DCR} reaches a stable interior solution with safety $\gtrsim0.93$; \textsc{Entropy–only} has safety fixed at~1 (no kernel); \textsc{Ungated} converges at much lower safety ($\approx0.48$). 

\begin{figure}[t]
  \centering
  \includegraphics[width=.98\linewidth]{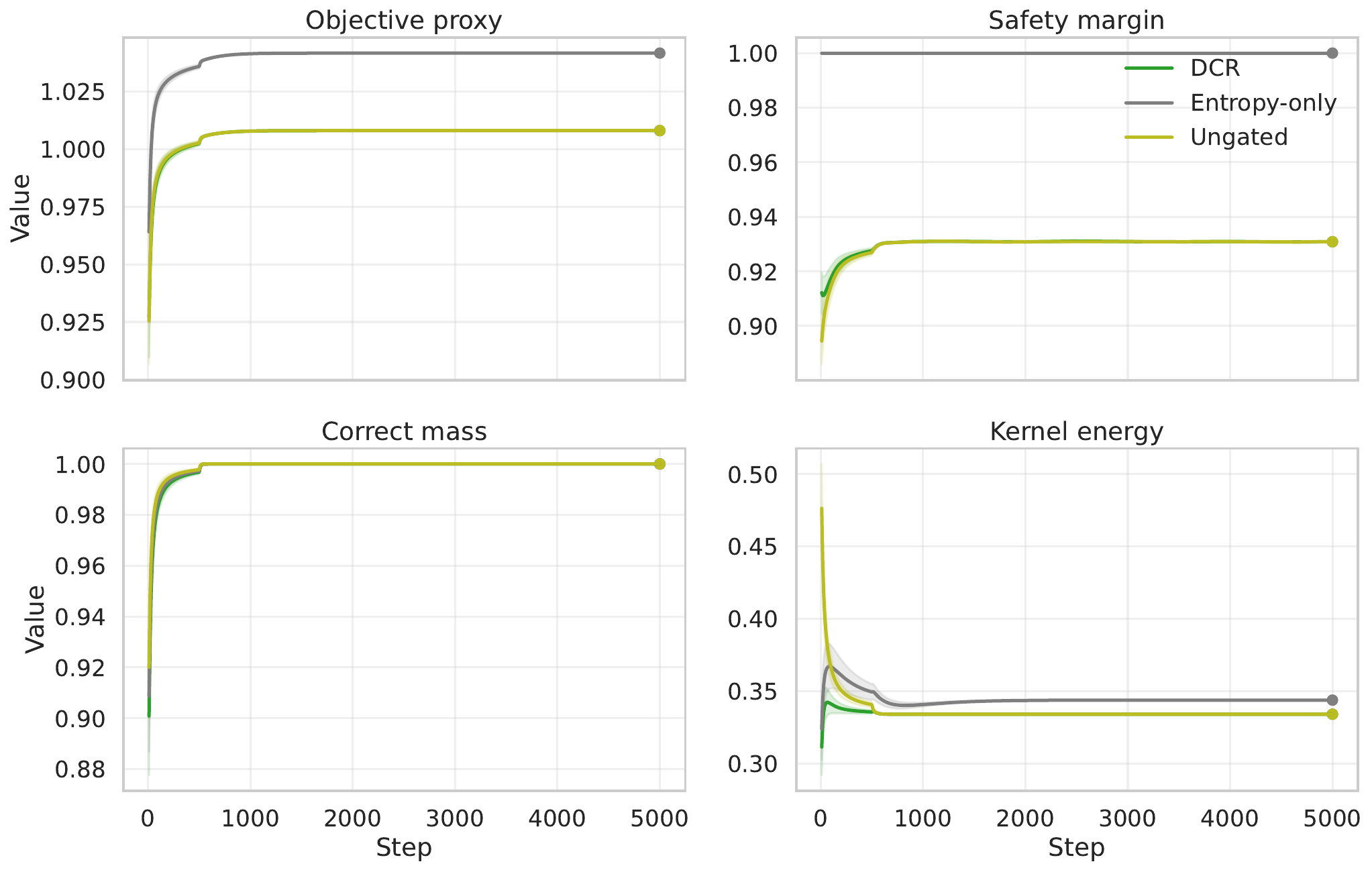}
  \caption{\textbf{Objective \& safety (overlay).} Overlay of $\Jp$ (left) and safety (right) for \textsc{DCR} (green), \textsc{Entropy–only} (gray), and \textsc{Ungated} (gold).}
  \label{fig:objective-overlay}
\end{figure}

\subsection{Safety–margin distribution (Fig.~\ref{fig:safety-hist})}
The histogram in Fig.~\ref{fig:safety-hist} reports the \emph{minimum} safety margin attained along training within the DCR band; all runs remain strictly positive (worst case $\approx0.267$), empirically validating the tuning rule that kernel pressure must not overwhelm the unit utility signal. 

\begin{figure}[t]
  \centering
  \includegraphics[width=.78\linewidth]{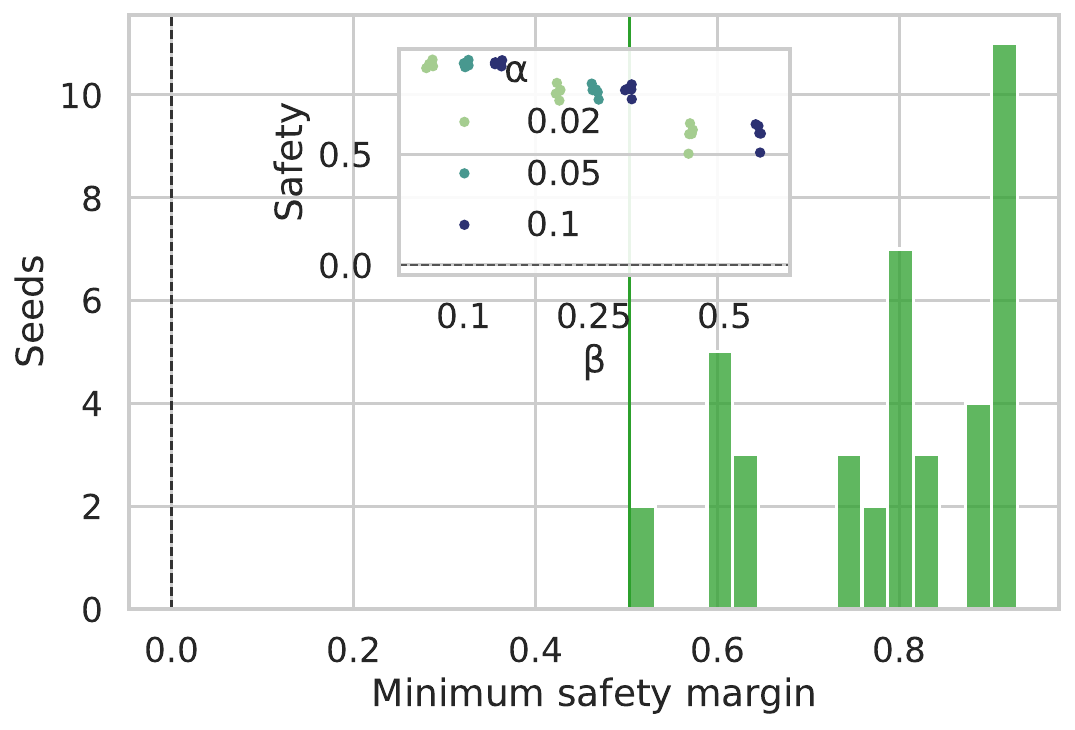}
  \caption{\textbf{Safety–margin distribution within the DCR band.} Minimum safety margin per run (bars) with a scatter inset over $(\alpha,\beta)$ (green markers). All seeds stay comfortably above~0 (min $\approx0.267$).}
  \label{fig:safety-hist}
\end{figure}